\documentclass{article}

\usepackage[final,nonatbib]{new_neurips_2022}

\usepackage[utf8]{inputenc} %
\usepackage[T1]{fontenc}    %
\usepackage{url}            %
\usepackage{booktabs}       %
\usepackage{amsfonts}       %
\usepackage{nicefrac}       %
\usepackage{microtype}      %
\usepackage[dvipsnames]{xcolor}         %
\usepackage{multicol}
\usepackage{subfiles}
\usepackage[hyperfootnotes=true, hidelinks]{hyperref}
\hypersetup{ %
    pdftitle={},
    pdfauthor={},
    colorlinks=true,
    linkcolor=MidnightBlue,
    citecolor=MidnightBlue,
    filecolor=MidnightBlue,
    urlcolor=MidnightBlue,
    }
\usepackage{soul} %
\usepackage{bbding}
\usepackage{pifont}
\usepackage{caption}

\makeatletter
\newif\ifnobrackets
\renewcommand\@cite[2]{\ifnobrackets\else[\fi{#1\if@tempswa , #2\fi}\ifnobrackets\else]\fi\nobracketsfalse}

\makeatother

\title{A Multi-Resolution Framework for U-Nets \\  with Applications to Hierarchical VAEs}%

\author{%
  Fabian Falck \textbf{\thanks{Equal contribution.}} $^{\ 1,3,4}$ \,\,\, Christopher Williams $^{*,1}$ \,\,\, Dominic Danks $^{2,4}$ \,\,\, George Deligiannidis $^{1}$
  \\ \bf Christopher Yau $^{1,3,4}$ \,\,\, Chris Holmes $^{1,3,4}$ \,\,\, Arnaud Doucet $^{1}$ \,\,\, Matthew Willetts $^{4}$ \\
  $^1$University of Oxford\, $^2$University of Birmingham\, \\
 $^3$Health Data Research UK\, $^4$The Alan Turing Institute \\
  \texttt{\{fabian.falck, williams, deligian, cholmes, doucet\}@stats.ox.ac.uk,} \\
  \texttt{\{ddanks, cyau, mwilletts\}@turing.ac.uk\}}
}

\usepackage{url}            %
\PassOptionsToPackage{hyphens}{url}

\usepackage{booktabs}       %
\usepackage{amsfonts}       %
\usepackage{amsmath}
\usepackage{amssymb}
\usepackage{nicefrac}       %
\usepackage{microtype}      %
\usepackage{wrapfig}

\DeclareMathOperator{\KL}{KL}

\usepackage{bm}
\usepackage{amsthm}

\newtheorem{theorem}{Theorem}
\usepackage{tikz}
\usetikzlibrary{bayesnet}
\usepackage{graphicx}
\usepackage{caption}
\usepackage{subcaption}
\usepackage{paralist}

\usepackage{algorithmicx}
\usepackage{algorithm}
\usepackage{algpseudocode}

\usepackage{array,multirow,graphicx}
\usepackage{float}

\usepackage{wrapfig}

\usepackage{todonotes}
\usepackage{tikz-cd}
\usepackage{bbm}
\usepackage[shortlabels,inline]{enumitem}

\usepackage{multicol}

\usepackage{wrapfig}

\usepackage{lscape}

\newcommand{\R}{\mathbb{R}}

\newcommand{\Z}{\mathbb{Z}}
\newcommand{\N}{\mathbb{N}}

\newcommand{\bE}{\mathbf{E}}
\newcommand{\bPsi}{\mathbf{\Psi}}
\newcommand{\Pb}{\mathbb{P}}
\newcommand{\Eb}{\mathbb{E}}
\newcommand{\Xb}{\mathbb{X}}
\newcommand{\Db}{\mathbb{D}}

\newcommand{\1}{\mathbbm{1}}
\newcommand{\sds}{\, \cdot \,}

\newcommand{\proj}{\text{proj}}
\newcommand{\embd}{\text{embd}}

\newcommand{\norm}[1]{ \left\| #1 \right\| }

\newcommand{\bmu}{\overleftarrow{\mu}}

\newcommand{\bsig}{\overleftarrow{\sigma}}

\theoremstyle{definition}
\newtheorem{definition}{Definition}%
\newtheorem{lemma}{Lemma}
\newtheorem{remark}{Remark}
\newtheorem{prop}{Proposition}
\newtheorem*{theorem*}{Theorem}

\newcommand*\colourcross[1]{%
  \expandafter\newcommand\csname #1cross\endcsname{\textcolor{#1}{\ding{55}}}%
}
\colourcross{black}

\newif\ifdraft
\draftfalse %

\newcommand{\ff}[1]{\textbf{\textcolor{blue}{#1}}}

\DeclareMathOperator{\expect}{\mathbb{E}}

\DeclareMathOperator{\ELBO}{\mathcal{L}}

\newcommand*\intd{\mathop{}\!\mathrm{d}}

\usepackage{mathtools}

\renewcommand{\vector}[1]{\boldsymbol{\mathbf{#1}}}
\renewcommand{\v}{\vector}

\newcommand*\vv[1]{\vec{\v{#1}}}

\newcommand{\B}[1]{\boldsymbol{#1}}

\newcommand{\newreptheorem}[2]{\newtheorem*{rep@#1}{\rep@title}\newenvironment{rep#1}[1]{\def\rep@title{#2 \ref*{##1}}\begin{rep@#1}}{\end{rep@#1}}}
\newreptheorem{theorem}{Proposition}

\usepackage[toc,page,header]{appendix}
\usepackage{minitoc}

\begin{document}
\maketitle
\begin{abstract}

U-Net architectures are ubiquitous in state-of-the-art deep learning, however their regularisation properties and relationship to wavelets are understudied.
In this paper, we formulate a multi-resolution framework which identifies U-Nets as finite-dimensional truncations of models on an infinite-dimensional function space.
We provide theoretical results which prove that average pooling corresponds to projection within the space of square-integrable functions and show that U-Nets with average pooling implicitly learn a Haar wavelet basis representation of the data.
We then leverage our framework to identify state-of-the-art hierarchical VAEs (HVAEs), which have a U-Net architecture, as a type of two-step forward Euler discretisation of multi-resolution diffusion processes which flow from a point mass, introducing sampling instabilities.
We also demonstrate that HVAEs learn a representation of time which allows for improved parameter efficiency through weight-sharing. 
We use this observation to achieve state-of-the-art HVAE performance with half the number of parameters of existing models, exploiting the properties of our continuous-time formulation.

\end{abstract}

\section{Introduction}
\label{sec:Introduction}

U-Net architectures are extensively utilised in modern deep learning models.
First developed for image segmentation in biomedical applications~\cite{ronneberger2015u}, U-Nets have been widely applied for text-to-image models~\cite{saharia2022photorealistic}, image-to-image translation~\cite{image_to_image}, image restoration~\cite{image_restoration_2016, image_restoration_2021}, super-resolution~\cite{shi2022conditional}, and multiview learning~\cite{contrastive_multiview_coding}, amongst other tasks \cite{9196843}.
They also form a core building block as the neural architecture of choice in state-of-the-art generative models, particularly for images, such as HVAEs~\cite{Child2020VeryImages, Vahdat2020NVAE:Autoencoder,hazami2022efficient,kohl2019hierarchical} and diffusion models~\cite{saharia2022photorealistic,DDPM,DDIM, ImprovedDDPM,song2020score, diffusion_models_beat_gans, de2021diffusion,  Kingma2021VariationalModels,rombach2022high}. 
In spite of their empirical success, it is poorly understood why U-Nets work so well, and what regularisation they impose.

In likelihood-based generative modelling, various model classes are competing for superiority, including normalizing flows~\cite{ho2019flow++,kingma2018glow}, autoregressive models~\cite{child2019generating,van2016conditional}, diffusion models, and hierarchical variational autoencoders (HVAEs), the latter two of which we focus on in this work.
HVAEs form groups of latent variables with a conditional dependence structure, use a U-Net neural architecture, and are trained with the typical VAE ELBO objective (for a detailed introduction to HVAEs, see Appendix \ref{sec:Background}).
HVAEs show impressive synthesis results on facial images, and yield competitive likelihood performance, consistently outperforming the previously state-of-the-art autoregressive models, VAEs and flow models on computer vision benchmarks~\cite{Child2020VeryImages,Vahdat2020NVAE:Autoencoder}.
HVAEs have undergone a journey of design iterations and architectural improvements in recent years, for example the introduction a deterministic backbone~\cite{Rezende2014StochasticModels,zhao2017learning,falck2021multi} and ResNet elements~\cite{Kingma2016ImprovedFlow,he2016deep} with shared parameters between the inference and generative model parts. 
There has also been a massive increase in the number of latent variables and overall stochastic depth,  as well as the use of different types of residual cells in the decoder~\cite{Child2020VeryImages,Vahdat2020NVAE:Autoencoder} (see \S\ref{sec:Related_work} and Fig.~\ref{fig:hvae_cells_all} for a detailed discussion). 
However, a theoretical understanding of these choices is lacking. 
For instance, it has not been shown why a residual backbone may be beneficial, or what the specific cell structures in VDVAE~\cite{Child2020VeryImages} and NVAE~\cite{Vahdat2020NVAE:Autoencoder} correspond to, or how they could be improved.

\begin{figure}
\centering
\includegraphics[width=\linewidth]{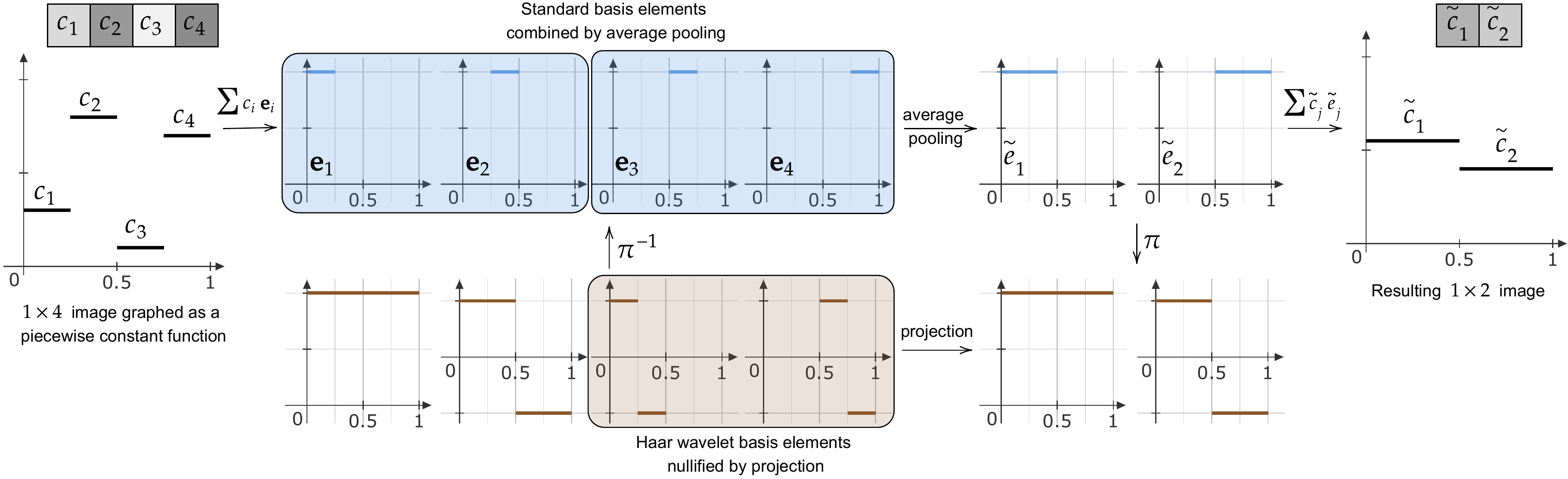}
\caption{
U-Nets with average pooling learn a Haar wavelet basis representation of the data.
}  
\label{fig:change_of_basis}
\end{figure}

In this paper we provide a theoretical framework for understanding the latent spaces in U-Nets, and apply this to HVAEs specifically. 
Doing so allows us to relate HVAEs to diffusion processes, and also to motivate a new type of piecewise time-homogenenous model which demonstrates state-of-the-art performance with approximately half the number of parameters of a~VDVAE~\cite{Child2020VeryImages}.
More formally, our contributions are as follows:
\textbf{(a)} 
We provide a multi-resolution framework for U-Nets. 
We formally define U-Nets as acting over a multi-resolution hierarchy of $L^2([0,1]^2)$. 
We prove that average pooling is a conjugate operation to projection in the Haar wavelet basis within $L^2([0,1]^2)$. 
We use this insight to show how U-Nets with average pooling implicitly learn a Haar wavelet basis representation of the data (see Fig. \ref{fig:change_of_basis}), helping to characterise the regularisation within U-Nets.
\textbf{(b)}
We apply this framework to state-of-the-art HVAEs as an example, identifying their residual cell structure as a type of two-step forward Euler discretisation of a multi-resolution diffusion bridge. 
We uncover that this diffusion process flows from a point mass, which causes instabilities, for instance during sampling, and identify parameter redundancies through our continuous-time formulation.
Our framework both allows us to understand the heuristic choices of existing work in HVAEs and enables future work to optimise their design, for instance their residual cell.
\textbf{(c)}
In our experiments, we demonstrate these sampling instabilities and train HVAEs with the largest stochastic depth ever, achieving state-of-the-art performance with half the number of parameters by exploiting our theoretical insights. 
We explain these results by uncovering that HVAEs secretly represent time in their state and show that they use this information during training.
We finally provide extensive ablation studies which, for instance, rule out other potential factors which correlate with stochastic depth, show the empirical gain of multiple resolutions, and find that Fourier features (which discrete-time diffusion models can strongly benefit from~\cite{Kingma2021VariationalModels}) do not improve performance in the HVAE setting.

\section{The Multi-Resolution Framework} 
\label{sec:multi-res}

A grayscale
image with infinite resolution can be thought of as the graph\footnote{
For a function $f(\,\cdot\,)$, its graph is the set $\bigcup_{x \in [0,1]^2} \{x, f(x) \}$.
} of a two-dimensional function over the unit square.
To store these infinitely-detailed images in computers, we project them to some finite resolution.
These projections can still be thought of as the graphs of functions with support over the unit square, but they are piecewise constant on finitely many intervals or `pixels’, e.g. $512^2$ pixels, and we store the function values obtained at these pixels in an array or `grid’.
The relationship between the finite-dimensional version and its infinitely-fine counterpart depends entirely on how we construct this projection to preserve the details we wish to keep.
One approach is to prioritise preserving the large-scale details of our images, so unless closely inspected, the projection is indistinguishable from the original.
This can be achieved with a multi-resolution projection
\cite{daubechies1992ten} of the image.
In this section we introduce a \emph{multi-resolution framework} for constructing neural network architectures that utilise such projections, prove what regularisation properties they impose, and show as an example how HVAEs with a U-Net \cite{ronneberger2015u} architecture can be interpreted in our framework. 
Proofs of all theorems in the form of an extended exposition of our framework can be found in Appendix~\ref{app:Proofs}.

\subsection{Multi-Resolution Framework: Definitions and Intuition}
\label{sec:Multi-Resolution Framework: Definitions and Intuition}

What makes a multi-resolution projection good at prioritising large-scale details can be informally explained through the following thought experiment.
Imagine we have an image, represented as the graph of a function, and its finite-dimensional projection drawn on the wall. 
We look at the wall, start walking away from it and stop when the image and its projection are indistinguishable by eye.
The number of steps we took away from the wall can be considered our measure of `how far away' the approximation is from  the underlying function.
The goal of the multi-resolution projection is therefore to have to take as few steps away as possible.
The reader is encouraged to physically conduct this experiment with the examples provided in Appendix~\ref{app:Thought experiment on multiresolution analysis}.
We can formalise the aforementioned intuition by defining a \emph{multi-resolution hierarchy}~\cite{daubechies1992ten} of subspaces we may project to: 

\begin{definition}%
\label{def:multi-res-approx-space}
[\emph{Daubechies (1992)}~\cite{daubechies1992ten}] Given a nested sequence of \emph{approximation spaces} $\cdots \subset V_1 \subset V_0 \subset V_{-1} \subset \cdots$, $\{V_{-j}\}_{{j\in}\mathbb{Z}}$ is a \emph{multi-resolution hierarchy} of the function space $L^2(\R^m)$ if:
    \begin{enumerate*}[start=1,label={(\bfseries A\arabic*) }]
        \item ${\overline{\bigcup_{j \in \Z} V_{-j}}  =  L^2(\R^m)}$;
        \item ${\bigcap_{j \in \Z} V_{-j}   = \{ 0 \}}$;
        \item ${f(\cdot) \in V_{-j} \Leftrightarrow f(2^j  \cdot  ) \in V_0}$;
        \item ${f(\cdot)\in V_0 \Leftrightarrow  f(  \cdot -n ) \in V_0}$ for $n \in \Z$.
    \end{enumerate*}
For a compact set $\Xb \subset \R^m$, a \emph{multi-resolution hierarchy} of $L^2(\Xb)$ is $\{V_{-j}\}_{{j\in}\mathbb{Z}}$ as defined above, restricting functions in $V_{-j}$ to be supported on $\Xb$.
\end{definition}

In Definition \ref{def:multi-res-approx-space}, the index $j$ references how many steps we took in our thought experiment, so negative $j$ corresponds to `zooming in' on the images.  %
The original image\footnote{
We here focus on grayscale, squared images for simplicity, but note that our framework can be seamlessly extended to colour images with a Cartesian product $L^2([0,1]^2) \times L^2([0,1]^2) \times L^2([0,1]^2)$, and other continuous signals such as time series.}  %
is a member of $L^2([0,1]^2)$, the space of square-integrable functions on the unit square, and its finite projection to $2^{j} \cdot 2^{j}$ many pixels is a member of $V_{-j}$.
Images can be represented as piecewise continuous functions in the subspaces
    $V_{-j} = \{ f \in L^2([0,1]) \ | \ f|_{[2^{-j} \cdot k, 2^{-j} \cdot (k+1) )} = c_k , \, k \in \{0, \dots, 2^j-1 \}, \, c_k \in \R \}.$
The nesting property $V_{-j+1} \subset V_{-j}$ ensures that any image with $(2^{j-1})^2$ pixels can also be represented by $(2^{j})^2$ pixels, but at a higher resolution.
Assumption (\textbf{A1}) states that with infinitely many pixels, we can describe any infinitely detailed image.
In contrast, (\textbf{A2}) says that with no pixels, we cannot approximate any images.
Assumptions (\textbf{A3}) and (\textbf{A4}) allow us to form a basis for images in any $V_{-j}$ if we know the basis of $V_0$.
One basis made by extrapolating from $V_0$ in this way is known as a \textit{wavelet basis} \cite{daubechies1992ten}. 
Wavelets have proven useful for representing images, for instance in the JPEG standard \cite{taubman2012jpeg2000}, and are constructed to be orthonormal.

Now suppose we have a probability measure $\nu_{\infty}$ over infinitely detailed images represented in $L^2([0,1]^2)$ and wish to represent it at a lower resolution.
Similar to how we did for infinitely detailed images, we want to project the measure $\nu_{\infty}$ to a lower dimensional measure $\nu_{j}$ on the finite dimensional space $V_{-j}$.
In extension to this, we want the ability to reverse this projection so that we may sample from the lower dimensional measure and create a generative model for $\nu_{\infty}$. 
We would like to again prioritise the presence of large-scale features of the original image within the lower dimensional samples. 
We do this by constructing a \emph{multi-resolution bridge} from $\nu_{\infty}$ to $\nu_j$, as defined below.

\begin{definition}%
\label{def:multi-res-bridge}
Let $\Xb \subset \R^{m}$ be compact, $\{V_{-j} \}_{j=0}^{\infty}$ be a multi-resolution hierarchy of scaled so $L^2(\Xb) = \overline{\bigcup_{j \in \N_0 } V_{-j} }$ and $V_0 = \{0\}$.
If $\Db(L^2(\Xb))$ is the space of probability measures over $L^2(\Xb)$, then a family of probability measures $\{\nu_t\}_{t \in [0,1]}$ on $L^2(\Xb)$ is a \emph{multi-resolution bridge} if:
\begin{enumerate}[(i),itemsep=1mm,topsep=0pt,parsep=0pt,partopsep=0pt,leftmargin=20pt]  %
    \item there exist increasing times $\mathcal{I} \coloneqq \{t_j \}_{j \in \N_0}$ where $t_0 =0$, $\lim_{j \to \infty} t_j = 1$, such that $s \in [t_j, t_{j+1})$ implies $\text{supp}( \nu_s ) \subset V_{-j}$, i.e $\nu_s \in \Db(V_{-j})$; and,
    \item for $s \in (0,1)$, the mapping $s \mapsto \nu_s$ is continuous for $s\in (t_j, t_{j+1})$ for some $j$. %
\end{enumerate}
\end{definition}

The continuous time dependence in Definition \ref{def:multi-res-bridge} plays a movie of the measure $\nu_{0}$ supported on $V_0$ growing to $\nu_{\infty}$, a measure on images with infinite resolution.
At a time interval $[t_j,t_{j+1})$, the space $V_{-j}$ which the measure is supported on is fixed.
We may therefore define a finite-dimensional model transporting probability measures within $V_{-j}$, but at $t_{j+1}$ the support flows over to $V_{-j-1}$.
Given a multi-resolution hierarchy, we may glue these finite models, each acting on a disjoint time interval, together in a unified fashion.
In Theorem \ref{thm:id-diff} we show this for the example of a continuous-time multi-resolution diffusion process truncated up until some time $t_{J} = T \in (0,1)$ and in the \emph{standard basis} discussed in \S\ref{sec:The regularisation property imposed by U-Net architectures with average pooling}, which will be useful when viewing HVAEs as discretisations of diffusion processes on functions in \S\ref{sec:Example: HVAEs are Sum Representation Diffusion Discretisations}. %

\begin{theorem}\label{thm:id-diff}
Let $B_j: [t_j, t_{j+1}) \times \Db(V_{-j}) \mapsto \Db(V_{-j})$ be a linear operator (such as a diffusion transition kernel, see Appendix \ref{app:Proofs}) for $j<J$ with coefficients $\mu^{(j)}, \sigma^{(j)} :[t_j, t_{j+1}) \times V_{-j} \mapsto V_{-j}$, and define the natural extensions within $V_{-J}$ in bold, i.e. $\bm{B}_j \coloneqq B_{j} \oplus \bm{I}_{V_{-j}^{\perp} }$.
Then the operator $\bm{B}: [0,T] \times \Db(V_{-J}) \mapsto \Db(V_{-J})$ and the coefficients $\bm{\mu}, \bm{\sigma} :[0,T] \times V_{-J} \mapsto V_{-J}$ given by

{\centering
 $ \displaystyle
\begin{aligned} 
    \bm{B} \coloneqq \sum_{j = 0}^{J} \1_{[t_j, t_{j+1})} \cdot \bm{B}_j,
    &&
    \bm{\mu} \coloneqq \sum_{j = 0}^{J} \1_{[t_j, t_{j+1})} \cdot \bm{\mu}^{(j)},
    &&
    \bm{\sigma} \coloneqq \sum_{j = 0}^{J} \1_{[t_j, t_{j+1})} \cdot \bm{\sigma}^{(j)},
\end{aligned}
 $ 
\par}
induce a multi-resolution bridge of measures from the dynamics for $t \in [0,T]$ and on the standard basis as $dZ_t = \bm{\mu}_t(Z_t) dt + \bm{\sigma}_t(Z_t) dW_t$ (see Appendix \ref{app:U-Nets in V_-J} for details) for $Z_t \in V_{-j}$ for $t \in [t_{j},t_{j+1})$, i.e. a multi-resolution diffusion process. 
\end{theorem}

\begin{wrapfigure}[24]{r}{.3\textwidth}
\vspace{-1em}
\centering
\includegraphics[width=1\linewidth]{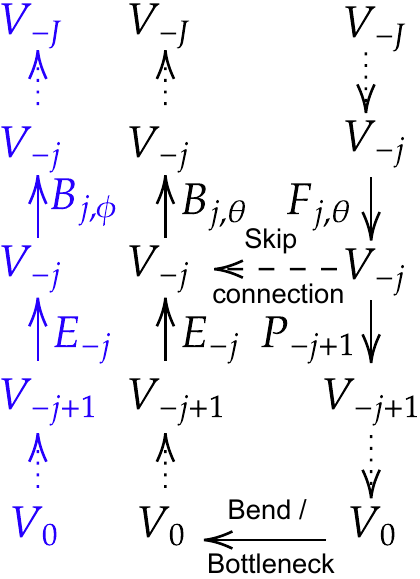}
\caption{
A U-Net in our multi-resolution framework. See Appendix \ref{app:U-Net Model} for details.
}  %
\label{fig:U-Net_short}

\end{wrapfigure}

The concept of a multi-resolution bridge will become important in Section \ref{sec:The regularisation property imposed by U-Net architectures with average pooling} where we will show that current U-Net bottleneck structures used for unconditional sampling impose a multi-resolution bridge on the modelled densities. 
To preface this, we here provide a description of a U-Net within our framework, illustrated in \ref{fig:U-Net_short}.
Consider $B_{j,\theta}, F_{j,\theta} :\Db(V_{-j}) \rightarrow \Db(V_{-j})$ as the forwards and backwards passes of a U-Net on resolution $j$.
Further, let $P_{-j+1} : \Db(V_{-j}) \rightarrow \Db(V_{-j+1})$ and $E_{-j}: \Db(V_{-j+1}) \rightarrow \Db(V_{-j})$ be the projection (here: average pooling) and embedding maps (e.g. interpolation), respectively.
When using an $L^2$-reconstruction error, a U-Net~\cite{ronneberger2015u} architecture implicitly learns a sequence of models $\bold{B}_{j,\phi} : \Db(V_{-j+1}) \times \Db(V_{-j+1}^{\perp}) \mapsto \Db(V_{-j})$ due to the orthogonal decomposition $V_{-j} = V_{-j+1} \oplus U_{-j+1}$ where $ U_{-j+1} \coloneqq V_{-j} \cap V_{-j+1}^{\perp} $.
The backwards operator for the U-Net has a (bottleneck) input from $ \Db(V_{-j+1})$ and a (skip) input yielding information from $\Db(V_{-j+1}^{\perp})$. 
A simple \emph{bottleneck} map $\mathit{U}_{j,\theta}:\Db(V_{-j}) \rightarrow \Db(V_{-j})$ (without skip connection) is given by
\begin{align}
    \mathit{U}_{j,\theta} \coloneqq B_{j, \theta} \circ E_{-j} \circ P_{-j+1} \circ  F_{j, \theta},
\end{align}
and a U-Net bottleneck with skip connection is
\begin{align}
    \mathbf{U}_{j,\phi} \coloneqq B_{j, \phi} (E_{-j} \circ P_{-j+1} \circ  F_{j, \theta} ,  F_{j, \theta} ).
\end{align}
In HVAEs, the map $ \mathbf{U}_{j,\phi}:\Db(V_{-j}) \rightarrow \Db(V_{-j})$ is trained to be the identity by minimising reconstruction error, and further shall approximate $\mathit{U}_{j,\theta} \approx \mathbf{U}_{j,\phi}$ via a KL divergence. 
The $L^2$-reconstruction error for $ \mathbf{U}_{j,\phi}$ has an orthogonal partition of the inputs from $V_{-j+1} \times V_{-j}$, hence the only new subspace added is $U_{-j+1}$. 
As each orthogonal $U_{-j+1}$ is added sequentially in HVAEs, the skip connections induce a multi-resolution structure of this hierarchical neural network structure.
What we will investigate in Theorem \ref{thm:truncation} is the regularisation imposed on this partitioning by enforcing $\mathit{U}_{j,\theta} \approx \mathbf{U}_{j,\phi}$, as is often enforced for generative models with VAEs.

\subsection{The regularisation property imposed by U-Net architectures with average pooling}
\label{sec:The regularisation property imposed by U-Net architectures with average pooling}

Having defined U-Net architectures within our multi-resolution framework, we are now interested in the regularisation they impose. 
We do so by analysing a U-Net when skip connections are absent, so that we may better understand what information is transferred through each skip connection when they are present.
In practice, a pixel representation of images is used when training U-Nets, which we henceforth call the \emph{standard basis} (see \ref{app:Dimension Reduction Conjucacy}, Eq. \eqref{eq:standard_basis}).
The standard basis is not convenient to derive theoretical results.
It is instead preferable to use a basis natural to the multi-resolution bridge imposed by a U-Net with a corresponding projection operation, which for average pooling is the \emph{Haar (wavelet) basis} \cite{haar1909theorie} (see Appendix \ref{app:Dimension Reduction Conjucacy}).
The Haar basis, like a Fourier basis, is an orthonormal basis of $L^2(\Xb)$ which has desirable $L^2$-approximation properties.
We formalise this in Theorem~\ref{thm:conj}
which states that the dimension reduction operation of average pooling in the standard basis is a conjugate operation to co-ordinate projection within the Haar basis (details are provided in Appendix \ref{app:Dimension Reduction Conjucacy}).

\begin{theorem}\label{thm:conj}
Given $V_{-j}$ as in Definition \ref{def:multi-res-approx-space}, let $x \in V_{-j}$ be represented in the standard basis $\bE_j$ and Haar basis $\bPsi_j$.
Let $\pi_j : \bE_j \mapsto \bPsi_j$ be the change of basis map illustrated in Fig. \ref{fig:change_basis_map}, then we have the conjugacy $\pi_{j-1} \circ \text{pool}_{-j,-j+1} = \proj_{V_{-j+1}} \circ \pi_{j}$. 
\end{theorem}
\begin{wrapfigure}[9]{r}{.4\textwidth}
\vspace{-2em}
\begin{tikzcd}
	{(V_{-j},\bE_j)} && {(V_{-j+1},\bE_{j-1})} \\
	{(V_{-j},\bPsi_j)} && {(V_{-j+1},\bPsi_{j-1})}
	\arrow["{\text{pool}_{-j,-j+1}}", from=1-1, to=1-3]
	\arrow["{\proj_{V_{-j+1}}}"', dashed, from=2-1, to=2-3]
	\arrow["{\pi_{j-1}}"', from=1-3, to=2-3]
	\arrow["{\pi_j^{-1}}"', from=2-1, to=1-1]
\end{tikzcd}
\caption{%
The function space $V_{-j}$ remains the same, but the basis changes under $\pi_j$.
}  %
\label{fig:change_basis_map}
\end{wrapfigure}
Theorem~\ref{thm:conj} means that if we project an image from $V_{-j}$ to $V_{-j+1}$ in the Haar wavelet basis, we can alternatively view this as changing to the standard basis via $\pi_j^{-1}$, performing average pooling, and reverting back via $\pi_{j-1}$~(see Figure~\ref{fig:change_basis_map}).
This is important because the Haar basis is orthonormal, which in Theorem \ref{thm:truncation} allows us to precisely quantify what information is lost with average pooling.

\begin{theorem}\label{thm:truncation}
Let $\{V_{-j}\}_{j=0}^J$ be a multi-resolution hierarchy of $V_{-J}$ where $V_{-j} = V_{-j+1} \oplus U_{-j+1}$, and further, let $F_{j,\phi}, \, B_{j,\theta} : \Db(V_{-j}) \mapsto \Db(V_{-j})$ be such that $ B_{j,\theta} F_{j,\phi} = I$ with parameters $\phi$ and $\theta$.
Define $\bm{F}_{j_1|j_2,\phi} \coloneqq \bm{F}_{j_1,\phi} \circ  \cdots  \circ  \bm{F}_{j_2,\phi}$ by $\bm{F}_{j,\phi}: \Db(V_{-j}) \mapsto \Db(V_{-j+1})$ where $\bm{F}_{j,\phi} \coloneqq \proj_{V_{-j+1}} \circ F_{j,\phi}$, and analogously define $\bm{B}_{j_1|j_2,\theta}$ with
$\bm{B}_{j,\theta} \coloneqq B_{j,\theta} \circ \embd_{V_{-j} }$. 
Then, the sequence $\{ \bm{B}_{1|j,\theta} (\bm{F}_{1|J,\phi} \nu_{J}) \}_{j=0}^{J}$ forms a discrete multi-resolution bridge between $\bm{F}_{1|J,\phi} \nu_{J}$ and $\bm{B}_{1|J,\theta} \bm{F}_{1|J,\phi} \nu_{J}$ at times $\{t_j\}_{j=1}^{J}$, and
\vspace{-0.5em}
\begin{align}\label{eq:truncation_bound}
    \sum_{j=0}^{J}
    \Eb_{X_{t_j} \sim \nu_{j} }{\norm{\text{\emph{proj}}_{U_{-j+1} }X_{t_j} }_2^2}/{\norm{\bm{F}_{j|J,\phi}}_2^2} 
    \leq 
    (\mathcal{W}_2( \bm{B}_{1|J,\theta} \bm{F}_{1|J,\phi} \nu_{J} , \nu_{J}))^2,
\end{align}
\vspace{-0.5em}
where $\mathcal{W}_2$ is the Wasserstein-$2$ metric and $\norm{\bm{F}_{j|J,\phi}}_2$ is the Lipschitz constant of $\bm{F}_{j|J,\phi}$.
\end{theorem}

Theorem \ref{thm:truncation} states that the bottleneck component of a U-Net
pushes the latent data distribution to a finite multi-resolution basis, specifically a Haar basis when average pooling is used.
To see this, note that the RHS of Eq. \eqref{eq:truncation_bound} is itself upper-bounded by the $L^2$-reconstruction error.
This is because the Wasserstein-2 distance finds the infinimum over all possible couplings between the data and the `reconstruction' measure, hence any coupling (induced by the learned model) bounds it. 
Note that models using a U-Net, for instance HVAEs or diffusion models, either directly or indirectly optimise for low reconstruction error in their loss function.
The LHS of Eq. \eqref{eq:truncation_bound} represents what percentage of our data enters the orthogonal subspaces $\{U_{-j}\}_{j=0}^{J}$ which are (by Theorem \ref{thm:conj}) discarded by the bottleneck structure when using a U-Net architecture with average pooling. 
Theorem \ref{thm:truncation} thus shows that as we minimise the reconstruction error during training, we minimise the percentage of our data transported to the orthogonal sub-spaces $\{U_{-j}\}_{j=0}^J$.
Consequently, the bottleneck architecture implicitly decomposes our data into a Haar wavelet decomposition, and when the skip connections are absent (like in a traditional auto-encoder) our network learns to compress the discarded subspaces $U_{-j}$.
This characterises the regularisation imposed by a U-Net in the absence of skip connections.

These results suggest that U-Nets with average pooling provide a direct alternative to Fourier features~\cite{Kingma2021VariationalModels,tancik2020fourier,wang2020high,rahimi2007random} which impose a Fourier basis, an alternative orthogonal basis on $L^2(\Xb)$, as with skip connections the U-Net adds each subspace $U_{-j}$ sequentially. 
However, unlike Fourier bases, there are in fact a multitude of wavelet bases which are all encompassed by the multi-resolution framework, and in particular, Theorem \ref{thm:truncation} pertains to all of them for the bottleneck structure. 
This opens the door to exploring conjugacy operations beyond average pooling induced by other wavelet bases optimised for specific data types.  %

\subsection{Example: HVAEs as Diffusion Discretisations}
\label{sec:Example: HVAEs are Sum Representation Diffusion Discretisations}

To show what practical inferences we can derive from our multi-resolution framework, we apply it to analyse state-of-the-art HVAE architectures (see Appendix \ref{app:Hierarchical VAEs} for an introduction), identifying parameter redundancies and instabilities.
Here and in our experiments, we focus on VDVAEs~\cite{Child2020VeryImages}. 
We provide similar results for Markovian HVAEs~\cite{LVAE, burda2015importance} %
and NVAEs~\cite{Vahdat2020NVAE:Autoencoder} (see \S~\ref{sec:Related_work}) in Appendix \ref{app:Forward Euler Diffusion Approximations}.

We start by inspecting VDVAEs.
As we show next, we can tie the computations in VDVAE cells to the (forward and backward) operators $F_{j,\phi}$ and $B_{j,\theta}$ within our framework and identify them as a type of two-step forward Euler discretisation of a diffusion process.
When used with a U-Net, as is done in VDVAE \cite{Child2020VeryImages}, this creates a \emph{multi-resolution diffusion bridge} by Theorem \ref{thm:vdvae_sde}.

\begin{theorem} 
\label{thm:vdvae_sde}
Let $t_{J} \coloneqq T \in (0,1)$ and consider (the $p_{\theta}$ backward pass) $\bm{B}_{\theta,1|J}: \Db(V_{-J}) \mapsto \Db(V_{0})$ given in multi-resolution Markov process in the standard basis:
\begin{align}
d Z_t = (\bmu_{1,t} (Z_t) + \bmu_{2,t} (Z_t) )dt + \bsig_t(Z_t) dW_t,  \label{eq:coupled_sdes}  
\end{align}
where  $\proj_{U_{-j}}Z_{t_j} = 0,  \ \|Z_t \|_2 > \|Z_s \|_2$ with $0\leq s<t \leq T$ and for a measure $\nu_{J} \in \Db(V_{-J})$ we have $X_T, \, Z_0 \sim \bm{F}_{\phi,J|1}\nu_{J} = \delta_{\{0\}}$.
Then, VDVAEs approximates this process, and its residual cells are a type of two-step forward Euler discretisation of this Stochastic Differential Equation (SDE). %
\end{theorem}

To better understand Theorem \ref{thm:vdvae_sde}, we visualise its residual cell structure of VDVAEs and the corresponding discretisation steps in Fig. \ref{fig:hvae_cell_vdvae}, and together those of NVAEs and Markovian HVAEs in Appendix \ref{app:Forward Euler Diffusion Approximations}, Fig. \ref{fig:hvae_cells_all}.
Note that this process is Markov and increasing in the $Z_i$ variables. 
Similar processes have been empirically observed as efficient first-order approximates to higher-order chains, for example the memory state in LSTMs \cite{hochreiter1997long}.
Further, VDVAEs and NVAEs are even claimed to be high-order chains (see Eqs. (2,3) in \cite{Child2020VeryImages} and Eq. (1) in \cite{Vahdat2020NVAE:Autoencoder}), despite only approximating this with a accumulative process.  %

\begin{wrapfigure}[21]{r}{.5\textwidth}
\vspace{-1.5em}
\centering
\includegraphics[width=1\linewidth]{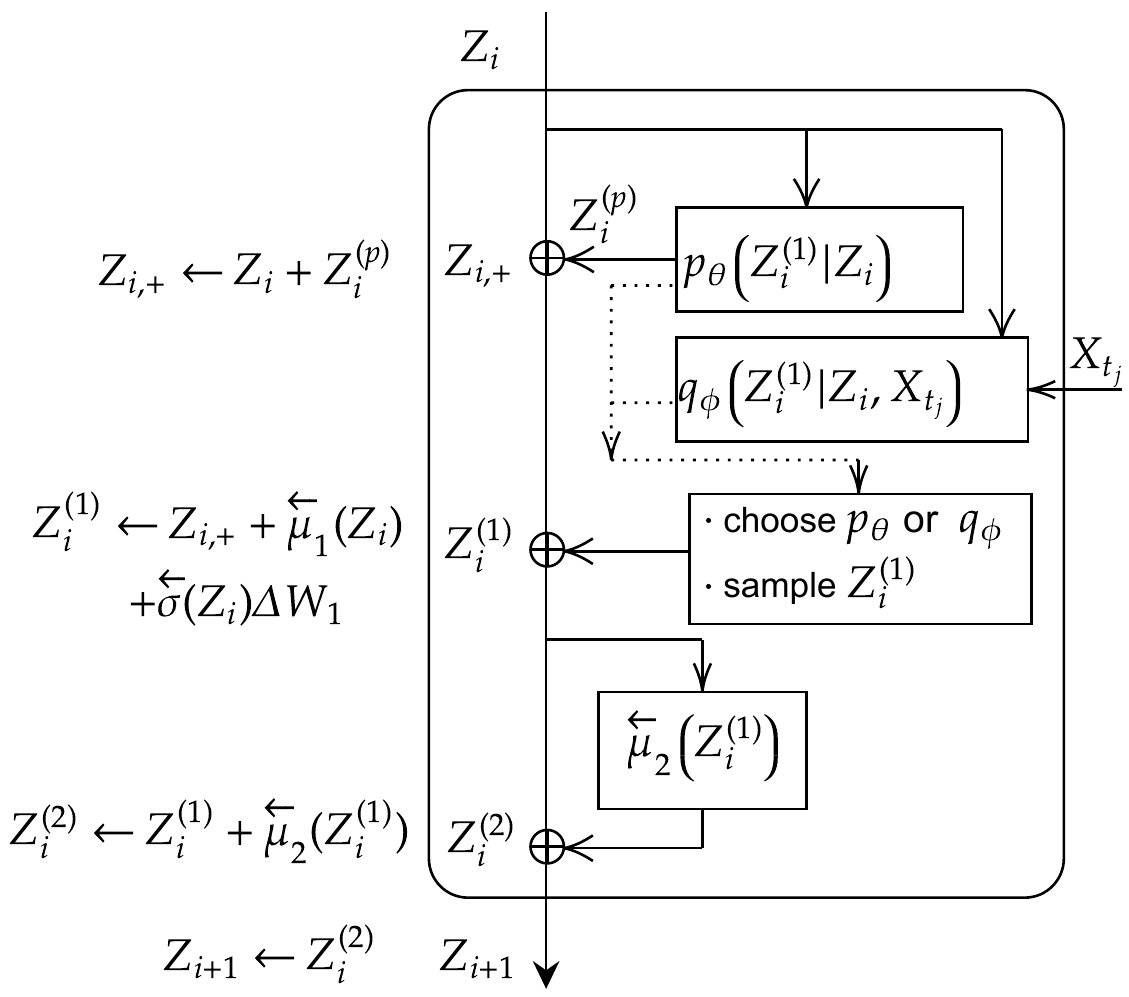}
\caption{
The VDVAE \cite{Child2020VeryImages} cell is a type of two-step forward Euler discretisations of the continuous-time diffusion process in Eq. \ref{eq:coupled_sdes}.
See Fig. \ref{fig:hvae_cells_all} for similar schemas on NVAE \cite{Vahdat2020NVAE:Autoencoder} and Markovian HVAE \cite{LVAE, burda2015importance}.
}  
\label{fig:hvae_cell_vdvae}
\end{wrapfigure}

To show how VDVAEs impose the growth of the $Z_t$, we prove that the bottleneck component of VDVAE's U-Net enforces $Z_0 =0$.
This is done by identifying that the measure $\nu_0$, which a VDVAE connects to the data $\nu_{\infty}$ via a multi-resolution bridge, is a point mass on the zero function.
Consequently the backward pass must grow from this, and the network learns this in a monotonic manner as we later confirm in our experiments (see \S\ref{sec:HVAEs secretly represent time and make use of it}).

\begin{theorem}\label{thm:discrete-VDVAE}
Consider the SDE in Eq.~\eqref{eq:coupled_sdes}, trained through the ELBO in Eq. \ref{eq:hvae_elbo}.
Let $\tilde{\nu}_{J}$ denote the data measure and $\nu_{0}= \delta_{\{0\}}$ be the initial multi-resolution bridge measure imposed by VDVAEs.
If $q_{\phi,j}$ and $p_{\theta,j}$ are the densities of $B_{\phi,1|j}\bm{F}_{J|1}\tilde{\nu}_{J}$ and  $B_{\theta,1|j}\nu_{0}$ respectively, then a VDVAE optimises the boundary condition $\min_{\theta,\phi} KL(q_{\phi,0,1} || q_{\phi,0}p_{\theta,1} )$, where a double index indicates the joint distribution.
\end{theorem}

Theorem \ref{thm:discrete-VDVAE} states that the VDVAE architecture forms multi-resolution bridge with the dynamics of Eq. \eqref{eq:coupled_sdes}, and connects our data distribution to the trivial measure on $V_0$: a Dirac mass at $0$ as the pooling here cascades completely to $V_0$.
From this insight, we can draw conclusions on instabilities and on parameter redundancies of this HVAE cell.
There are two major instabilities in this discretisation. 
First, the imposed $\nu_0$ is disastrously unstable as it enforces a data set, with potentially complicated topology to derive from a point-mass in $U_{-j}$ at each $t = t_j$, and we observe the resulting sampling instability in our experiments in \S\ref{sec:Sampling instabilities in HVAEs}.
We note that similar arguments are applicable in settings without a latent hierarchy imposed by a U-Net, see for instance \cite{cornish2020relaxing}.
The VDVAE architecture does, however, bolster this rate through the $Z_{i,+}^{(\sigma)}$ term, which is absent in NVAEs~\cite{Vahdat2020NVAE:Autoencoder}, in the discretisation steps of the residual cell.
We empirically observe this controlled backward error in Fig. \ref{fig:magnitude_increase_mnist} [Right].
We refer to Fig. \ref{fig:hvae_cells_all} for a detailed comparison of HVAE cells and their corresponding discretisation of the coupled SDE in Eq. \eqref{eq:coupled_sdes}.

Moreover, the current form of VDVAEs is over-parameterised and not informed by this continuous-time formulation. 
The continuous time analogue of VDVAEs \cite{Child2020VeryImages} in Theorem \ref{thm:vdvae_sde} has time dependent coefficients $\bmu_{t,1}, \bmu_{t,2}, \bsig_t$.
We hypothesise that the increasing diffusion process in $Z_i$ implicitly encodes time. 
Hence, explicitly representing this in the model, for instance via ResNet blocks with independent parameterisations at every time step, is redundant, and a time-homogeneous model (see Appendix \ref{app:Time-homogenuous model} for a precise formulation)---practically speaking, performing weight-sharing across time time steps/layers---has the same expressivity, but requires far fewer parameters than the state-of-the-art VDVAE. 
It is worth noting that such a time-homogeneous model would make the parameterisation of HVAEs more similar to the recently popular (score-based) diffusion models \cite{sohl2015deep,ho2020denoising} which perform weight-sharing across all time steps.

\section{Experiments}
\label{sec:Experiments}

\begin{figure*}[t]
    \centering
    \includegraphics[width=.42\linewidth]{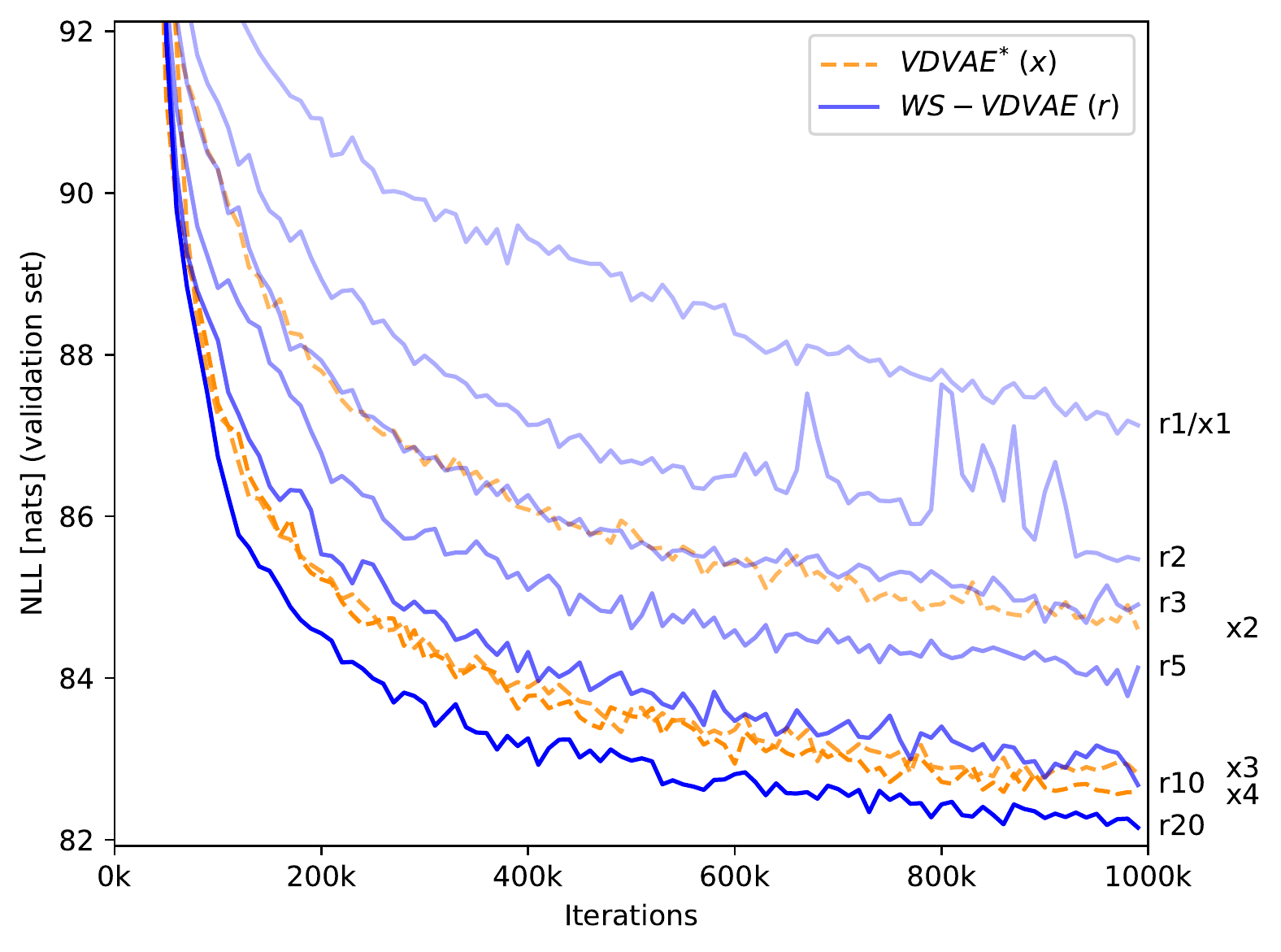}    %
    \includegraphics[width=.42\linewidth]{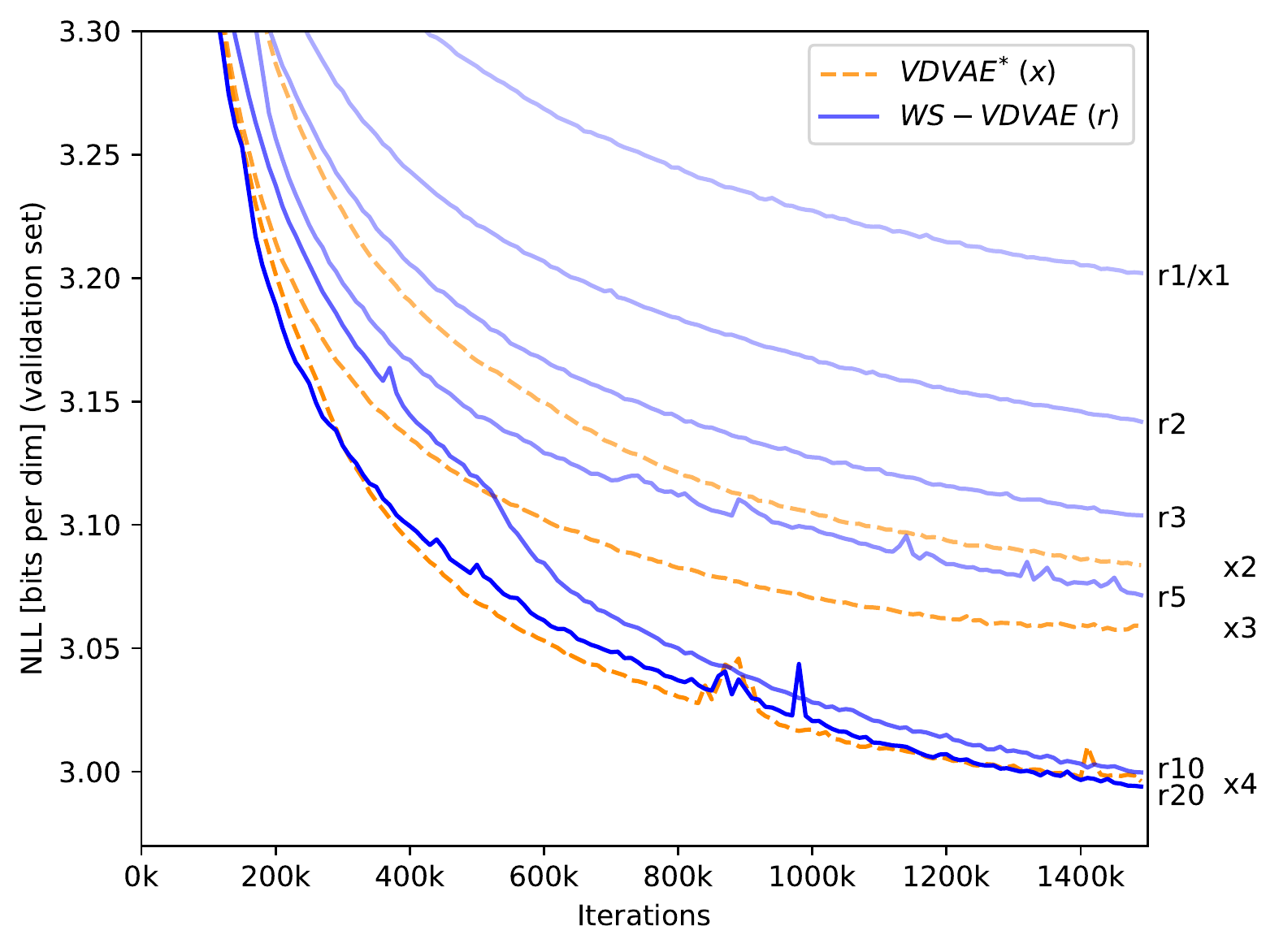}   %
    \caption{
    A small-scale study on parameter efficiency of HVAEs.
    We compare models with with 1,2,3 and 4 parameterised blocks per resolution ($\{x1, x2, x3, x4\}$) against models with a single parameterised block per resolution weight-shared $\{2,3,5,10,20\}$ times ($\{r2,r3,r5,r10,r20\}$).
    We report NLL ($\downarrow$) measured on the validation set of MNIST [left] and CIFAR10 [right].
    NLL performance increases with more weight-sharing repetitions and surpasses models without weight-sharing but with more parameters. 
    }
    \label{fig:small_scale_param_efficiency}
\end{figure*}

\begin{wraptable}{r}{9cm}
\caption{
A large-scale study of parameter efficiency in HVAEs.
We compare our runs of VDVAE with original hyperparameters \cite{Child2020VeryImages} (\texttt{VDVAE$^*$}) against our weight-shared VDVAE (\texttt{WS-VDVAE}). 
While \texttt{WS-VDVAE}s have improved parameter efficiency by a factor of 2, they reach similar NLL as \texttt{VDVAE$^*$} with the simple modification inspired by our framework (weight sharing).
We note that a parameter count cannot be provided for VDM \cite{Kingma2021VariationalModels} as the code is not public and the manuscript does not specify it.
}
\centering
\begin{tabular}{ccccc} \toprule  
Dataset & Method &  Type & \#Params  &  NLL $\downarrow$ \\ \midrule   %
\multirow{4}{*}[0cm]{\rotatebox{90}{\parbox{.98cm}{\textbf{MNIST} $28\times28$}}}  %
& \texttt{WS-VDVAE} (ours) & VAE & \textbf{232k}   & $\leq 79.98$ \\   %
& \texttt{VDVAE$^*$} (ours) & VAE & 339k   & $\leq 80.14$ \\   %
& NVAE \cite{Vahdat2020NVAE:Autoencoder} & VAE & 33m   &  $\leq 78.01$ \\  %
\midrule
\multirow{6}{*}[0cm]{\rotatebox[origin=c]{90}{\parbox{1.5cm}{\textbf{CIFAR10} $32 \times 32$}}} 
& \texttt{WS-VDVAE}  (ours) & VAE & \textbf{25m} & $\leq2.88$   \\  %
& \texttt{WS-VDVAE}  (ours) & VAE & 39m   & $\leq 2.83$ \\  %
& \texttt{VDVAE$^*$} (ours) & VAE & 39m   & $\leq 2.87$ \\  %
& NVAE \cite{Vahdat2020NVAE:Autoencoder} & VAE & 131m   & $\leq2.91$ \\  %
& VDVAE \cite{Child2020VeryImages} & VAE & 39m  &  $\leq2.87$ \\  %
& VDM \cite{Kingma2021VariationalModels} & Diff & --  & $\leq 2.65$ \\  %
\midrule
\multirow{5}{*}[0cm]{\rotatebox{90}{\parbox{1.5cm}{\textbf{ImageNet} $32 \times 32$}}} 
& \texttt{WS-VDVAE}  (ours) & VAE & \textbf{55m}   & $\leq 3.68$ \\  %
& \texttt{WS-VDVAE}  (ours) & VAE & 85m   & $\leq 3.65$ \\  %
& \texttt{VDVAE$^*$} (ours) & VAE & 119m   & $\leq 3.67$ \\  %
& NVAE \cite{Vahdat2020NVAE:Autoencoder} & VAE & 268m   & $\leq3.92$ \\   %
& VDVAE \cite{Child2020VeryImages} & VAE & 119m   & $\leq3.80$ \\   %
& VDM \cite{Kingma2021VariationalModels} & Diff  & --  & $\leq3.72$ \\  %
\midrule
\multirow{3}{*}[0cm]{\rotatebox{90}{\parbox{1.1cm}{\textbf{CelebA} $64 \times 64$}}} 
& \texttt{WS-VDVAE}  (ours) & VAE & \textbf{75m}  & $\leq 2.02$  \\  %
 & \texttt{VDVAE$^*$} (ours) & VAE & 125m  & $\leq2.02$  \\  %
& NVAE \cite{Vahdat2020NVAE:Autoencoder} & VAE & 153m   & $\leq2.03$ \\  %
\bottomrule 
\end{tabular}
\label{tab:sota_quant_comp}
\end{wraptable}

In the following we probe the theoretical understanding of HVAEs gained through our framework, demonstrating its utility in four experimental analyses: (\emph{a}) Improving parameter efficiency in HVAEs, (\emph{b}) Time representation in HVAEs and how they make use of it, (\emph{c}) Sampling instabilities in HVAEs, and (\emph{d}) Ablation studies.

We train HVAEs using VDVAE~\cite{Child2020VeryImages} as the basis model on five datasets: MNIST~\cite{lecun2010mnist}, CIFAR10~\cite{krizhevsky2009learning}, two downsampled versions of ImageNet~\cite{deng2009imagenet,chrabaszcz2017downsampled}, and CelebA~\cite{liu2015faceattributes}, splitting each into a training, validation and test set (see Appendix \ref{app:Datasets} for details).
In general, reported numeric values refer to Negative Log-Likelihood~(NLL) in nats (MNIST) or bits per dim (all other datasets) on the test set at model convergence, if not stated otherwise.
We note that performance on the validation and test set have similar trends in general.
An optional \textit{gradient checkpointing} implementation to trade in GPU memory for compute is discussed in Appendix~\ref{app:Model and training details}.
Appendices \ref{app:Model and training details} and \ref{app:Experimental details} define the HVAE models we train, i.e. $p_\theta(\v{z}_L), p_\theta(\v{z}_l|\v{z}_{>l}), q_\phi(\v{z}_L|\v{x}), q_\phi(\v{z}_l|\v{z}_{>l}, \v{x})$ and  $p_\theta(\v{x}|\vv{z})$, and present additional experimental details and results.
We provide our PyTorch code base at \href{https://github.com/FabianFalck/unet-vdvae}{https://github.com/FabianFalck/unet-vdvae} (see Appendix~\ref{app:Code, computational resources, existing assets used} for details).

\subsection{``More from less'': Improving parameter efficiency in HVAEs}  %
\label{sec:``More from less'': Parameter efficiency in HVAEs}

In \S\ref{sec:Example: HVAEs are Sum Representation Diffusion Discretisations}, we hypothesised that a time-homogeneous model has the same expressivity as a model with time-dependent coefficients, yet uses much less parameters.
We start demonstrating this effect by weight-sharing ResNet blocks across time on a small scale. 
In Fig.~\ref{fig:small_scale_param_efficiency}, we train HVAEs on MNIST and CIFAR10 with $\{1, 2, 3, 4\}$ ResNet blocks (referred to as \{\texttt{x1}, \texttt{x2}, \texttt{x3}, \texttt{x4}\}) in each resolution with spatial dimensions $\{32^2, 16^2, 8^2, 4^2, 1^2\}$ (\texttt{VDVAE$^*$}), and compare their performance when weight-sharing a single parameterised block per resolution $\{2,3,5,10,20\}$ times (referred to as \{\texttt{r2},\texttt{r3},\texttt{r5},\texttt{r10},\texttt{r20}\}; \texttt{WS-VDVAE}), excluding projection and embedding blocks.
As hypothesised by our framework, yet very surprising in HVAEs, NLL after 1m iterations measured on the validation set gradually increases the more often blocks are repeated even though all weight-sharing models have an identical parameter count to the $x1$ model (MNIST: 107k, CIFAR10: 8.7m).  %
Furthermore, the weight-sharing models often outperform or reach equal NLLs compared to \texttt{x2}, \texttt{x3}, \texttt{x4}, all of which have more parameters (MNIST: 140k; 173k; 206k. CIFAR10: 13.0m; 17.3m; 21.6m), yet fewer activations, latent variables, and number of timesteps at which the coupled SDE in Eq. \eqref{eq:coupled_sdes} is discretised.

We now scale these findings up to large-scale hyperparameter configurations. 
We train VDVAE closely following the state-of-the-art hyperparameter configurations in \cite{Child2020VeryImages}, specifically with the same number of parameterised blocks and without weight-sharing (\texttt{VDVAE$^*$}), and compare them against models with weight-sharing (\texttt{WS-VDVAE}) and fewer parameters, i.e. fewer parameterised blocks, in Table~\ref{tab:sota_quant_comp}. %
On all four datasets, the weight-shared models achieve similar NLLs with fewer parameters compared to their counterparts without weight-sharing:
We use $32 \%$, $36 \%$, $54 \%$, and $40 \%$ less parameters on the four datasets reported in Table \ref{tab:sota_quant_comp}, respectively.
For the larger runs, weight-sharing has diminishing returns on NLL as these already have many discretisation steps.
To the best of our knowledge, our models achieve a new state-of-the-art performance in terms of NLL compared to any HVAE on CIFAR10, ImageNet32 and CelebA.
Furthermore, our \texttt{WS-VDVAE} models have stochastic depths of 57, 105, 235, 125, respectively, the highest ever trained.  
In spite of these results, it is worth noting that current HVAEs, and VDVAE in particular remains notoriously unstable to train, partly due to the instabilities identified in Theorem \ref{thm:discrete-VDVAE}, and finding the right hyperparameters helps, but cannot solve this.

\subsection{HVAEs secretly represent time and make use of it}  %
\label{sec:HVAEs secretly represent time and make use of it}

\begin{figure}[t]
    \centering
    \includegraphics[width=.42\linewidth]{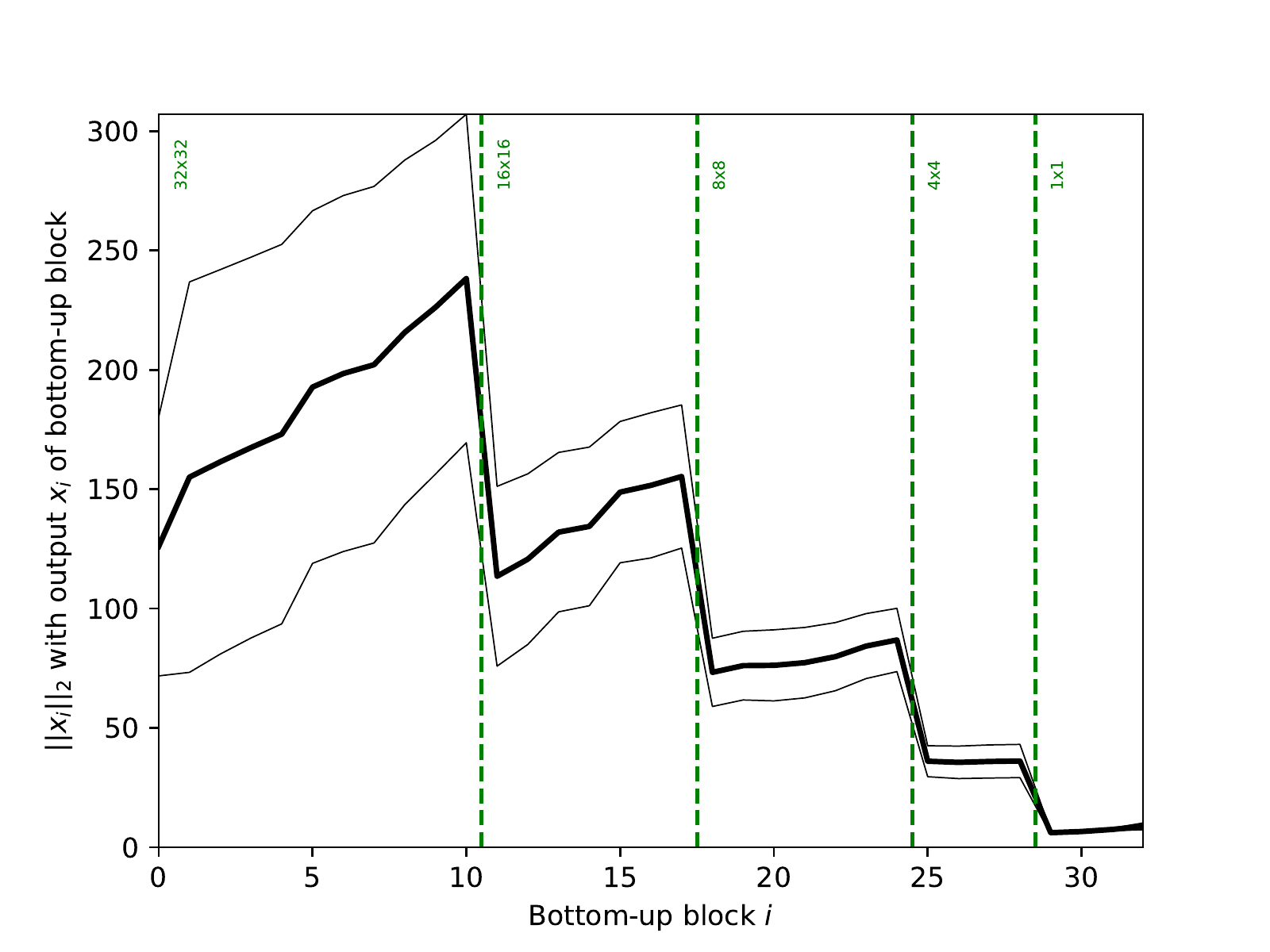}  
    \includegraphics[width=.42\linewidth]{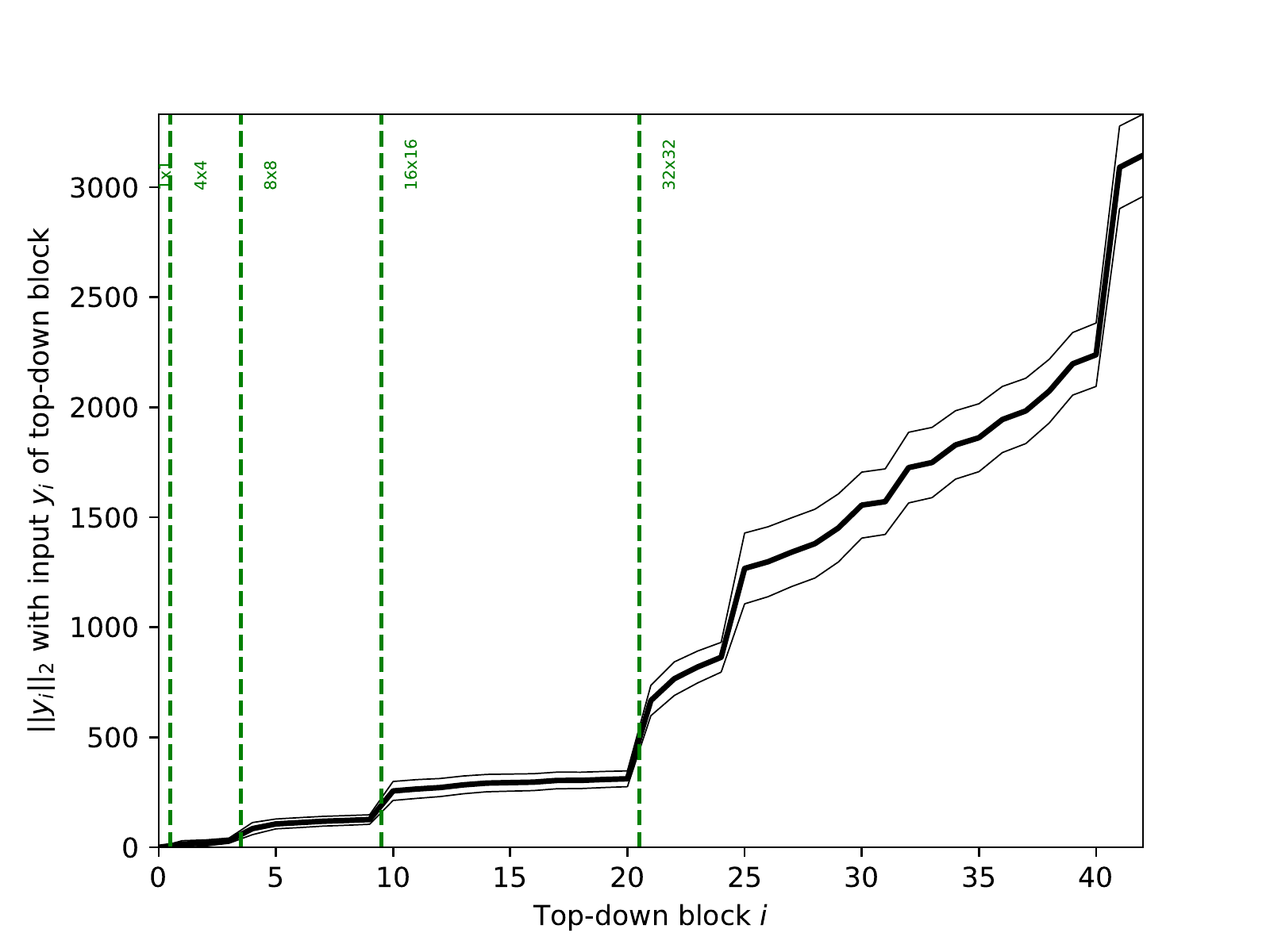}
    \caption{
    HVAEs secretly represent a notion of time: 
    We measure the $L_2$-norm of the residual state for the [Left] forward/bottom-up pass and the [Right] backward/top-down pass over 10 batches with 100 data points each.
    In both plots, the thick, central line refers to the average and the thin, outer lines refer to $\pm 2$ standard deviations.   %
    }
    \label{fig:magnitude_increase_mnist}
\end{figure}

In \S\ref{sec:``More from less'': Parameter efficiency in HVAEs}, we showed how we can exploit insight on HVAEs through our framework to make HVAEs more parameter efficient. 
We now want to explain and understand this behavior further.
In Fig.~\ref{fig:magnitude_increase_mnist}, we measure $\|Z_i \|_2$, the $L_2$-norm of the residual state at every backward/top-down block with index i, over several batches for models trained on MNIST (see Appendix~\ref{app:add_exp_details_results_HVAEs secretly represent time and make use of it} for the corresponding figure of the forward/bottom-up pass, and similar results on CIFAR10 and ImageNet32). 
On average, we experience an increase in the state norm across time in every resolution, interleaved by discontinuous `jumps' at the resolution transitions (projection or embedding) where the dimension of the residual state changes. 
This supports our claim in \S\ref{sec:multi-res} that HVAEs discretise multi-resolution diffusion processes which are increasing in the $Z_i$ variables, and hence learn to represent a notion of time in their residual state.

It is now straightforward to ask how HVAEs benefit from this time representation during training: 
As we show in Table~\ref{tab:state_norm}, when normalising the state by its norm at every forward and backward block during training, i.e. forcing a ``flat line'' in Fig.~\ref{fig:magnitude_increase_mnist} [Left], learning deteriorates after a short while, resulting in poor NLL results compared to the runs with a regular, non-normalised residual state.  %
This evidence confirms our earlier stated hypothesis: The time representation in ResNet-based HVAEs encodes information which recent HVAEs heavily rely on during learning.

\subsection{Sampling instabilities in HVAEs}
\label{sec:Sampling instabilities in HVAEs}

\begin{wraptable}{R}{.3\textwidth}  %
\caption{
NLL of HVAEs with and without normalisation of the residual state $Z_i$.
}
\centering
\begin{tabular}{ccc} \toprule  %
Residual state & NLL \\ 
\midrule 
\textbf{MNIST}  \\  
Normalised (\blackcross) & $\leq 464.68$  \\  %
Non-normalised & $\leq 81.69$ \\  %
\midrule 
\textbf{CIFAR10}  \\    
Normalised (\blackcross) & $\leq 6.80$  \\  %
Non-normalised & $\leq 2.93$  \\  %
\midrule 
\textbf{ImageNet}  \\   
Normalised & $\leq 6.76$ \\  %
Non-normalised & $\leq 3.68$ \\  %
\bottomrule
\end{tabular}
\label{tab:state_norm}
\end{wraptable}

High fidelity unconditional samples of faces, e.g. from models trained on CelebA, cover the front pages of state-of-the-art HVAE papers~\cite{Child2020VeryImages,Vahdat2020NVAE:Autoencoder}.
Here, we question whether face datasets are an appropriate benchmark for HVAEs.
In Theorem \ref{thm:discrete-VDVAE}, we identified the aforementioned state-of-the-art HVAEs as flow from a point mass, hypothesising instabilities during sampling.
And indeed, when sampling from our trained \texttt{VDVAE$^*$} with state-of-the-art configurations, we observe high fidelity and diversity samples on MNIST and CelebA, but unrecognisable, yet diverse samples on CIFAR10, ImageNet32 and ImageNet64, in spite of state-of-the-art test set NLLs (see Fig. \ref{fig:uncond_samples_instability} and Appendix \ref{app:add_exp_details_results_Sampling instabilities in HVAEs}).
We argue that MNIST and CelebA, i.e. numbers and faces, have a more uni-modal nature, and are in this sense easier to learn for a discretised  multi-resolution process flowing to a point mass, which is uni-modal, than the other ``in-the-wild'', multi-modal datasets.
Trying to approximate the latter with the, in this case unsuitable, HVAE model leads to the sampling instabilities observed.

\begin{figure}[t]
    \centering
    \includegraphics[width=1.0\linewidth]{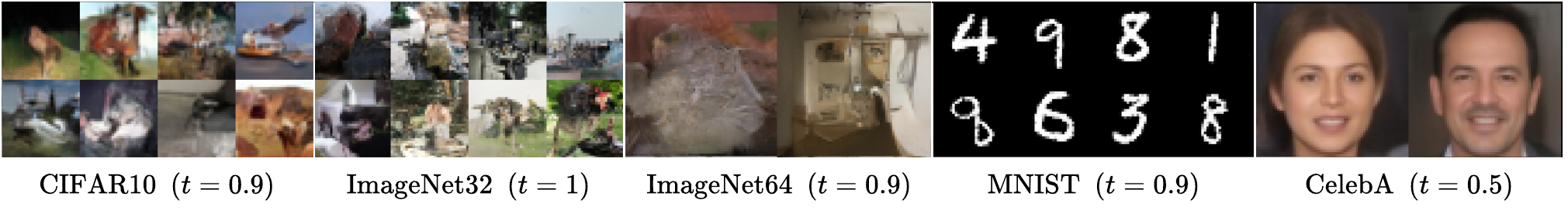}
    \caption{Unconditional samples (not cherry-picked) of \texttt{VDVAE$^*$}. 
    While samples on MNIST and CelebA demonstrate high fidelity and diversity, samples on CIFAR10, ImageNet32 and ImageNet64 are diverse, but are unrecognisable, demonstrating the instabilities identified by Theorem \ref{thm:id-diff}.
    Temperatures $t$ are tuned for maximum fidelity.
    }
    \label{fig:uncond_samples_instability}
\end{figure}

\subsection{Ablation studies}
\label{sec:Ablation studies}

We conducted several ablation studies which support our experimental results and further probe our multi-resolution framework for HVAEs. 
In this section we note key findings---a detailed account of all ablations can be found in Appendix \ref{app:add_exp_details_results_Ablation_studies}.
In particular, we find that the number of latent variables, which correlates with stochastic depth, does not explain the performance observed in \S\ref{sec:``More from less'': Parameter efficiency in HVAEs}, supporting our claims.
We further show that Fourier features do not provide a performance gain in HVAEs, in contrast to state-of-the-art diffusion models, where they can significantly improve performance \cite{Kingma2021VariationalModels}. 
This is consistent with our framework's finding that a U-Net architecture with pooling is already forced to learn a Haar wavelet basis representation of the data, hence introducing another basis does not add value.
We also demonstrate that residual cells are crucial for the performance of HVAEs as they are able to approximate the dynamics of a diffusion process and impose an SDE structure into the model, empirically compare a multi-resolution bridge to a single-resolution model, and investigate synchronous vs. asynchronous processing in time between the forward and backward pass. %

\section{Related work}
\label{sec:Related_work}

\textbf{U-Nets. }
A U-Net~\cite{ronneberger2015u} is an autoencoding architecture with multiple resolutions where skip connections enable information to pass between matched layers on opposite sides of the autoencoder's bottleneck.
These connections also smooth out the network's loss landscape~\cite{visualising_loss_landscape}.
In the literature, U-Nets tend to be convolutional, and a wide range of different approaches have been used for up-sampling and down-sampling between resolutions, with many using average pooling for the down-sampling operation \cite{DDPM,DDIM,song2020score,diffusion_models_beat_gans,Kingma2021VariationalModels}.
In this work, we focus on U-Nets as operators on measures interleaved by average pooling as the down-sampling operation (and a corresponding inclusion operation for up-sampling), and we formally characterise U-Nets in Section \ref{sec:Multi-Resolution Framework: Definitions and Intuition} and Appendix \ref{app:U-Net Model}.
Prior to our work, the dimensionality-reducing bottleneck structure of U-Nets was widely acknowledged as being useful, however it was unclear what regularising properties a U-Net imposes. 
We provided these in~\S\ref{sec:multi-res}.

\textbf{HVAEs. }The evolution of HVAEs can be seen as a quest for a parameterisation with more expressiveness than single-latent-layer VAEs \cite{kingma2013auto}, while achieving stable training dynamics that avoid common issues such as posterior collapse~\cite{LVAE,bowman2015generating} or exploding gradients.
Early HVAEs such as LVAE condition each latent variable directly on only the previous one by taking samples forward \cite{LVAE, burda2015importance}. 
Such VAEs suffer from stability issues even for very small stochastic depths.
\textit{Nouveau VAEs (NVAE)}~\cite{Vahdat2020NVAE:Autoencoder} and \textit{Very Deep VAEs (VDVAE)}~\cite{Child2020VeryImages} combine the improvements of several earlier HVAE models (see Appendix \ref{sec:Background} for details), while scaling up to larger stochastic depths. 
Both use ResNet-based backbones, sharing parameters between the generative and recognition parts of the model.  %
VDVAE is the considerably simpler approach, in particular avoiding common tricks such as a warm-up deterministic autoencoder training phase or data-specific initialisation.
VDVAE achieves a stochastic depth of up to 78, improving performance with more ResNet blocks.  %
Worth noting is that while LVAE and NVAE use convolutions with appropriately chosen stride to jump between resolutions, VDVAE use average pooling.  %
In all HVAEs to date, a theoretical underpinning which explains architectural choices, for instance the choice of residual cell, is missing, and we provided this in Section \S\ref{sec:Example: HVAEs are Sum Representation Diffusion Discretisations}.  %

\section{Conclusion}
\label{sec:Conclusion}

In this work, we introduced a multi-resolution framework for U-Nets.
We provided theoretical results which uncover the regularisation property of the U-Nets bottleneck architecture with average pooling as implicitly learning a Haar wavelet representation of the data.
We applied our framework to HVAEs, identifying them as multi-resolution diffusion processes flowing to a point mass.  
We characterised their backward cell as a type of two-step forward Euler discretisations, providing an alternative to score-matching to approximate a continuous-time diffusion process \cite{song2020score,de2021diffusion}, and observed parameter redundancies and instabilities.  %
We verified the latter theoretical insights in both small- and large-scale experiments, and in doing so trained the deepest ever HVAEs.   %
We explained these results by showing that HVAEs learn a representation of time and performed extensive ablation studies.

An important limitation is that the proven regularisation property of U-Nets is limited to using average pooling as the down-sampling operation. 
Another limitation is that we only applied our framework to HVAEs, though it is possible to apply it to other model classes.
It could also be argued that the lack of exhaustive hyperparameter optimisation performed is a limitation of the work as it may be possible to obtain improved results.
We demonstrate, however, that simply adding weight-sharing to the hyperparameter settings given in the original VDVAE paper~\cite{Child2020VeryImages} leads to state-of-the-art performance with improved parameter efficiency, and hence view it as a strength of our results.

\newpage

\begin{ack}
Fabian Falck acknowledges the receipt of studentship awards from the Health Data Research UK-The Alan Turing Institute Wellcome PhD Programme in Health Data Science (Grant Ref: 218529/Z/19/Z), and the Enrichment Scheme of The Alan Turing Institute under the EPSRC Grant EP/N510129/1.
Chris Williams acknowledges support from the Defence Science and Technology (DST) Group and from a ESPRC DTP Studentship. 
Dominic Danks is supported by a Doctoral Studentship from The Alan Turing Institute under the EPSRC Grant EP/N510129/1.
Christopher Yau is funded by a UKRI Turing AI Fellowship (Ref: EP/V023233/1).
Chris Holmes acknowledges support from the Medical Research Council Programme Leaders award MC\_UP\_A390\_1107, The Alan Turing Institute, Health Data Research, U.K., and the U.K. Engineering and Physical Sciences Research Council through the Bayes4Health programme grant.
Arnaud Doucet acknowledges support of the UK Defence Science and Technology Laboratory (Dstl) and EPSRC grant EP/R013616/1. This is part of the collaboration between US DOD, UK MOD and UK EPSRC under the Multidisciplinary University Research Initiative. 
Arnaud Doucet also acknowledges support from the EPSRC grant EP/R034710/1.
Matthew Willetts is grateful for the support of UCL Computer Science and The Alan Turing Institute.

The authors report no competing interests.

The three compute clusters used in this work were provided by the Alan Turing Institute, the Oxford Biomedical Research Computing (BMRC) facility, and the Baskerville Tier 2 HPC service (https://www.baskerville.ac.uk/) which we detail in the following.
First, this research was supported in part through computational resources provided by The Alan Turing Institute under EPSRC grant EP/N510129/1 and with the help of a generous gift from Microsoft Corporation.
Second, we used the Oxford BMRC facility, a joint development between the Wellcome Centre for Human Genetics and the Big Data Institute supported by Health Data Research UK and the NIHR Oxford Biomedical Research Centre. 
The views expressed are those of the author(s) and not necessarily those of the NHS, the NIHR or the Department of Health.
Third, Baskerville was funded by the EPSRC and UKRI through the World Class Labs scheme (EP/T022221/1) and the Digital Research Infrastructure programme (EP/W032244/1) and is operated by Advanced Research Computing at the University of Birmingham.

We thank Tomas Lazauskas, Jim Madge and Oscar Giles from the Alan Turing Institute’s Research Engineering team for their help and support.
We thank Adam Huffman, Jonathan Diprose, Geoffrey Ferrari and Colin Freeman from the Biomedical Research Computing team at the University of Oxford for their help and support.
We thank Haoting Zhang (University of Cambridge) for valuable comments on the implementation; Huiyu Wang (Johns Hopkins University) for a useful discussion on gradient checkpointing; and Ruining Li and Hanwen Zhu (University of Oxford) for kindly proofreading the manuscript.

\end{ack}

\bibliography{1_bib.bib}  %

\setcounter{equation}{0}
\setcounter{figure}{0}
\renewcommand\theequation{\thesection.\arabic{equation}}
\renewcommand\thefigure{\thesection.\arabic{figure}}
\setcounter{table}{0}
\renewcommand{\thetable}{\thesection.\arabic{table}}

\newpage
\appendix
\begin{appendices}

\section{Framework Details and Technical Proofs} %
\label{app:Proofs}
Here we provide proofs for the theorems in the main paper and additional theoretical results supporting these.

\subsection{Definitions and Notations}
\label{app:Definitions and Notations}

The following provides an index of commonly used notation throughout this manuscript for reference. 

The \textit{function space} of interest in this work is $L^2(\Xb)$, the space of square integrable functions, where $\Xb$ is a compact subset of $\R^{m}$ for some integer m, for instance, $\Xb = [0, 1]$. 
This set of functions is defined as 
\begin{align}
    L^2(\Xb)
    =
    \{ f: \Xb \rightarrow \R \, | \,
    \|f\|_2 < \infty, \,
    f \text{ Borel measurable}
    \}.
\end{align}
$L^2(\Xb)$ forms a vector space with the standard operations.

We denote $V_{-j} \subset L^2(\Xb)$ as a finite-dimensional approximation space.
With the nesting property, $V_{-j+1} \subset V_{-j} $, the space $U_{-j+1}$ is the orthogonal complement of $V_{-j+1}$ within $V_{-j}$, i.e. $V_{-j} = U_{-j+1} \oplus V_{-j+1}$.

The \textit{integration} shorthand notations used are as follows.
For an integrable function $t \mapsto f(t)$, we use 
\begin{align}
    f(t)dt := \int_{0}^t f(s) ds.
\end{align}
The function $f$ may be multi-dimensional in which case we mean the multi-dimensional integral in whichever basis is being used.
For \textit{stochastic} integrals, we only analyse dynamics within the truncation $V_{-J}$ of $L^2(\Xb)$.
In this case, $W_t$ refers to a Brownian motion on the same amount of dimensions as $V_{-J}$ in the \emph{standard}, or `pixel', basis of $V_{-J}$. 
The shorthand 
\begin{align}
    g(W_t) dW_t := \int_{0}^t g(W_s) dW_s,
\end{align}
is used for the standard It\^o integral. 
Last, for a stochastic process $X_t$ on $V_{-J}$ we mean the standard convention
\begin{align}
    \int_{0}^{t} dX_s = X_t - X_0.
\end{align}

For \textit{measures}, we use $\Db$ to prefix a set for which we consider the space of probability measures over: for instance, $\Db(\Xb)$ denotes the space of probability measures over $\Xb$.
We often refer to measures over functions (i.e. images): 
recall that $V_{-J}$ is an $L^2$-function space and we take $\Db(V_{-J})$ to be probability measures over this space. 

When referenced in Definition \ref{def:multi-res-bridge}, the distance metric between two measures $\nu_1$ and $\nu_2$ which yields the topology of \textit{weak continuity} is the \textit{Monge--Kantorovich} metric \cite{la2015monge,rachev1991probability}
\begin{align}
    d_{\Pb} (\nu_1, \nu_2) = \sup_{f \in \text{Lip}_1(\Xb)} \int f d(\nu_1 - \nu_2),
    \label{eq:weak_continuity}
\end{align}
where
\begin{align}
    \text{Lip}_1(\Xb) = \{ f: \Xb \rightarrow \R \hspace{1pt} \big| \hspace{1pt} \hspace{2pt} |f(x) - f(y) | \leq d(x,y), \forall x, y \in \Xb \}.  
\end{align}

Further, we use the Wasserstein-2 metric which in comparison to the weak convergence above has additional moment assumptions. 
It is given by
\begin{align}
    \mathcal{W}_2(\nu_1, \nu_2) 
    =
    \left(
    \inf_{\gamma \in \Gamma(\nu_1, \nu_2) }
    \Eb \|X_1-X_2 \|_2^2
    \right)^{1/2},
\end{align}
where $(X_1,X_2) \sim \gamma$ and $\Gamma(\nu_1, \nu_2)$ is the space of measures on $\Db(V_{-J} \times V_{-J}$) with marginals $\nu_1$ and $\nu_2$.

\subsection{Dimension Reduction Conjugacy}
\label{app:Dimension Reduction Conjucacy}

Assume momentarily the one dimensional case where $\Xb=[0,1]$. 
Let $V_{-j}$ be an \textit{approximation space} contained in $L^2(\Xb)$ (see Definition~\ref{def:multi-res-approx-space}) pertaining to image pixel values
\begin{align}\label{eq:Haar-1d-Vj}
    V_{-j} 
    = 
    \{
    f \in L^2([0,1]) \ | \ f_{[2^{-j} \cdot k, 2^{-j} \cdot (k+1) )} = c_k , \, k \in \{0, \dots, 2^{
    j}-1 \}, \, c_k \in \R 
    \}.
\end{align}
For a function $f \in V_{-j}$, there are several ways to express $f$ in different bases. 
Consider the \textit{standard (or `pixel') basis} for a fixed $V_{-j}$ given via 
\begin{align}
    e_{j,k} = \1_{[2^{-j} \cdot k, 2^{-j} \cdot(k+1)]}.
    \label{eq:standard_basis}
\end{align}
Clearly, the family $\bE_j \coloneqq \{e_{j,k} \}_{k=0}^{2^j-1}$ is an orthogonal basis of $V_{-j}$, hence full rank with dimension $2^{j}$.
Functions in $V_{-j}$ may be expressed as 
\begin{align}
    f = \sum_{k=0}^{2^j-1} c_k \cdot e_{j,k},
\end{align}
for $c_k \in \R$.

First, let us recall the average \emph{pooling} operation in these bases $\bE_{j}$ and $\bE_{j-1}$ of $V_{-j}$ and $V_{-j+1}$, where $\text{pool}_{-j,-j+1}:V_{-j} \rightarrow V_{-j+1}$. 
Its operation is given by
\begin{align}
    \text{pool}_{-j,-j+1}(f) = \text{pool}_{-j,-j+1} \left( \sum_{k=0}^{2^j-1} c_k \cdot e_{j,k} \right) = \sum_{i=0}^{2^{j-1}-1} \tilde{c}_i \cdot e_{j-1,i},
\end{align}
where for $i \in \{0,\dots, 2^{j-1}-1 \}$ we have the coefficient relation
\begin{align}
    \tilde{c}_i = \frac{c_{2i}+c_{2i+1}}{2} = \frac{1}{2^{-j} }\int_{[2^{-j}\cdot(2i),2^{-j}\cdot(2i+1))} f(x)dx.
\end{align}
Average pooling and its imposed basis representation are commonly used in U-Net architectures~\cite{ronneberger2015u}, for instance in state-of-the-art diffusion models \cite{DDPM} and HVAEs \cite{Child2020VeryImages}.

Note that across approximation spaces of two resolutions $V_{-j}$ and $V_{-j+1}$, the standard bases $\bE_{j}$ and $\bE_{j-1}$ share no basis elements. 
As basis elements change at each resolution, it is difficult to analyse $V_{-j}$ embedded in $V_{-j+1}$.
What we seek is a basis for all $V_{-j}$ such that any basis element in this set at resolution $j$ is also a basis element in $V_{-J}$, the approximation space of highest resolution $J$ we consider. 
This is where wavelets serve their purpose:
We consider a \textit{multi-resolution (or `wavelet') basis} of $V_{-J}$ \cite{mallat1989multiresolution}. 
For the purpose of our theoretical results below, we are here focusing on a \textit{Haar wavelet} basis \cite{haar1909theorie} which we introduce in the following, but note that our framework straightforwardly generalises to other wavelet bases.
Begin with $\phi_{1} = \1_{[0,1)}$ as $L^2$-basis element for $V_{-1}$, the space of constant functions on $[0,1)$.
For $V_{-2}$ we have the space of $L^2$ functions which are constant on $[0,1/2)$ and $[1/2,1)$, which we receive by adding the basis element $\psi = \sqrt{2} (\1_{[0,1/2)} - \1_{[1/2,1)})$. 
Here $\phi_{1}$ is known as the father wavelet, and $\psi$ as the mother wavelet. 
To make a basis for general $V_{-j}$ we localise these two wavelets with scaling and translation, i.e
\begin{align}
    \psi_{i,k} = 2^{-i/2} \cdot \psi (2^{i} (\sds - k)) \quad \text{ where } i \in \{0, j\}, k \in \{0, 2^{-i+1}\}.
    \label{eq:scale_trans}
\end{align}
It is straight-forward to check that $\bPsi_j \coloneqq \{\psi_{i,k} \}_{i = 0,k=0}^{j,2^{i-1}}$ is an orthonormal basis of $V_{-j}$ on $[0,1]$.
Further, the truncated basis $\bPsi_{j-1}$, which is a basis for $V_{-j+1}$, is contained in the basis $\bPsi_{j}$.
This is in contrast to $\bE_{j-1}$ which has basis elements distinct from the elements in the basis $\bE_{j}$ on a higher resolution. 

The collections $\bE_j$ and $\bPsi_j$ both constitute full-rank bases for $V_{-j}$.
They further have the same dimension and so there is a linear isomorphism $\pi_j:V_{-j} \rightarrow V_{-j}$ for change of basis, i.e.
\begin{align}
    \pi_j(e_{j,i}) = \psi_{j,i}.
\end{align}
This can be normalised to be an isometry.
We now analyse the pooling operation in our basis $\bPsi_j$, restating Theorem \ref{thm:conj} from the main text and providing a proof.

\textbf{Theorem \ref{thm:conj}. }
Given $V_{-j}$ as in Definition \ref{def:multi-res-approx-space}, let $x \in V_{-j}$ be represented in the standard basis $\bE_j$ and Haar basis $\bPsi_j$.
Let $\pi_j : \bE_j \mapsto \bPsi_j$ be the change of basis map illustrated in Fig. \ref{fig:change_basis_map}, then we have the conjugacy $\pi_{j-1} \circ \text{pool}_{-j,-j+1} = \proj_{V_{-j+1}} \circ \pi_{j}$. 

\begin{proof}
Define the conjugate pooling map in the wavelet basis, $\text{pool}^*_{j,j+1} : V_{-j} \rightarrow V_{-j+1}$ computed on the bases $\bPsi_j$ and $\bPsi_{j-1}$,
\begin{align}
    \text{pool}^*_{-j,-j+1} \coloneqq \pi_{j-1} \circ \text{pool}_{-j,-j+1} \circ \pi_{j}^{-1} .
\end{align}
\[\begin{tikzcd}
	{(V_{-j},\bE_j)} && {(V_{-j+1},\bE_{j-1})} \\
	{(V_{-j},\bPsi_j)} && {(V_{-j+1},\bPsi_{j-1})}
	\arrow["{\text{pool}_{-j,-j+1}}", from=1-1, to=1-3]
	\arrow["{\text{pool}_{-j,-j+1}^*}"', dashed, from=2-1, to=2-3]
	\arrow["{\pi_{j-1}}"', from=1-3, to=2-3]
	\arrow["{\pi_j^{-1}}"', from=2-1, to=1-1]
\end{tikzcd}\]
Due to the scaling and translation construction in Eq. \eqref{eq:scale_trans} and because the pooling operation is local, we need only consider the case for $\text{pool}_{-2,-1}$.
This is because one can view pooling between the higher-resolution spaces as multiple localised pooling operations between $V_{-2}$ and $V_{-1}$.
Now note that $\text{pool}_{-2,-1}$ maps $V_{-2}$ to $V_{-1}$.
Further,
\begin{align}
    \int_{\Xb} \psi(x) dx = 0,
\end{align}
where $\psi = \sqrt{2} (\1_{[0,1/2)} - \1_{[1/2,1)})$ is the mother wavelet.
For $v \in V_{-2}$ let $v$ have the wavelet representation $v = \tilde{c}_2 \psi+ \tilde{c}_1 \phi_1$, where $\phi_{1} = \1_{[0,1)}$ is the father wavelet.
To pool we compute the average of the two coefficients (`pixel values')
\begin{align}
    \text{pool}_{-2,-1}(v) = \int_{\Xb} v(x) dx = \int_{\Xb} \tilde{c}_2 \psi(x) + \tilde{c}_1 \phi_1(x) dx = \tilde{c}_1.
\end{align}
Thus average pooling here corresponds to truncation of the wavelet basis for $V_{-2}$ to the wavelet basis for $V_{-1}$. 
As this basis is orthonormal over $L^2(\Xb)$, truncation corresponds to $L^2$ projection, i.e. $\text{pool}^*_{-j,-j+1} = \proj_{V_{-j+1}}$, as claimed.
\end{proof}

Theorem \ref{thm:conj} shows that the pooling operation is conjugate to projection in the Haar wavelet approximation space, and computed by truncation in the Haar wavelet basis. 
The only quantity we needed for our basis over the $V_{-j}$ was the vanishing moment quantity 
\begin{align}
     \int_{\Xb} \psi(x) dx = 0.
\end{align}
To extend this property to higher dimensions, such as the two dimensions of gray-scale images, we use the tensor product of $[0,1]$, and further, the tensor product of basis functions. 
This property is preserved, and hence the associated average pooling operation is preserved on the tensor product wavelet basis, too. 
To further extend it to color images, one may consider the cartesian product of several L2 spaces.

\subsection{Average pooling Truncation Error}
In this section we prove Theorem \ref{thm:truncation}, which quantifies the regularisation imposed by an average pooling bottleneck trained by minimising the reconstruction error. 
The proof structure is as follows:
First we give an intuition for autoencoders with an average pooling bottleneck, then derive the relevant assumptions for Theorem \ref{thm:truncation}.
We next prove our result under strong assumptions.
Last, we weaken our assumptions so that our theorem is relevant to HVAE architectures. 

Suppose we train an autoencoder on $V_{-j}$ without dimension reduction, calling the parameterised forward (or encoder/bottom-up) and backward (or decoder/top-down) passes $F_{j,\phi}, \, B_{j,\theta}:V_{-j} \mapsto V_{-j}$ respectively.
We can optimise $F_{j,\phi}$ and $B_{j,\theta}$ w.r.t. $\phi$ and $\theta$ to find a perfect reconstruction, i.e. $x = B_{j,\theta}F_{j,\phi}x$ for all $x$ in our data as there is no bottleneck (no dimensionality reduction): 
$B_{j,\theta}$ need only be a left inverse of $F_{j,\phi}$, as in
\begin{align}
    B_{j,\theta} F_{j,\phi} = I.
    \label{eq:bfident}
\end{align}
Importantly, we can choose $F_{j,\phi}$ and $B_{j,\theta}$ satisfying \ref{eq:bfident} \textit{independent} of our data.
For instance, they could both be the identity operator and achieve perfect reconstruction, but contain no information about the generative characteristics of our data.
Compare this to a \textit{bottleneck} with average pooling, i.e. an autoencoder with dimension reduction.
Here, we consider the dimension reduction from $V_{-j}$ to $V_{-j+1}$, where we split $V_{-j} = V_{-j+1} \oplus U_{-j+1}$.
As we have seen in Theorem \ref{thm:conj}, through average pooling, we keep information in $V_{-j+1}$, and discard the information in $U_{-j+1}$.
For simplicity, let $\embd_{V_{-j}}$ be the inclusion of the projection $\proj_{V_{-j+1}}$.
Now to achieve perfect reconstruction 
\begin{align}
    x = (B_{j,\theta} \circ \embd_{V_{-j}} \circ \proj_{V_{-j+1}} \circ F_{j,\phi} ) x,
\end{align}
we require $(\proj_{U_{-j+1}} F_{j,\phi}) x = 0$.
Simply put, the encoder $F_{j,\phi}$ should make sure that the discarded information in the bottleneck is nullified. 

We may marry this observation with a simple U-Net structure (without skip connection) with $L^2$-reconstruction and average pooling dimension reduction. 
Let $V_{-j}$ be one of our multi-resolution approximation spaces and $\Db(V_{-j})$ be the space of probability measures over $V_{-j}$.
Recall in a multi-resolution basis we have $V_{-j} = V_{-j+1} \oplus U_{-j+1}$ where $U_{-j+1}$ is the $-j+1$ orthogonal compliment within $V_{-j}$.
For any $v \in V_{-j}$ we may write $v = \proj_{V_{-j+1}} v \oplus \proj_{U_{-j+1}}v $ and analyse the truncation error in $V_{-j+1}$, i.e. the discarded information, via
\begin{align}
    \| v - \embd_{V_{-j}} \circ \proj_{V_{-j+1}} v \|_2^2 = \| \proj_{U_{-j+1}}v \|_2^2. 
\end{align}
If we normalise this value to  
\begin{align}
    \frac{\| v - \embd_{V_{-j}} \circ \proj_{V_{-j+1}} v \|_2^2}{\|v\|_2^2}
    =
    \frac{\| \proj_{U_{-j+1}}v \|_2^2}{\|v\|_2^2} \in [0,1],
\end{align}
then this is zero when $v$ is non-zero only within $V_{-j+1}$ and zero everywhere within $U_{-j+1}$.
Suppose now that we have a measure $\nu_{j} \in \Db(V_{-j})$, we could quantify \emph{how much} of the norm of a sample from $\nu_{j}$ comes from the $U_{-j+1}$ components by computing 
\begin{align}
   \Eb_{v \sim \nu_{j}} \frac{\| \proj_{U_{-j+1}}v \|_2^2}{\|v\|_2^2} =  \int \frac{\| \proj_{U_{-j+1}}v \|_2^2}{\|v\|_2^2} d \nu_j (v) \in [0,1].
\end{align}
This value forms a convex sum with its complement projection to $\proj_{V_{-j+1}}$, demonstrating the splitting of mass across $V_{-j+1}$ and $U_{-j+1}$, as we show in Lemma \ref{lemma:convex_proj}.
\begin{lemma}
\label{lemma:convex_proj}
Let $\nu_{j} \in \Db(V_{-j})$ be atom-less at $0$, then
\begin{align}
    \Eb_{v \sim \nu_{j}} \frac{\| \proj_{V_{-j+1}}v \|_2^2}{\|v\|_2^2}
    +
    \Eb_{v \sim \nu_{j}} \frac{\| \proj_{U_{-j+1}}v \|_2^2}{\|v\|_2^2}
    =
    1.
\label{eq:convex_proj}
\end{align}
\end{lemma}
\begin{proof}
For any $v \in V_{-j}$ we have $\|v\|_2^2 = \| \proj_{V_{-j+1}}v \|_2^2 + \| \proj_{U_{-j+1}}v \|_2^2$ due to orthogonality of $V_{-j+1}$ and $U_{-j+1}$.
As both $\| \proj_{V_{-j+1}}v \|_2^2$ and  $\| \proj_{U_{-j+1}}v \|_2^2$ are projections, they are bounded by $\|v\|_2^2$ giving that the integrands in Eq. \eqref{eq:convex_proj} are bounded by one, and so for all $v \neq 0$ (no point mass at 0) the expectation is bounded.
\end{proof}

From the splitting behaviour of masses in the $L^2$-norm observed in Lemma \ref{lemma:convex_proj} we see that
\begin{enumerate}
    \item if $\Eb_{v \sim \nu_j} {\| \proj_{U_{-j+1}}v \|_2^2}/{\|v\|_2^2}$ is large, then, on average, samples from $\nu_j$ have most of their size in the $U_{-j+1}$ subspace; or,
    \item if $\Eb_{v \sim \nu_j} {\| \proj_{U_{-j+1}}v \|_2^2}/{\|v\|_2^2}$ is small, then, on average, samples from $\nu_j$ have most of their size in the $V_{-j+1}$ subspace.
\end{enumerate}

In the latter case, $\| \proj_{U_{-j+1}}v \|_2^2 \approx 0$, i.e. $\embd_{V_{-j}} \circ \proj_{V_{-j+1}} v \approx v$.
We get the heuristic $\embd_{V_{-j}} \circ \proj_{V_{-j+1}} \approx I$ on the measure $\nu_j$, yielding a perfect reconstruction.

Let $\embd_{V_{-j}} \circ \proj_{V_{-j+1}}, I: V_{-j} \rightarrow V_{-j}$, then this heuristic performs the operator approximation
\begin{align}
    \Eb_{v \sim \nu_j}\|(\embd_{V_{-j}} \circ \proj_{V_{-j+1}} - I) v\|_2^2,
\end{align}
quantifying `how close' these operators are on $\nu_j$.
For many measures, this (near) equivalence between operators will not hold.
But what if instead, we had an operator $D: V_{-j} \rightarrow V_{-j}$ such that the push-forward of $\nu_j$ through this operator had this quality.
Practically, this push-forward operator will be parameterised by neural networks, for instance later in the context of U-Nets.
For simplicity, we will initially consider the case where $D$ is linear on $V_{-j}$, then we consider when $D$ is Lipschitz.
\begin{lemma}\label{lemma:linear-Dj}
Given $V_{-j}$ with the $L^2$-orthogonal decomposition $V_{-j} = V_{-j+1} \oplus U_{-j+1}$, let $D_{-j} : V_{-j} \rightarrow V_{-j}$ be an invertible linear operator and define $F_{j}: V_{-j} \rightarrow V_{-j+1}$ and $B_{j}: V_{-j+1} \rightarrow V_{-j}$ through 
\begin{align}
    F_{j} = \proj_{V_{-j+1}} \circ D_{j}, 
    &&
    B_{j} = D_{j}^{-1} \circ \embd_{V_{-j}}.  
\end{align}
Then $B_{j}F_{j} \equiv I$ on $V_{-j}$, or otherwise, we have the truncation bound
\begin{align}
    \frac{\norm{\proj_{U_{-j+1} } D_{j} v}_2^2}{\norm{D_{j} }_2^2}
    \leq
     \norm{(I - B_j F_j)v}_2^2.
\end{align}
\end{lemma}

\begin{proof}
Consider the operator $ D_j(I-B_j F_j) : V_{-j} \rightarrow V_{-j}$ which is linear and obeys the multiplicative bound $\norm{ D_j(I-B_{j} F_j)} \leq \norm{D_j} \norm{I - B_j F_j}$.
This implies for any $v \in V_{-j}$,
\begin{align}
    \frac{\norm{D_j (I -B_j F_j) v}_2^2}{\norm{D_j}_2^2} \leq  \norm{(I  - B_j F_j)v}_2^2.
\end{align}
The numerator is equal to 
\begin{align}
    \norm{D_j (I -B_j F_j) v}_2^2 =  \norm{(D_j - \embd_{V_{-j}} \circ \proj_{V_{-j+1}} \circ D_j )  v }_2^2.
\end{align}
As we have the orthogonal decomposition $V_{-j} = V_{-j+1} \oplus U_{-j+1}$, we know 
\begin{align}
    I  &= \proj_{V_{-j+1}} \oplus \proj_{U_{-j+1}} 
    \\
    &= \embd_{V_{-j}} \circ \proj_{V_{-j+1}} +  \embd_{V_{-j}} \circ \proj_{U_{-j+1}}, 
\end{align}
and as $\norm{\embd_{V_{-j}}}_2 = 1$, we get
\begin{align}
    \norm{(I - \embd_{V_{-j}} \circ \proj_{V_{-j+1}} \circ D_{j} )  v }_2^2 = \norm{\proj_{U_{-j+1}} \circ D_j   v }_2^2.
\end{align}
So now as $ \norm{D_j (I -B_j F_j) v}_2^2 = \norm{\proj_{U_{-j+1}} \circ D_j   v }_2^2$, we may use $\norm{D_j (I -B_j F_j) v}_2^2 \leq \norm{D_j}^2 \norm{I -B_j F_j v}_2^2$ to get the desired result.
\end{proof}

The quantity ${\norm{\proj_{U_{-j+1} } F_{j} v}_2^2}/{\norm{F_{j} }_2^2}$ is analogous to the in Lemma \ref{lemma:convex_proj} discussed quantity ${\| \proj_{U_{-j+1}}v \|_2^2}/{\|v\|_2^2}$, but we now have a `free parameter', the operator $D_j$.

Next, suppose $D_j$ is trainable with parameters $\theta$.
We do so by minimising the reconstruction cost 
\begin{align}
    \Eb_{v \sim \nu_j} \norm{(I  - B_j F_j)v}_2^2,
\end{align}
which upper-bounds our `closeness metric' in Lemma \ref{lemma:linear-Dj}.  %

In the linear case ($D_j$ is linear), to ensure that $D_{j,\theta}$ is invertible we may parameterise it by an (unnormalised) LU-decomposition of the identity 
\begin{align}
    I = D_{j,\theta}^{-1}D_{j,\theta} = L_{j,\theta} U_{j,\theta},
\end{align}
where the diagonal entries of $L_{j,\theta}$ and $U_{j,\theta}$ are necessarily inverses of one-another.  %
This is a natural parameterisation when considering a U-Net with dimensionality reduction.
Building from Lemma \ref{lemma:linear-Dj}, we can now consider the stacked U-Net (without skip connections), i.e. a U-Net with multiple downsampling/upsampling and forward/backward operators stacked on top of each other, in the linear setting. 
In Proposition \ref{prop:linear-U}, we show that this $LU$-parameterisation forces the pivots of $U_{j,\theta}$ to tend toward zero. 

\begin{prop}\label{prop:linear-U}
Let $\{V_{-j} \}_{j=0}^J$ be a multi-resolution hierarchy of $V_{-J}$ with the orthogonal decompositions $V_{-j} = V_{-j+1} \oplus U_{-j+1}$ and $F_{j,\phi}, \, B_{j,\theta} : V_{-j} \rightarrow V_{-j}$ be bounded linear operators such that $B_{j,\theta} F_{j,\phi} = I$.
Define $\bm{F}_{j,\phi}: V_{-j} \rightarrow V_{-j+1}$ and $\bm{B}_{j,\theta}: V_{-j+1} \rightarrow V_{-j}$ by
\begin{align}
    \bm{F}_{j,\phi} \coloneqq \proj_{V_{-j+1}} \circ F_{j,\phi},
    &&
    \bm{B}_{j,\theta} \coloneqq B_{j,\theta} \circ \embd_{V_{-j} },
\end{align}
with compositions 
\begin{align}
    \bm{F}_{j_1|j_2,\phi} \coloneqq \bm{F}_{j_1,\phi} \circ  \cdots  \circ  \bm{F}_{j_2,\phi},
    &&
    \bm{B}_{j_1|j_2,\phi} \coloneqq \bm{B}_{j_1,\phi} \circ  \cdots  \circ  \bm{B}_{j_2,\phi}.
\end{align}
Then

\begin{align}
    \sum_{j=1}^{J}
    \frac{\norm{\proj_{U_{-j+1} } F_{j} v}_2^2}{\norm{F_{j} }_2^2}
    \leq 
    \norm{ ({I} - \bm{B}_{1|J,\theta}\bm{F}_{1|J,\phi})v}_2^2 .
\end{align}

\end{prop}

\begin{proof}
The operator $\mathbf{F}_{1|J}$ is linear, and decomposes into a block operator form with pivots $\mathbf{F}_{j|J}$ for each $j \in \{1,\dots, J\}$.
Each $\mathbf{F}_{j|J}$ is $L^2$-operator norm bounded by $\|F_{j}\|_2$, so if 
\begin{align}
    \lambda_{1|J} \coloneqq \text{diag} ( \|F_{1}\|_2, \dots , \|F_{J}\|_2 ),
\end{align}
then $ \|\lambda_{1|J}^{-1} \mathbf{F}_{1|J} \|_2 \leq 1$.
Last, as the spaces $\{U_{-j}\}_{j=0}^{J}$ are orthogonal and $\mathbf{F}_{1|J}$ has triangular form:
\begin{align}
    \|\lambda_{1|J}^{-1}(F_{1|J}-\mathbf{F}_{1|J})v\|_2^2 = \sum_{j=1}^{J}
    \frac{\norm{\proj_{U_{-j+1} } F_{j} v}_2^2}{\norm{F_{j} }_2^2},
\end{align}
and $\|\lambda_{1|J}^{-1}(F_{1|J}-\mathbf{F}_{1|J})v\|_2^2 \leq \|(I-\mathbf{B}_{1|J}\mathbf{F}_{1|J})v\|_2^2$.
\end{proof}

Here in the linear case, a U-Net's encoder is a triangular matrix where the basis vectors are the Haar wavelets. 
Proposition \ref{prop:linear-U} states that the pivots of this matrix are minimised.
Adversely, this diminishes the rank of the autoencoder and pushes our original underdetermined problem to a singular one. 
In other words, the U-Net is in this case demanding to approximate the identity (via an $LU$-like-decomposition), a linear operator, with an operator of diminishing rank.

 \begin{prop}\label{prop:measure-trun}
Let $\Db(\Xb)$ be the space of probability measures over $\Xb$, and assume for $\overline{F}_j,\overline{B}_j : \Db(\Xb) \rightarrow \Db(\Xb)$ that these are inverses of one-another and $\overline{F}_j$ is Lipschitz, that is
\begin{align}
    \overline{F}_j\overline{B}_j = I, && \mathcal{W}_2(\overline{F}_j\nu_1, \overline{F}_j\nu_2 ) \leq \|\overline{F}_j\|_2 W_2(\nu_1, \nu_2 ).
\end{align} 
Then for any $\nu \in \Db(\Xb)$ with bounded second moment, 
\begin{align}
    \Eb_{X_j \sim \overline{F}_j \nu} \frac{\| \text{proj}_{U_{-j} } X_j \|_2^2}{\|\overline{F}_j\|_2^2}
    \leq \mathcal{W}_2( \nu, \overline{B}_j \circ P_{V_{-j}} \circ \overline{F}_j \nu ).
\end{align}
\end{prop}

\begin{proof}
First as $\overline{F}_j\overline{B}_j = I$ we know that 
\begin{align}
 \mathcal{W}_2( \overline{F}_j \nu ,P_{V_{-j}} \overline{F}_j \nu)
 =
\mathcal{W}_2(\overline{F}_j\overline{B}_j \overline{F}_j \nu , \overline{F}_j\overline{B}_j  P_{V_{-j}} \overline{F}_j \nu).
\end{align}
But for any $X \in V_{-j}$ we have the orthogonal decomposition 
\begin{align}
    X = \text{proj}_{V_{-j} } X \oplus \text{proj}_{U_{-j} } X,
\end{align}
which respects the $L^2$-norm by
\begin{align}
    \|X\|_2^2 = \|\text{proj}_{V_{-j} } X \|_2^2 + \|\text{proj}_{U_{-j} } X\|^2,
\end{align}
and in particular,
\begin{align}
    \| X - \text{proj}_{V_{-j} } X \|_2^2 = \|\text{proj}_{U_{-j} } X\|^2.
\end{align}
This grants
\begin{align}
    (W_2(\overline{F}_j\nu, P_{V_{-j}}\overline{F}_j\nu))^2 
    &=  \inf_{\gamma \in \Gamma(\overline{F}_j\nu, P_{V_{-j}}\overline{F}_j\nu) } \Eb \|X-Y \|_2^2
    \\
    &=  \inf_{\gamma \in \Gamma(\overline{F}_j\nu, P_{V_{-j}}\overline{F}_j\nu) } \Eb \|\text{proj}_{V_{-j}}X-\text{proj}_{V_{-j}}Y \|_2^2 + \|\text{proj}_{U_{-j}}X \|_2^2 
    \\
    &=  \inf_{\gamma \in \Gamma(\overline{F}_j\nu, P_{V_{-j}}\overline{F}_j\nu ) } \Eb  \|\text{proj}_{U_{-j}}X \|_2^2
    \\
    &= (\mathcal{W}_2(\text{proj}_{U_{-j}}\overline{F}_j \nu, \delta_{\{0\} }))^2.
\end{align}
Now the Lipschitz of $\overline{F}_j$ yields
\begin{align}
    \frac{\mathcal{W}_2(\overline{F}_j\overline{B}_j \overline{F}_j \nu , \overline{F}_j\overline{B}_j  P_{V_{-j}} \overline{F}_j \nu)}{\|\overline{F}_{j}\|_2} 
    \leq 
    \mathcal{W}_2(\overline{B}_j \overline{F}_j \nu , \overline{B}_j  P_{V_{-j}} \overline{F}_j \nu)
    =
    \mathcal{W}_2(\nu , \overline{B}_j  P_{V_{-j}} \overline{F}_j \nu).
\end{align}
Squaring and substituting grants 
\begin{align}
    \frac{(\mathcal{W}_2(\text{proj}_{U_{-j}}\overline{F}_j \nu, \delta_{\{0\} }))^2}{\|\overline{F}_{j}\|_2^2}
    \leq
    \mathcal{W}_2(\nu , \overline{B}_j  P_{V_{-j}} \overline{F}_j \nu).
\end{align}
\end{proof}
 
To work on multiple resolution spaces, we need to define what the triangular operator over our space of measures is.
For a cylinder set $B$ on $V_{-J} = V_0 \oplus \bigoplus_{j=0}^J U_{-j}$ we can assume it has the form $\bigotimes_{j} B_{j}$ where $B_j$ is a cylinder on $U_{j}$.
Break $\nu_J$ into the multi-resolution sub-spaces by defining projection onto $\Db(U_{-j})$ through 
\begin{align}
    \proj_{U_{-j}}(\nu_J) (B_j) \coloneqq \nu_J(B_j \otimes U_{-j}^{\perp}),
\end{align}
where $B_j$ is a cylinder for $U_{-j}$.
This projection of measures is respected by evaluation in that
\begin{align}
     \Eb_{X_j \sim \proj_{U_{-j}}\nu_{J} } X_j  =  \int v_j d \proj_{U_{-j}}\nu_{J}(v_j) = \int \proj_{U_{-j}}v d \nu_{J}(v) = \Eb_{X_j \sim \nu_{J} } \proj_{U_{-j}} X .
\end{align}
As $\|X\|_2^2 = \sum_{j} \proj_{U_{-j}} \| \proj_{U_{-j}} X \|_2^2$ due to the orthogonality of the spaces, then
\begin{align}
    \Eb_{X_j \sim \proj_{U_{-j}}\nu_{J} } \|X_j\|_2^2
    =
    \sum_{j}
    \Eb_{X_j \sim \proj_{U_{-j}} \nu_{J} }  \|X_j\|_2^2
    = 
    \sum_{j}
    \Eb_{X \sim \nu_{J} } \|\proj_{U_{-j}} X\|_2^2.
\end{align}

Define the extension, with a convenient abuse of notation, of $\proj_{V_{-j+1}}$ on $\Db(V_{-J})$ to be
\begin{align}
    \proj_{V_{-j+1}} (\nu_{J}) \coloneqq \proj_{V_{-j+1}}(\nu_{J}) \otimes \proj_{V_{-j+1}^{\perp}}(\nu_{J}).
\end{align}
If $F_{-j}: \Db(V_{-j}) \rightarrow \Db(V_{-j})$ are linear operators for $j \in \{0,\dots,J\}$, extend each $F_{j}: \Db(V_{-j}) \rightarrow \Db(V_{-j})$ to $\Db(V_{-j}) \times \Db(V_{-j}^{\perp})$ through 
\begin{align}
    \overline{F}_{j} \coloneqq F_{j} \oplus I.
\end{align}
For a measure $\nu_{J} \in \Db(V_{-J})$ we can split it into $\Db(V_{-j}) \times \Db(V_{-j}^{\perp})$ via
\begin{align}
    \proj_{V_{-j}} \nu_J \times \proj_{V_{-j}^{\perp}} \nu_J,
\end{align}
which also remains a measure in $\Db(V_{-J})$ as $\Db(V_{-j}) \times \Db(V_{-j}^{\perp}) \subset \Db(V_{-J})$.
Now the operator $\overline{F}_{j}$ acts on the product measure $\nu_{j} \otimes \nu_{j}^{\perp}$ by
\begin{align}
    \overline{F}_{j} ( \nu_{j} \otimes \nu_{j}^{\perp} ) = F_{j} \, \nu_{j} \otimes I \,  \nu_{j}^{\perp}.
\end{align}
Now we may define the map $\bm{F}_j: \Db(V_{-J}) \rightarrow \Db(V_{-j}) \times \Db(V_{-j}^{\perp})$ through
\begin{align}
    \bm{F}_j \coloneqq \overline{F}_{j}  \proj_{V_{-j}},
\end{align}
and its compositions by 
\begin{align}
    \bm{F}_{j_1|j_2} = \bm{F}_{j_1} \circ \cdots \circ \bm{F}_{j_2},
\end{align}
which too is an operator on $\Db(V_{-J})$.

Further if we have a measure $\nu_j$ on $V_{-j}$ we can form the embedding map 
\begin{align}
    \embd_{j}\nu_{j} = \nu_{j} \otimes \bigotimes_{i=j}^{J} \delta_{\{ 0\} },
\end{align}
which we extend to $\Db(V_{-J})$ by a convenient abuse of notation 
\begin{align}
    \proj_j \nu_{J} = \proj_j(\nu_{J}) \otimes \bigotimes_{i=j}^{J} \delta_{\{ 0\} }.
\end{align}
Let $B_{-j} : \Db(V_{-j}) \rightarrow \Db(V_{-j})$ be the linear operator which is the inverse of $F_{-j}$.
Now if we extend $B_{j}: \Db(V_{-j}) \rightarrow \Db(V_{-j})$ to $\Db(V_{-j}) \times \Db(V_{-j}^{\perp})$ like before through
\begin{align}
    \overline{B}_{j} \coloneqq B_{j} \oplus {I},
\end{align}
so the map $ \overline{B}_{j} \embd_j $ is well defined on $\Db(V_{-J})$.
Now analogously define $\bm{B}_j$ and its compositions by 
\begin{align}
    \bm{B}_j \coloneqq \overline{B}_{j} \embd_{V_{-j}}
    &&
    \bm{B}_{j_1|j_2} = \bm{B}_{j_2} \circ \cdots \circ \bm{B}_{j_1}.
\end{align}

In an analogous way, the operator $\bm{F}_{j_1|j_2}$ is `upper triangular' and $\bm{B}_{j_1|j_2}$ is `lower triangular'. 
In this way, we are again seeking a lower/upper ($LU$-) decomposition of the identity on $\Db(V_{-J})$.
Now we may prove Theorem \ref{thm:truncation}.

\textbf{Theorem \ref{thm:truncation}. }
Let $\{V_{-j}\}_{j=0}^J$ be a multi-resolution hierarchy of $V_{-J}$ where $V_{-j} = V_{-j+1} \oplus U_{-j+1}$, and further, let $F_{j,\phi}, \, B_{j,\theta} : \Db(V_{-j}) \mapsto \Db(V_{-j})$ be such that $B_{j,\theta}F_{j,\phi}=I$ with parameters $\phi$ and $\theta$.
Define $\bm{F}_{j_1|j_2,\phi} \coloneqq \bm{F}_{j_1,\phi} \circ  \cdots  \circ  \bm{F}_{j_2,\phi}$ by $\bm{F}_{j,\phi}: \Db(V_{-j}) \mapsto \Db(V_{-j+1})$ where $\bm{F}_{j,\phi} \coloneqq \proj_{V_{-j+1}} \circ F_{j,\phi}$, and analogously define $\bm{B}_{j_1|j_2,\theta}$ with
$\bm{B}_{j,\theta} \coloneqq B_{j,\theta} \circ \embd_{V_{-j} }$. 
Then, the sequence $\{ \bm{B}_{1|j,\theta} (\bm{F}_{1|J,\phi} \nu_{J}) \}_{j=0}^{J}$ forms a discrete multi-resolution bridge between $\bm{F}_{1|J,\phi} \nu_{J}$ and $\bm{B}_{1|J,\theta} \bm{F}_{1|J,\phi} \nu_{J}$ at times $\{t_j\}_{j=1}^{J}$, and
\vspace{-0.5em}
\begin{align}\label{eq:truncation_bound}
    \sum_{j=0}^{J}
    \Eb_{X_{t_J} \sim \nu_{J} }{\norm{\text{\emph{proj}}_{U_{-j+1} }X_{t_j} }_2^2}/{\norm{\bm{F}_{j|J,\phi}}_2^2} 
    \leq 
    (\mathcal{W}_2( \bm{B}_{1|J,\theta} \bm{F}_{1|J,\phi} \nu_{J} , \nu_{J}))^2,
\end{align}
\vspace{-0.5em}
where $\mathcal{W}_2$ is the Wasserstein-$2$ metric and $\norm{\bm{F}_{j|J,\phi}}_2$ is the Lipschitz constant of $\bm{F}_{j|J,\phi}$.

 \begin{proof}
 All we must show is that successively chaining the projections from Proposition \ref{prop:measure-trun} decomposes like in Proposition \ref{prop:linear-U}. 
For $X_1,X_2 \sim \nu$, $\mathcal{W}_2(F_{j} F_{j+1} \nu , P_{-j+2}F_{j-1} P_{-j+1}F_{j} \nu )$ consider $f_{j},f_{j-1}$ as realised paths for our kernel and write $ \|f_{j-1} f_{j} X_1 - \proj_{V_{-j+2}} f_{j-1} \proj_{V_{-j+1}} f_{j} X_2 \|_2^2$
\begin{align*}
    &=
    \|\proj_{V_{-j+1}}(f_{j} f_{j+1} X_1 - \proj_{V_{-j+2}} f_{j-1} \proj_{V_{-j+1}} f_{j} X_2 )\|_2^2
    \\
    &\quad +
    \|\proj_{U_{-j+1}}(f_{j} f_{j+1} X_1 - \proj_{V_{-j+2}} f_{j-1} \proj_{V_{-j+1}} f_{j} X_2 )\|_2^2
\end{align*}
due to the triangular form and the orthogonality of the multi-resolution basis.
Let $\nu_{-j+1} = \proj_{V_{-j+1}} \nu_{-j}$, then as $\proj_{V_{-j+1}}$ commutes with any term equivalent to the identity operator on $V_{-j+1}$, the first term becomes
\begin{align}
    \| f_{j-1} X_{1, t_{j+1}} - \proj_{V_{-j+2}}f_{j-1} X_{2, t_{j+1}} )\|_2^2,
\end{align}
where $X_{1, t_{j+1}},X_{2, t_{j+1}} \sim \nu_{-j+1}$.
When an optimal coupling is made, this term becomes $\|\proj_{U_{-j+2}}X_{1, t_{j+1}}\|_2^2$.
The second term has $\proj_{U_{-j+1}} \proj_{V_{-j+2}}$ nullified, and again commutes where appropriate making this 
\begin{align}
    \|\proj_{U_{-j+1}}X_{1, t_{j+1}}\|_2^2.
\end{align}
We may again use the triangular form to utilise the identify
\begin{align}
    \|\proj_{U_{-j+1}} F_{j} \|_2^2 \leq \|F_{j}\|_2^2,
\end{align}
to define 
\begin{align}
    \lambda_{j|J} \coloneqq \text{diag}( %
    \|\proj_{U_{-j+1}} F_{j|J}\|_2^2 , \dots,  \|\proj_{U_{-J+1}} F_{J|J}\|_2^2 )
\end{align}
so that
\begin{align}
    \mathcal{W}_2( \lambda^{-1}_{j|J} (F_{j|J} \nu_1) , \lambda^{-1}_{j|J} (F_{j|J} \nu_2) ) \leq  \mathcal{W}_2( \nu_1 , \nu_2).
\end{align}
Piecing the decomposition and scaling together, we yield 
\begin{align}
    \Eb_{\nu_{-j+2}}
    \|\proj_{U_{-j+2}}X_{1, t_{j+1}}\|_2^2/\|F_{j-2|j}\|_2^2
    +
    \Eb_{\nu_{-j+1}}
    \|\proj_{U_{-j+1}}X_{1, t_{j+1}}\|_2^2/\|F_{j-2|j}\|_2^2
    \\
    \leq
    (\mathcal{W}_2 (\nu, \mathbf{B}_{j-2|j} \mathbf{F}_{j-2|j} ))^2.
\end{align}
Iterating over $j$ in the fashion given yields the result. 
Last, measures within
\begin{align}
    \mathcal{U}_{\bm{BF} } \coloneqq \{ \nu_J \, | \, \bm{F}_{j|J} \gamma_J = \overline{F}_{j|J}  \otimes \bigotimes_{i=j}^J \delta_{ \{ 0 \} } \},
\end{align}
are invariant under $\bm{B}_{J|1} \bm{F}_{1|J}$, further, $\bm{B}_{J|1} \bm{F}_{1|J}$ projects onto this set.
To see this, take any measure $\nu_J \in \Db(V_{-J})$ and apply $\bm{F}_{j|J}$.
The information in $V_{-j}^{\perp}$ split by $\bm{P}_{j}$ is replaced by $\delta_{ \{ 0 \} }$ in the backward pass. 
Thus $\bm{B}_{J|1} \bm{F}_{1|J} \bm{B}_{J|1} \bm{F}_{1|J} = \bm{B}_{J|1} \bm{F}_{1|J}$.
 \end{proof}

\subsection{U-Nets in $V_{-J}$}
\label{app:U-Nets in V_-J}

Here we show how U-Nets can be seen as only computing the non-truncated components of a multi-resolution diffusion bridge on $V_{-J}$ --- the computations are performed in $V_{-j}$ for $j< J$ at various layers.
This amounts to showing the embedding presented in Theorem \ref{thm:id-diff}.

\textbf{Theorem \ref{thm:id-diff}. }
Let $B_j: [t_j, t_{j+1}) \times \Db(V_{-j}) \mapsto \Db(V_{-j})$ be a linear operator (such as a diffusion transition kernel, see Appendix \ref{app:Proofs}) for $j<J$ with coefficients $\mu^{(j)}, \sigma^{(j)} :[t_j, t_{j+1}) \times V_{-j} \mapsto V_{-j}$, and define the natural extensions within $V_{-J}$ in bold, i.e. $\bm{B}_j \coloneqq B_{j} \oplus \bm{I}_{V_{-j}^{\perp} }$.
Then the operator $\bm{B}: [0,T] \times \Db(V_{-J}) \mapsto \Db(V_{-J})$ and the coefficients $\bm{\mu}, \bm{\sigma} :[0,T] \times V_{-J} \mapsto V_{-J}$ given by

{\centering
  $ \displaystyle
    \begin{aligned} 
    \bm{B} \coloneqq \sum_{j = 0}^{J} \1_{[t_j, t_{j+1})} \cdot \bm{B}_j,
    &&
    \bm{\mu} \coloneqq \sum_{j = 0}^{J} \1_{[t_j, t_{j+1})} \cdot \bm{\mu}^{(j)},
    &&
    \bm{\sigma} \coloneqq \sum_{j = 0}^{J} \1_{[t_j, t_{j+1})} \cdot \bm{\sigma}^{(j)},
\end{aligned}
  $ 
\par}
induce a multi-resolution bridge of measures from the dynamics for $t \in [0,T]$ and on the standard basis as $dX_t = \bm{\mu}_t(X_t) dt + \bm{\sigma}_t(X_t) dW_t$ (see Appendix \ref{app:U-Nets in V_-J} for the details of this integration) for $X_t \in V_{-J}$, i.e. a (backward) multi-resolution diffusion process. 

\begin{proof}
At time $t=0$ we have that $\text{supp} \nu_0 \subset V_0 = \{0\}$, so $\Db(V_0) = \delta_{\{0\}}$.
For the $s$ in the first time interval $[t_0, t_1)$ it must be the case $\nu_{s} = \delta_{\{0\}}$, so $\mu^{(j)}_s, \sigma^{(j)}_s = 0$ and $B_0(s) \equiv I$.
The extension is thus $\bm{B}_0(s) \equiv I$ on $V_{-J}$.
At $t = t_1$, the operator $\bm{B}_1 \equiv I$ on $V_{1}^{\perp}$ grants $\bm{\mu}_s, \bm{\sigma}_s =0 $ here,
On $V_{1}$, $\bm{B}_1$, it is an operator with domain in $V_{1}$, granting $\text{supp} \nu_{t_1} \subset V_1$.
For $s$ within the interval $(t_1,t_2)$ we maintain $\text{supp} \nu_{s} \subset V_1$, and by induction we can continue this for any $s \in [t_{j},t_{j+1}) \subsetneq [0,1]$.
Let $E_{j}$ be a basis of $V_{-j}$, then as $\bm{\mu}_s, \bm{\sigma}_s = 0$ on $V_{-j}^{\perp}$ the diffusion SDE on $[t_{j},t_{j+1}) \times V_{-j}$ given in the basis $E_{j}$ by 
\begin{align}\label{eq:XjSDE}
    dX_t^{(j)} = {\mu}_t^{(j)}(X_t) dt + {\sigma}_t^{(j)}(X_t) dW_t
\end{align}
embeds into an SDE on $V_{-J}$ with basis $E_{J}$ by 
\begin{align}\label{eq:XSDE}
    dX_t = \bm{\mu}_t(X_t) dt + \bm{\sigma}_t(X_t) dW_t,
\end{align}
which maintains $X_t \in V_{-j}$ as $\bm{\sigma}_t \equiv 0$ on the complement.
\end{proof}

In practice, we will compute the sample paths made from Equation \ref{eq:XjSDE}, but we can in theory think of this as Equation \ref{eq:XSDE}.
The U-Net sequential truncation of spaces, then sequential inclusion of these spaces is what forms the multi-dimensional bridge with our sampling models.

\subsection{Forward Euler Diffusion Approximations}
\label{app:Forward Euler Diffusion Approximations}

Here we show that the backward cell structure of state-of-the-art HVAEs approximates an SDE evolution within each resolution. 

\textbf{Theorem \ref{thm:vdvae_sde}. }
Let $t_{J} \coloneqq T \in (0,1)$ and consider (the $p_{\theta}$ backward pass) $\bm{B}_{\theta,1|J}: \Db(V_{-J}) \mapsto \Db(V_{0})$ given in multi-resolution Markov process in the standard basis:
\begin{align}
d Z_t = (\bmu_{1,t} (Z_t) + \bmu_{2,t} (Z_t) )dt + \bsig_t(Z_t) dW_t,  \label{eq:coupled_sdes}  
\end{align}
where  $\proj_{U_{-j+1}}Z_{t_j} = 0,  \ \|Z_t \|_2 > \|Z_s \|_2$ with $0\leq s<t \leq T$ and for a measure $\nu_{J} \in \Db(V_{-J})$ we have $ Z_0 \sim \bm{F}_{\phi,J|1}\nu_{J}$.
Then, VDVAEs approximates this process, and its residual cells are a type of two-step forward Euler discretisation of this Stochastic Differential Equation (SDE). %

\begin{proof}
The evolution 
\begin{align}
    d Z_t =  (\bmu_{1,t} (Z_t) + \bmu_{2,t} (Z_t) )dt + \bsig_t(Z_t) dW_t,
\end{align}
subject to $Z_0 = 0,  \|Z_t \|_2 > \|Z_s \|_2$ and $X_T, \, Z_0 \sim \bm{F}_{\phi,J|1}\nu_{J}$.
By Theorem \ref{thm:truncation} we know $\bm{F}_{\phi,J|1}\nu_{J}$ enforces the form  
\begin{align}
    \proj_{V_1} \overline{F}_{\phi,J|1}\nu_{J} \otimes \bigotimes_{j=1}^{J-1} \delta_{\{0 \} }
\end{align}
when $\phi$ is trained with a reconstruction loss. 
By Theorem \ref{thm:discrete-VDVAE}, the full cost used imposes $\proj_{V_1} \overline{F}_{\phi,J|1}\nu_{J} = \delta_{\{0 \} }$, further, VDVAE initialises $Z_0 = \delta_{ \{ 0 \} }$.
This enforces $Z_0 = 0$ as $Z_0 \sim \delta_{\{0\}}$.
For the backward SDE, consider the splitting 
\begin{align}
    d Z_t^{(1)} =\bmu_{1,t} (Z_t^{(1)})dt + \bsig_t(Z_t^{(1)}) dW_t,
    &&
    d Z_t^{(2)} = \bmu_{2,t} (Z_t^{(2)}) dt,
\end{align}
where $dZ_t = d Z_t^{(1)} + d Z_t^{(2)}$ when $Z_t = Z_t^{(1)} = Z_t^{(2)}$.
For the split SDE make the forward-Euler discretisation 
\begin{align}
     Z_{i+1}^{(1)} = Z_{i}^{(1)} + \int_{i}^{i+1} \bmu_{1,t} (Z_t^{(1)}) dt + \int_{i}^{i+1} \bsig_t( Z_{i}^{(1)}) dW_t \approx Z_{i}^{(1)} + \bmu_{1,i} (Z_i^{(1)})+  \bsig_i( Z_{i}^{(1)}) (W_1).
\end{align}
Now the second deterministic component can also be approximated with a forward-Euler discretisation 
\begin{align}
     Z_{i+1}^{(2)}  = Z_{i}^{(2)}  + \int_{i}^{i+1} \bmu_{2,t} (Z_t^{(2)}) dt \approx Z_{i}^{(2)} +  \bmu_{2,i}  (Z_{i}^{(2)}).
\end{align}
As $Z_0 = 0$, we need only show the update at a time $i$, so assume we have $Z_i$. 
First we update in the SDE step, so make the assignment and update
\begin{align}
    Z_{i}^{(1)} \leftarrow Z_{i}, && Z_{i+1}^{(1)} =  Z_{i}^{(1)} + \bmu_{i,1} (Z_i^{(1)}) + \bsig_i(Z_i^{(1)}) \Delta W_1.
\end{align}
Now assign $Z_i^{(2)} \leftarrow Z_{i+1}^{(1)}$ so we may update in the mean direction with 
\begin{align}
    Z_{i+1}^{(2)} = Z_{i}^{(2)} + \bmu_{i,2} (Z_i^{(2)}) ,
\end{align}
with the total update $Z_{i+1} \leftarrow Z_{i+1}^{(2)}$.
This gives the cell update for NVAE in Figure \ref{fig:hvae_cells_all}.
To help enforce the growth $\|Z_t \|_2 > \|Z_s \|_2$, VDVAE splits $Z_{i}^{(1)} = Z_{i} + Z_{i,+}$ where $Z_{i,+}$ increases the norm of the latent process $Z_t$.
This connection and the associated update are illustrated in Figure \ref{fig:hvae_cells_all} [left].
Note here that if no residual connection through the cell was used (just the re-parameterisation trick in a VAE), then we degenerate to a standard Markovian diffusion process and yield the Euler-Maruyama VAE structure in Figure \ref{fig:hvae_cells_all} [right].

\begin{remark}
To simplify the stepping notation in the HVAE backward cells (Figures \ref{fig:hvae_cell_vdvae} and \ref{fig:hvae_cells_all}), we use $Z_{i}^{(1)} = Z_{i} + \bmu_{1,i} (Z_i)+  \bsig_i( Z_{i}) (W_1)$ and $Z_{i}^{(2)} = Z_{i}^{(1)} + \bmu_{i,2} (Z_i^{(1)}) $ so that the index $i$ refers to all computations of the $i^{th}$ backward cell. 
\end{remark}

\end{proof}

\subsection{Time-homogenuous model}
\label{app:Time-homogenuous model}

Recall VDVAE has the continuous time analogue
\begin{align}
d Z_t =  (\bmu_{1,t} (Z_t) + \bmu_{2,t} (Z_t) )dt + \bsig_t(Z_t) dW_t,
\end{align}
where  $Z_0 = 0,   \ \|Z_t \|_2 > \|Z_s \|_2$ with $0\leq s<t \leq T$ and for a measure $\nu_{J} \in \Db(V_{-J})$.
Due to Theorem \ref{thm:discrete-VDVAE}, we know that the initial condition of VDVAE's U-Net is the point mass $\delta_{\{0\}}$.
As the backwards pass flows from zero to positive valued functions, this direction is increasing and the equation is stiff with few layers.
The distance progression from zero is our proxy for time, and we can use its `position' to measure this.
Thus, the coefficients $ \bmu_{t,1},\bmu_{t,2}, \bsig_t$ need not have a time dependence as this is already encoded in the norm of the $Z_t$ processes.
Thus, the time-homogeneous model postulated in the main text is: 
\begin{align}
    d Z_t = (\bmu_{1} (Z_t) + \bmu_{2} (Z_t)) dt + \bsig(Z_t) dW_t,  \label{eq:time_homog_1}
    \\
    Z_0 = 0,  \  \|Z_t \|_2 > \|Z_s \|_2  \label{eq:time_homog_2}.
\end{align}
In practice, the loss of time dependence in the components corresponds to weight sharing the parameters across time, as explored in the experimental section. 
Weight sharing, or a time-homogeneous model, is common for score based diffusion models \cite{sohl2015deep,ho2020denoising}, and due to our identification we are able to utilise this for HVAEs.

\subsection{HVAE Sampling}
Here we use our framework to comment on the sampling distribution imposed by the U-Net within VDVAE. 

\textbf{Theorem \ref{thm:discrete-VDVAE}. }
Consider the SDE in Eq.~\eqref{eq:coupled_sdes}, trained through the ELBO in Eq. \ref{eq:hvae_elbo}.
Let $\tilde{\nu}_{J}$ denote the data measure and $\nu_{0}= \delta_{\{0\}}$ be the initial multi-resolution bridge measure imposed by VDVAEs.
If $q_{\phi,j}$ and $p_{\theta,j}$ are the densities of $B_{\phi,1|j}\bm{F}_{J|1}\tilde{\nu}_{J}$ and  $B_{\theta,1|j}\nu_{0}$ respectively, then a VDVAE optimises the boundary condition $\min_{\theta,\phi} KL(q_{\phi,0,1} || q_{\phi,0}p_{\theta,1} )$, where a double index indicates the joint distribution.
\begin{proof}
We need to only show two things.
First, due to Theorems \ref{thm:truncation} and \ref{thm:vdvae_sde}, we know that the architecture imposes 
\begin{align}
    \proj_{V_1} \overline{F}_{\phi,J|1}\nu_{J} \otimes \bigotimes_{j=1}^{J-1} \delta_{\{0 \} },
\end{align}
so we must analyse how $\proj_{V_1} \overline{F}_{\phi,J|1}\nu_{J}$ is trained.
Second, we use Theorem \ref{thm:vdvae_sde} to view the discretised version of the continuous problem, and identify the error in the two-step forward Euler splitting. 

On the first point, VDVAE uses an ELBO reconstruction with a KL divergence between the backwards pass of the data $\overline{B}_{\phi,J|1}\overline{F}_{\phi,J|1}\tilde{\nu}_{J}$ (the `$q_{\phi}$-distribution'), and the backwards pass of the model imposed by the U-Net $\overline{B}_{\phi,J|1}\nu_{0}$ (the `$p_{\theta}$ distribution').
As $Z_0$ is zero initialised, we know $\nu_{0} = \delta_{\{0\}}$.
We need to show the cost function used imposes this initialisation  on $\overline{B}_{\phi,1|0}\overline{F}_{\phi,J|0}\tilde{\nu}_{J}$.
Let $X_T \sim \overline{F}_{\phi,J|0}\tilde{\nu}_{J}$, call the distribution of this $q_{\phi,0}$.
We also use $Z_1 \sim q_{\phi,1}$ for a sample from $\overline{B}_{\phi,1|0}\overline{F}_{\phi,J|0}\tilde{\nu}_{J}$ and 
$Z_1 \sim p_{\theta,1}$ for a sample from $\overline{B}_{\phi,1|0} \nu_0$.
For a realisation $x$ of $X_T$, VDVAE computes 
\begin{align}
    KL(q_{\phi,1|0} (\sds | X_T = x) || p_{\theta,1|0}(\sds | Z_0 = 0) ) = KL(q_{\phi,1|0} (\sds | X_T = x) || p_{\theta,1}(\sds) ),
\end{align}
which in training is weighted by each datum, so the total cost in this term is 
\begin{align}
    \int KL(q_{\phi,1|0} (\sds | X_T = x) || p_{\theta,1}(\sds) ) q_{\phi,0}(X_T = x) dx.
\end{align}
But this is equal to,
\begin{align}
    &\int \int \log
    \left(
    \frac{q_{\phi,1|0}(Z_1 = z | X_T = x)}{p_{\theta,1}(Z_1 = z_1)}
    \right)
    q_{\phi,1|0}(Z_1 = z | X_T = x) 
    q_{\phi,0}(X_T = x)
    dz
    dx
    \\
    &=
     \int \int \log
    \left(
    \frac{q_{\phi,0,1}(Z_1 = z , X_T = x)}{p_{\theta,1}(Z_1 = z_1)q_{\phi,0}(X_T = x)}
    \right)
    q_{\phi,0,1}(Z_1 = z , X_T = x) 
    dz
    dx
    \\
    &= KL(
    q_{\phi,0,1}(Z_1  , X_T )
    ||
    p_{\theta,1}(Z_1 )q_{\phi,0}(X_T ))
    =
    KL(
    q_{\phi,0,1}
    ||
    p_{\theta, 1}q_{\phi,0}).
\end{align}
The distribution of $p_{\theta,1}$ is Gaussian as a one time step diffusion evolution from the initial point mass $\nu_{0} = \delta_{\{0\}}$. 

\end{proof}
Theorem \ref{thm:id-diff} states that the choice of the initial latent variable in VDVAE imposes a boundary condition on the continuous SDE formulation.
Further, this boundary condition is enforced into the final output $X_T$ of the encoder within VDVAE.

\begin{landscape}
\vspace{1cm}
\begin{figure}[h!]
    \centering
    \includegraphics[width=1\linewidth]{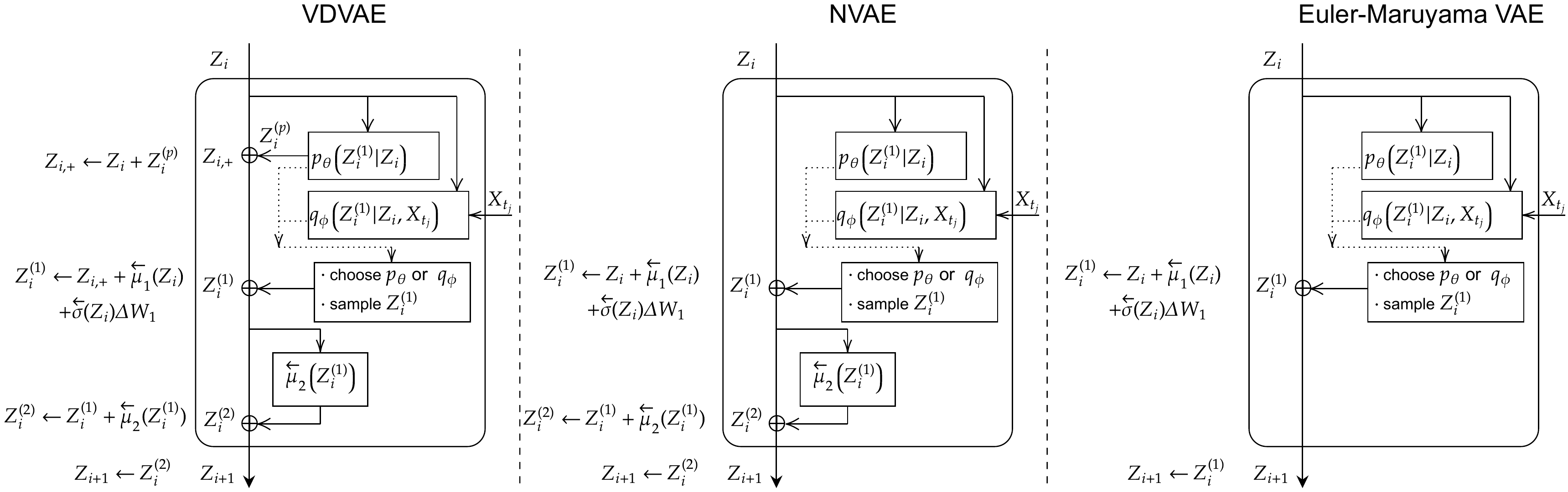}  
    \caption{
    HVAE top-down cells are resembling two-step forward Euler discretisations of a continuous-time diffusion process in Eq. \ref{eq:coupled_sdes}.
    We here provide the residual cell structures of [left] VDVAE \cite{Child2020VeryImages}, [middle] NVAE \cite{Vahdat2020NVAE:Autoencoder} and [right] a Euler-Maruyama VAE.
    Either $q_\phi$ (conditional) or $p_\theta$ (unconditional) are used in the sampling step (indicated by the dotted lines) during training and generation, respectively.
    $X_{t_j}$ is an input from the in effect non-stochastic bottom-up pass, $Y_i$ is the input from the previous, and $Y_{i+1}$ the output to the next residual cell.
    $\oplus$ indicates element-wise addition.
    }
    \label{fig:hvae_cells_all}
\end{figure}
\end{landscape}

\newpage

\newpage

\section{Background}
\label{sec:Background}

\ff{
}

\subsection{Multi-Resolution Hierarchy and thought experiment}
\label{app:Thought experiment on multiresolution analysis}

Let $\Xb \subset \R^m$ be compact and $L^2(\Xb)$ be the space of square-integrable functions over this set. 
We are interested in decomposing $L^2(\Xb)$ across multiple resolutions.

\textbf{Definition \ref{def:multi-res-approx-space} (abbreviated).}
A \textit{multi-resolution hierarchy} is one of the form 
\begin{align}
    \cdots \subset V_1 \subset V_0 \subset V_{-1} \subset \cdots \label{assumption:nesting}\\
    \overline{\bigcup_{j \in \Z} V_{-j}} = L^2(\R^m) \\
    \bigcap_{j \in \Z} V_{-j} = \{0 \} \\
    f (\cdot) \in V_{-j} \iff f(2^j \, \cdot \, ) \in V_0 \\
    f (\cdot) \in V_0 \iff f( \, \cdot \, -n ) \in V_0\text{, for }n \in \Z. 
\end{align}

Each $V_{-j}$ is a finite truncation of $L^2(\Xb)$.
What we are interested in is to consider a function $f \in L^2(\Xb)$ and finding a finite dimensional approximation in $V_{-J}$, say, for $J>0$. 
Further, for gray-scale images, $\Xb = [0,1]^2$, the space of pixel-represented images.
To simplify notation, we just consider $\Xb = [0,1]$ for the examples below, but we can extend this to gray-scale images, and to colour images with a Cartesian product.

The `pixel' multi-resolution hierarchy is given by the collection of sub-spaces
\begin{align}\label{eq:Haar-1d-Vj}
    V_{-j} 
    = 
    \{
    f \in L^2([0,1]) \ | \ f|_{[2^{-j} \cdot k, 2^{-j} \cdot (k+1) )} = c_k , \, k \in \{0, \dots, 2^{j}-1 \}, \, c_k \in \R 
    \}.
\end{align}
It can be readily checked that these sub-spaces obey the assumptions of Definition \ref{def:multi-res-approx-space}.
An example image projected into such sub-spaces, obtained from a discrete Haar wavelet transform, is illustrated in Fig. \ref{fig:thought_experiment}.
We call it the pixel space of functions as elements of this set are piece-wise constant on dyadically split intervals of resolution $2^{j}$, i.e a pixelated image. 
For each $V_{-j}$ there is an obvious basis of size $2^{j}$ where we store the coefficients $(c_0, c_1, \dots, c_{2^j-1}) \in \R^{2^j}$.
The set of basis vectors for it is the \textit{standard basis} $\{e_i \}_{i=0}^{2^j-1}$ which are $0$ for all co-ordinates except for the $i^{th}+1$ entry which is $1$.
This basis is not natural to the multi-resolution structure of $V_{-j}$. 
This is because all the basis functions change when we project down to $V_{-j+1}$.
We want to use the multi-resolution structure to create a basis which naturally relates $V_{-j}$, $V_{-j+1}$, and any other sub-space. 
To do this consider $V_{-j} \cap V_{-j+1}^{\perp} \subset V_{-j}$. 
Define this orthogonal compliment to be $U_{-j+1} \coloneqq V_{-j+1}^{\perp}$, then see $V_{-j} = V_{-j+1} \oplus U_{-j+1}$. 
Doing this recursively finds $V_{-j} = V_0 \oplus \bigoplus_{i=0}^{-j+1} U_i$, and taking the limit 
\begin{align}
    L^2(\Xb) = \bigoplus_{i=0}^{-\infty} U_i \oplus V_0.
\end{align}
Each of the sub-spaces $\{U_{-j}\}_{j=0}^{\infty}$ are mutually orthogonal as each $V_{-j} \perp U_{-j}$.
Now suppose we had a basis set $\bm{\Psi}_j$ for each $U_{-j}$ and $\bm{\Phi}_0$ for $V_0$.
As these spaces are orthogonal, so are the basis sets to each other, too.
We can make a basis for $V_{-j}$ with $\text{span}(\bm{\Phi}_0, \bm{\Psi}_0, \cdots, \bm{\Psi}_{-j+1})$.
For the above examples, $V_0$ needs only a single basis function $\phi_{0,k} = \1_{[k,k+1)}$, further if $\psi = \sqrt{2} (\1_{[0,1/2)} - \1_{[1/2,1)})$, then given the functions $x \mapsto \psi_{j,k}(x) \coloneqq 2^{j/2} \cdot \psi (2^{-j} (x - k))$ we have $\{\psi_{j,k}\}$ is a basis for $V_{-j}$.

\clearpage

\begin{figure}[ph!]
    \centering
    \begin{tabular}{cc}
    \includegraphics[width=.3\linewidth]{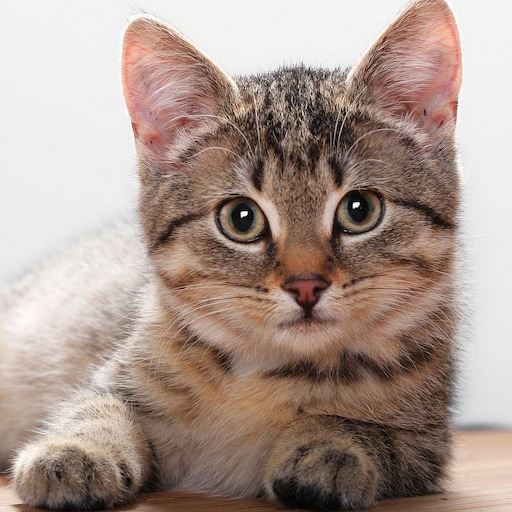} & 
    \includegraphics[width=.3\linewidth]{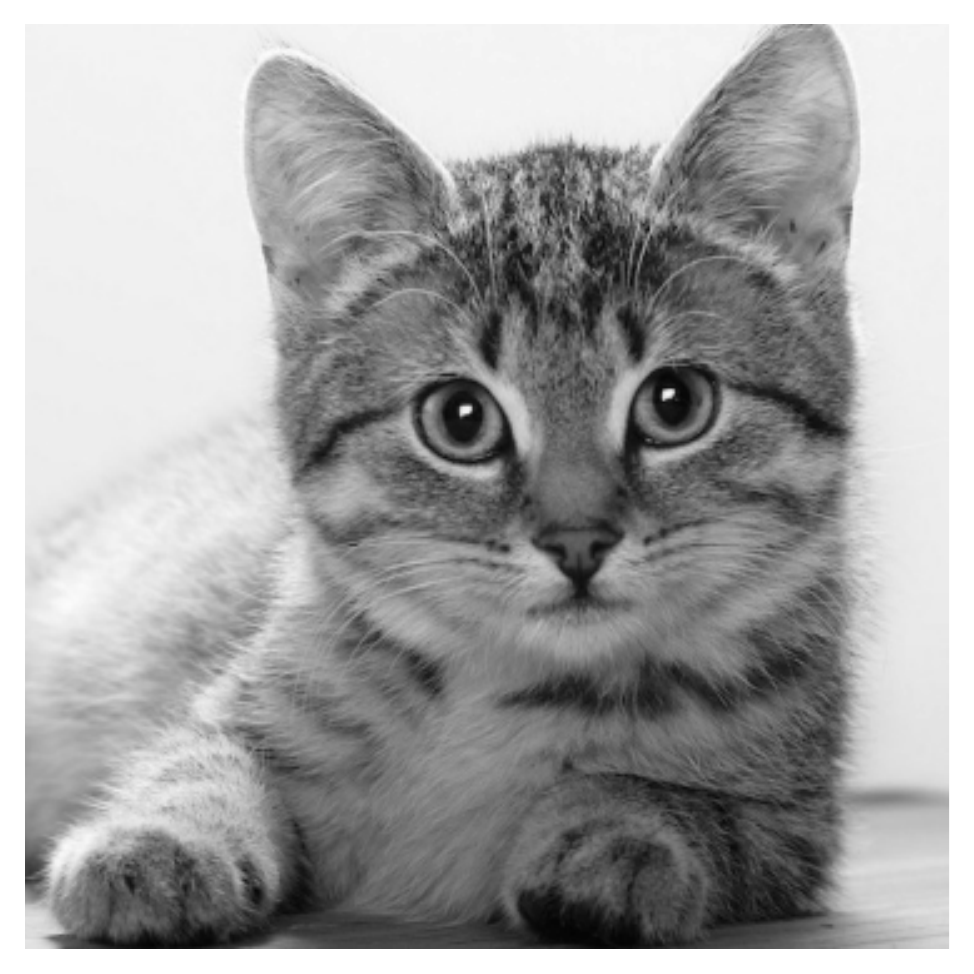} \\
    original $(512 \times 512)$ & gray-scaled original $(512 \times 512)$ \\
    \end{tabular}
    \begin{tabular}{ccc}
    \includegraphics[width=.3\linewidth]{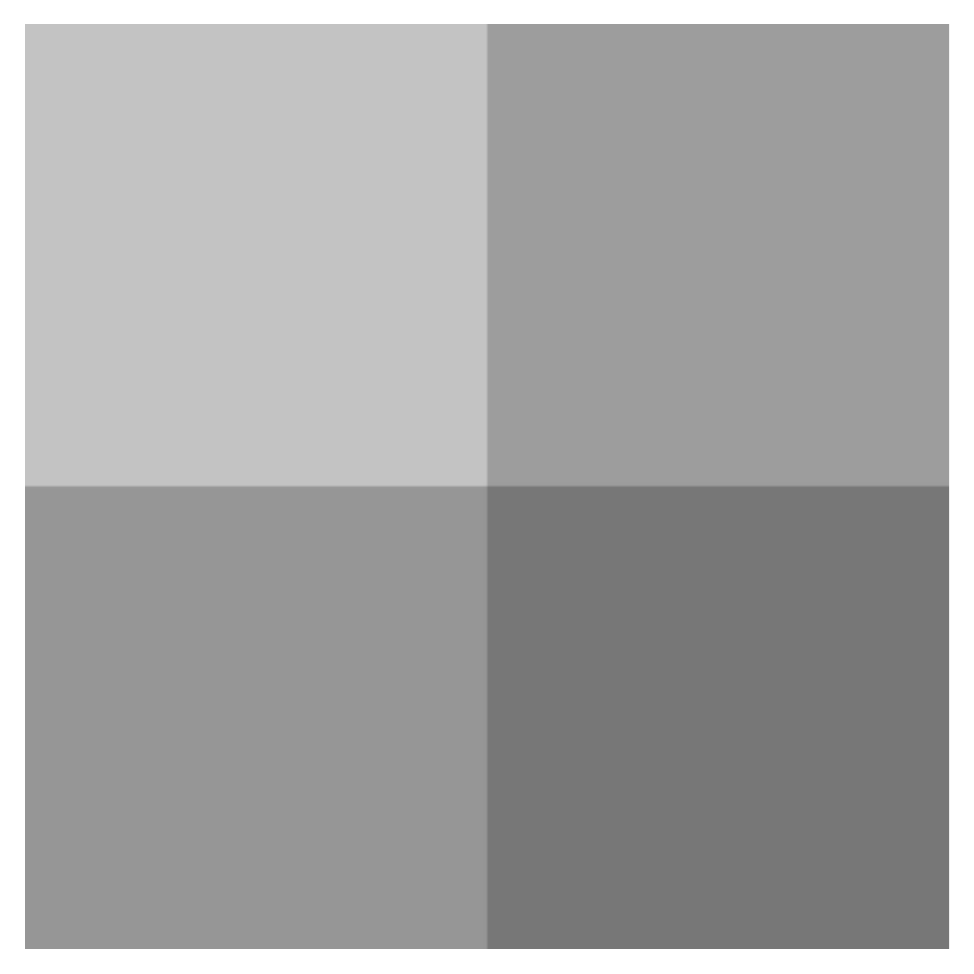} &
    \includegraphics[width=.3\linewidth]{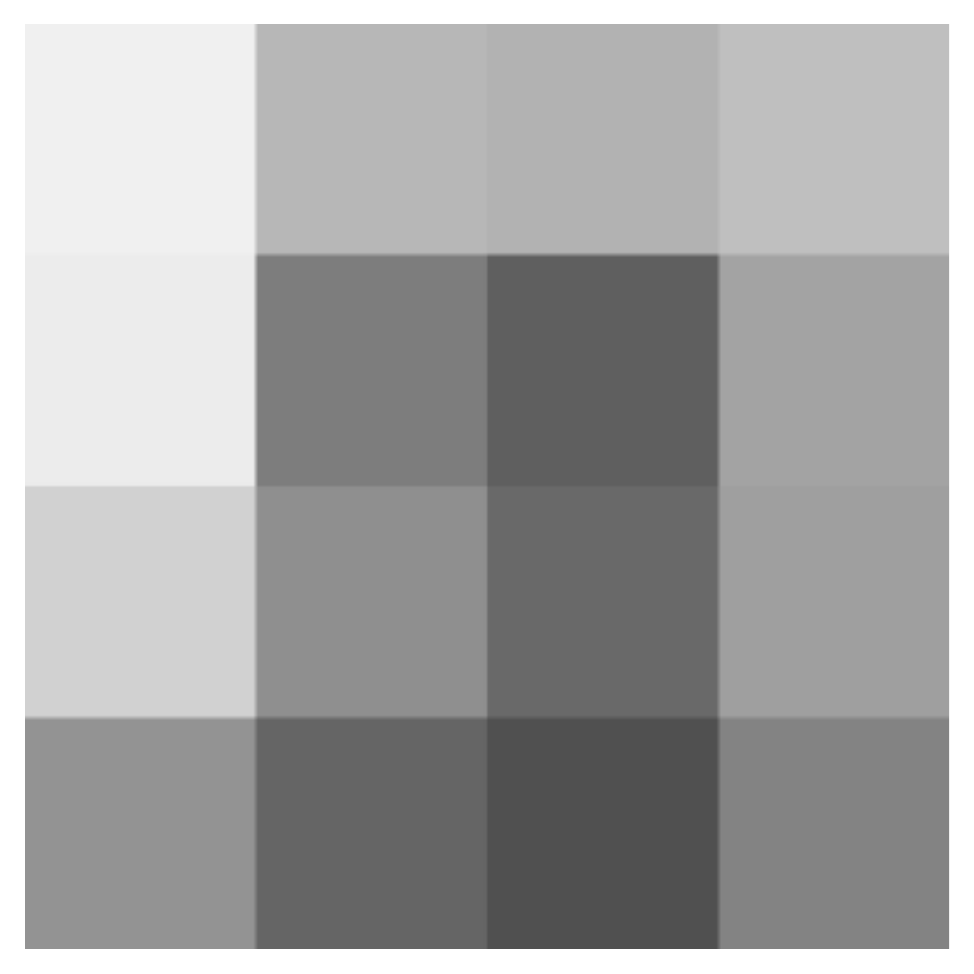} & \includegraphics[width=.3\linewidth]{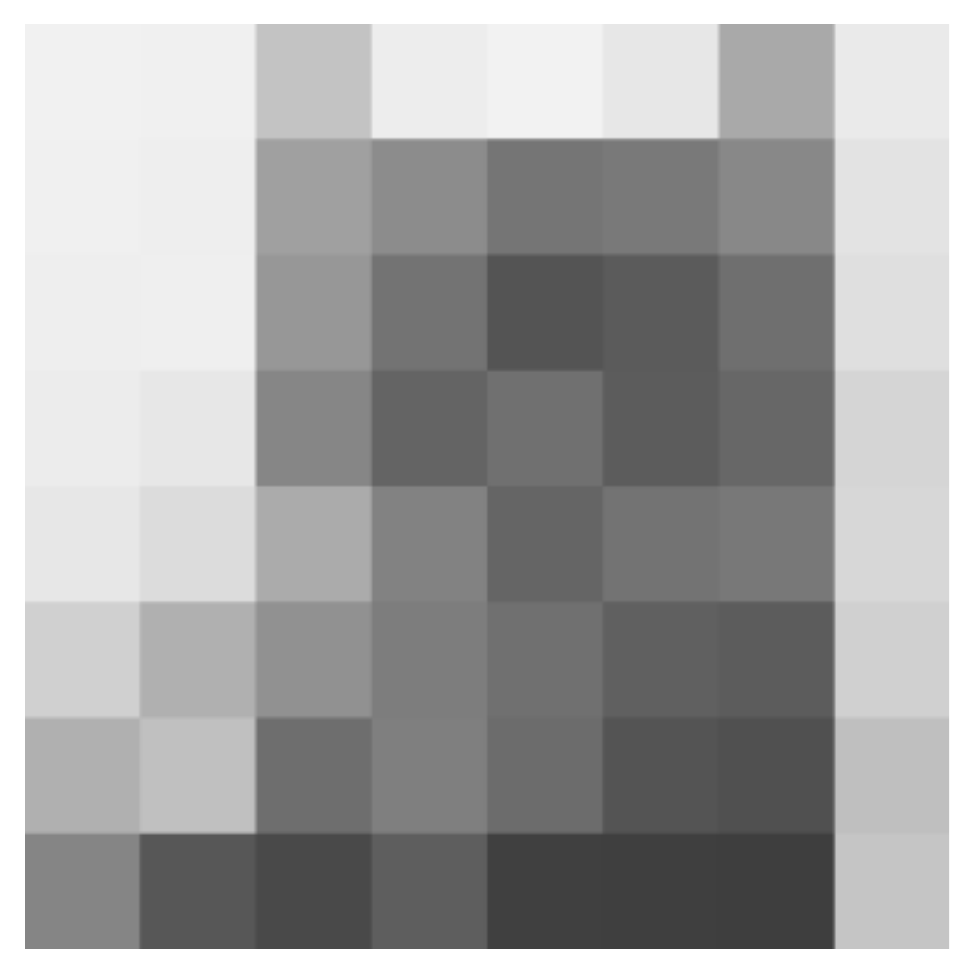} \\
    $V_{-1} \text{ } (2 \times 2)$ & $V_{-2} \text{ } (4 \times 4)$ & $V_{-3} \text{ } (8 \times 8)$  \\
    \includegraphics[width=.3\linewidth]{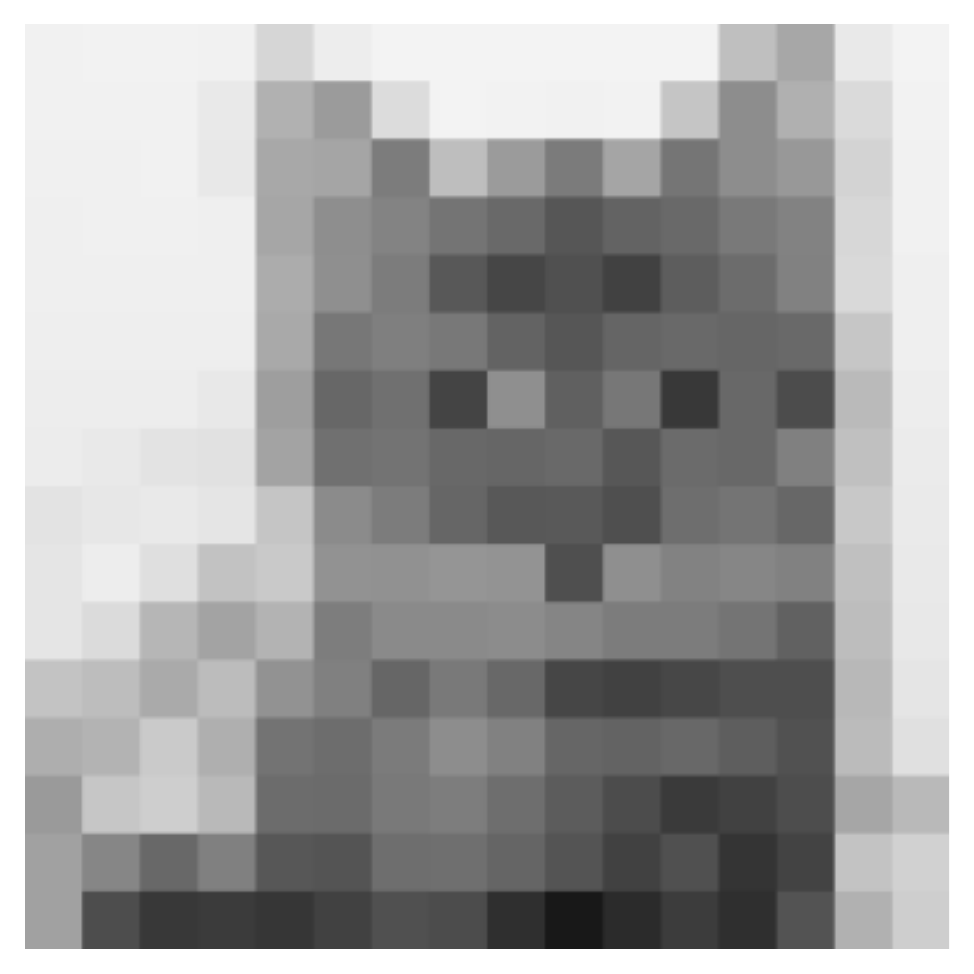} & \includegraphics[width=.3\linewidth]{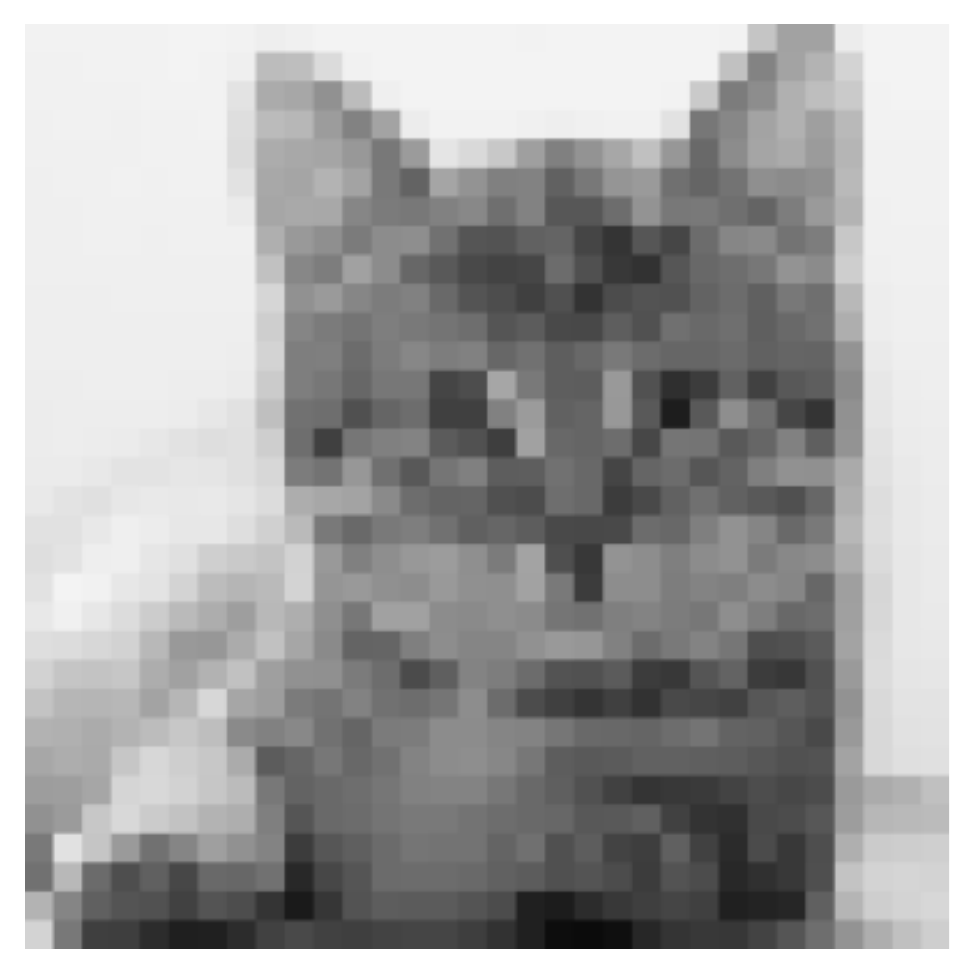} & 
    \includegraphics[width=.3\linewidth]{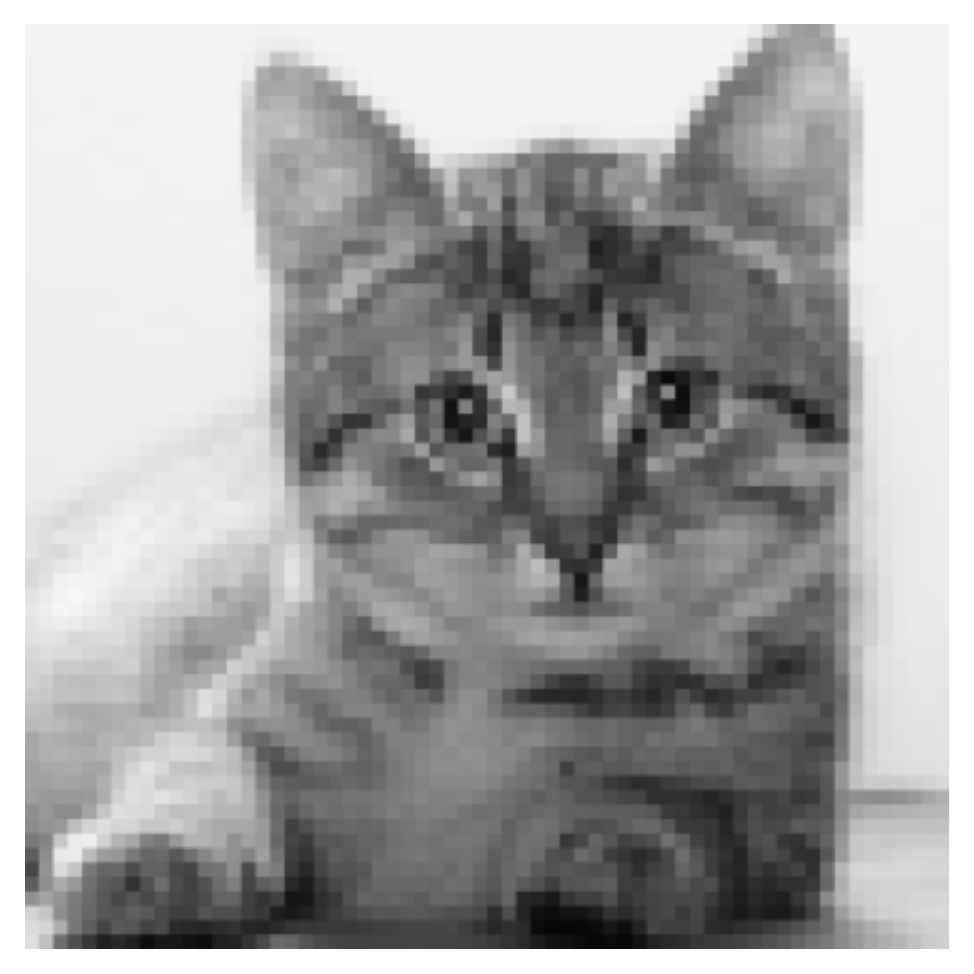}  \\
    $V_{-4} \text{ } (16 \times 16)$ & $V_{-5} \text{ } (32 \times 32)$ & $V_{-6} \text{ } (64 \times 64)$ \\
     \includegraphics[width=.3\linewidth]{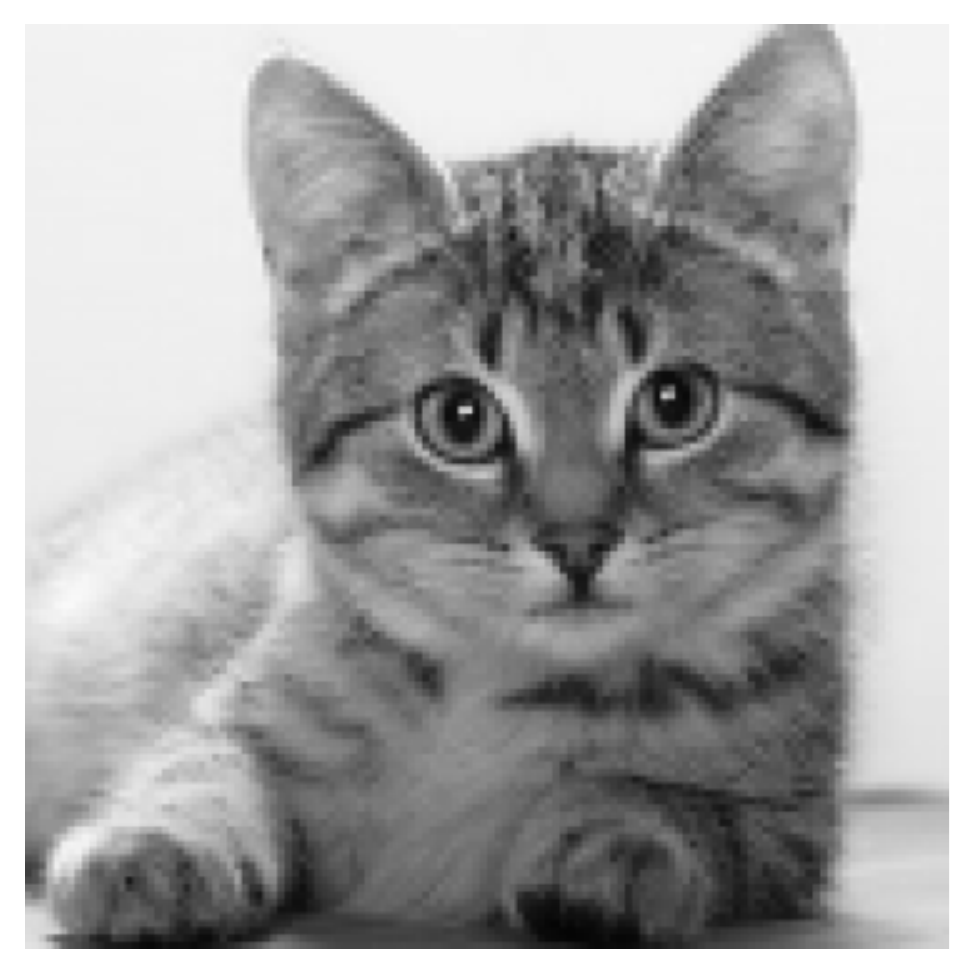} & 
     \includegraphics[width=.3\linewidth]{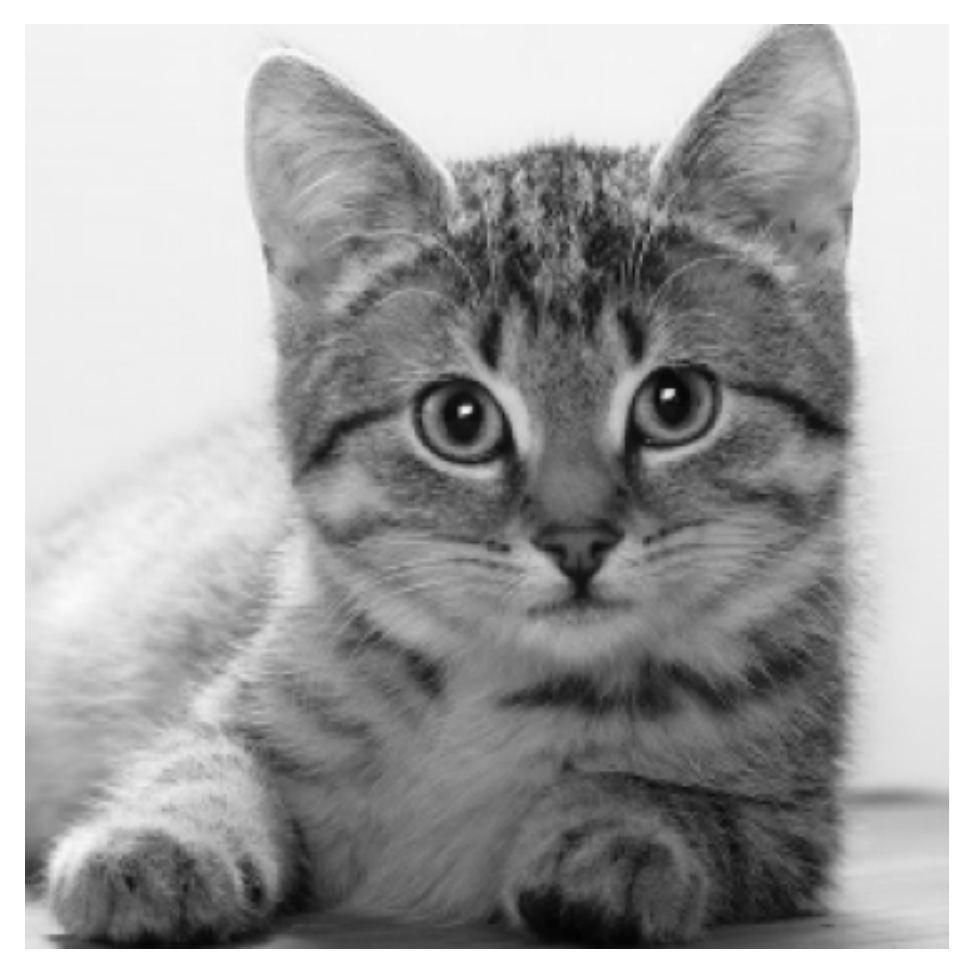} & 
     \includegraphics[width=.3\linewidth]{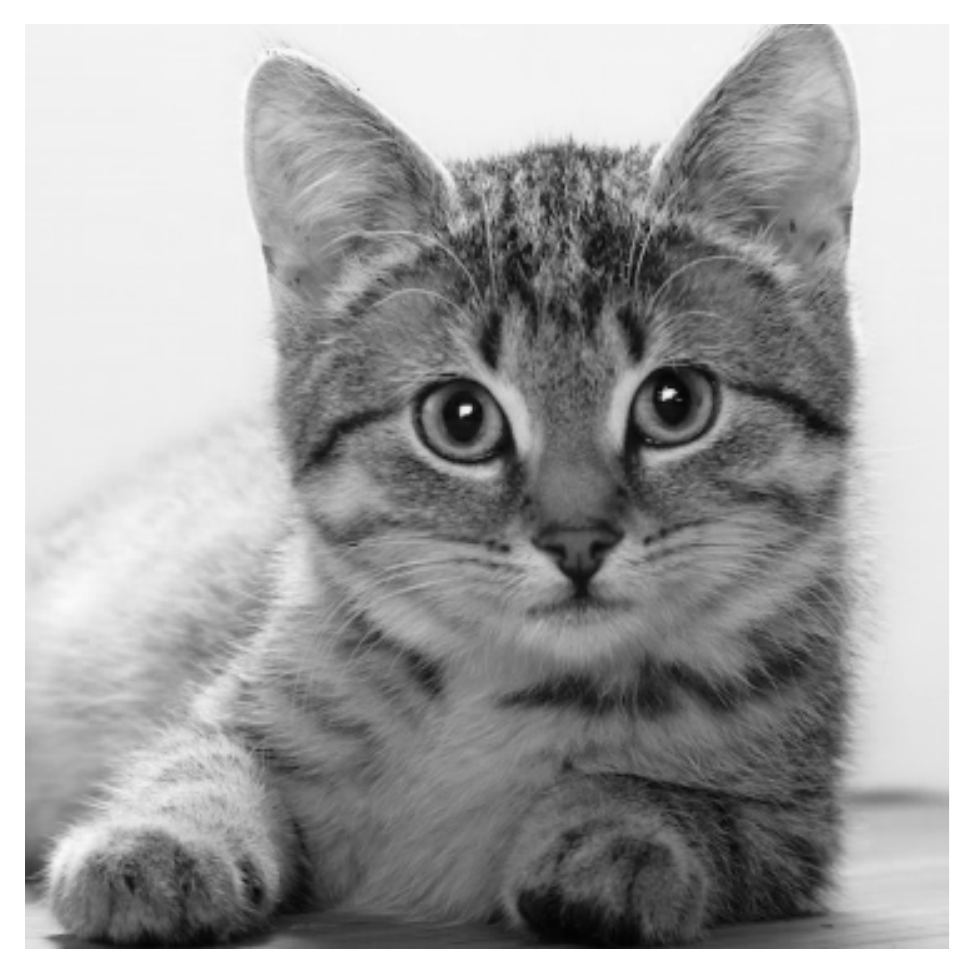} \\
     $V_{-7} \text{ } (128 \times 128)$ & $V_{-8} \text{ } (256 \times 256)$ & $V_{-9} \text{ } (512 \times 512)$ \\
    \end{tabular}
    \caption{
    The thought experiment discussed in \S \ref{sec:multi-res}.
    The original colour image [top-left], its gray-scale version [top-right], and its
    Haar wavelet projections to the approximation spaces $V_{-j}$ for $j \in \{ 1, \dots, 9 \}$. 
    }
    \label{fig:thought_experiment}
\end{figure}

\newpage
\subsection{U-Net}
\label{app:U-Net Model}

\begin{wrapfigure}[36]{r}{.45\textwidth}
\vspace{-2em}
\[\begin{tikzcd}
	\textcolor{rgb,255:red,58;green,51;blue,255}{V_{-J}} & {V_{-J}} && {V_{-J}} \\
	\textcolor{rgb,255:red,58;green,51;blue,255}{\text{out}} & {\text{out}} && {\text{in}} \\
	\textcolor{rgb,255:red,58;green,51;blue,255}{V_{-j}} & {V_{-j}} && {V_{-j}} \\
	\\
	\textcolor{rgb,255:red,58;green,51;blue,255}{V_{-j}} & {V_{-j}} && {V_{-j}} \\
	\textcolor{rgb,255:red,58;green,51;blue,255}{V_{-j+1}} & {V_{-j+1}} && {V_{-j+1}} \\
	\textcolor{rgb,255:red,58;green,51;blue,255}{\text{out}} & {\text{out}} && {\text{in}} \\
	\textcolor{rgb,255:red,58;green,51;blue,255}{V_0} & {V_0} && {V_0}
	\arrow["{P_{-j+1}}", from=5-4, to=6-4]
	\arrow["{E_{-j}}", from=6-2, to=5-2]
	\arrow["{B_{j,\theta}}", shorten <=10pt, shorten >=10pt, from=5-2, to=3-2]
	\arrow[draw={rgb,255:red,128;green,128;blue,128}, from=2-4, to=3-4]
	\arrow[draw={rgb,255:red,128;green,128;blue,128}, from=3-2, to=2-2]
	\arrow[draw={rgb,255:red,128;green,128;blue,128}, from=6-4, to=7-4]
	\arrow[draw={rgb,255:red,128;green,128;blue,128}, from=7-2, to=6-2]
	\arrow[dotted, no head, from=7-4, to=8-4]
	\arrow["{\text{Bend / Bottleneck}}"', from=8-4, to=8-2]
	\arrow[dotted, no head, from=8-2, to=7-2]
	\arrow[dotted, no head, from=2-4, to=1-4]
	\arrow["{B_{j,\phi}}", color={rgb,255:red,58;green,51;blue,255}, shorten <=10pt, shorten >=10pt, from=5-1, to=3-1]
	\arrow[draw={rgb,255:red,58;green,51;blue,255}, from=6-1, to=5-1]
	\arrow[draw={rgb,255:red,58;green,51;blue,255}, from=7-1, to=6-1]
	\arrow[draw={rgb,255:red,58;green,51;blue,255}, dotted, no head, from=8-1, to=7-1]
	\arrow[draw={rgb,255:red,58;green,51;blue,255}, from=3-1, to=2-1]
	\arrow[draw={rgb,255:red,58;green,51;blue,255}, dotted, no head, from=2-1, to=1-1]
	\arrow[dotted, no head, from=2-2, to=1-2]
	\arrow["{F_{j,\theta}}", shorten <=10pt, shorten >=10pt, from=3-4, to=5-4]
	\arrow["{\text{Skip connection}}"', dashed, from=5-4, to=5-2]
\end{tikzcd}\]
\caption{
The repeated structure in a U-Net, where $V_{-j+1}$ is a lower dimensional latent space compared to $V_{-j}$.
$f_{j,\theta}, b_{j,\theta}$ are in practice typically parameterised by neural networks (e.g. convolutional neural networks); $P_{-j+1}$ is a dimension reduction operation (e.g. average pooling) to a lower-dimensional latent space; and, $E_{-j}$ is a dimension embedding operation (e.g. deterministic interpolation) to a higher-dimensional latent space. 
This structure is repeated to achieve a desired dimension of the latent space at the U-Net bottleneck.
}  %
\label{fig:U-Net}
\end{wrapfigure}
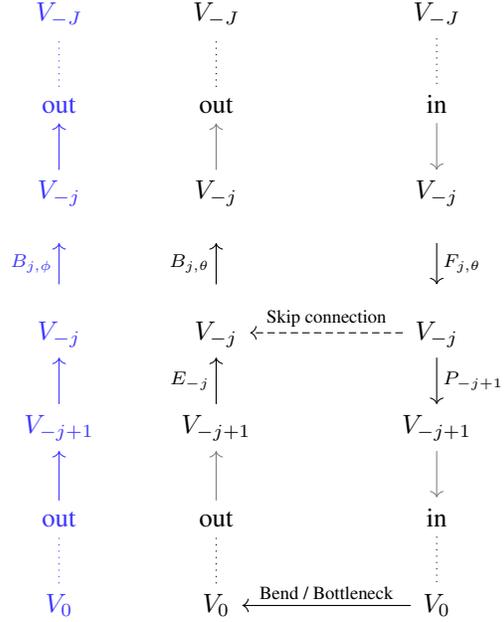

In practice, a U-Net \cite{ronneberger2015u} is a neural network structure which consists of a forward pass (encoder) and backward pass (decoder), wherein layers in the forward pass have some form of dimension reduction, and layers in the backward pass have some form of dimension embedding. 
Furthermore, there are `skip connection' between corresponding layers on the same level of the forward and backward pass.

We now formalise this notion, referring to an illustration of a U-Net in Fig. \ref{fig:U-Net}.
In black, we label the latent spaces to be $V_{-j}$ for all $j$, where the original data is in $V_{-J}$ and the U-Net `bend' (bottleneck) occurs at $V_0$.
We use $f_{j,\theta}$ to be the forward component, or encoder, of the U-Net, and similarly $b_{j,\theta}$ as the backward component or decoder, operating on the latent space $V_{-j}$.
$P_{-j+1}$ refers to the dimension reduction operation between latent space $V_{-j}$ and $V_{-j+1}$, and $E_{-j}$ refers to its corresponding dimension embedding operation between latent spaces $V_{-j+1}$ and $V_{-j}$.
A standard dimension reduction operation in practice is to take $P_{-j+1}$ as average pooling, reducing the resolution of an image. 
Similarly, the embedding step may be some form of deterministic interpolation of the image to a higher resolution. 
We note that the skip connection in Fig. \ref{fig:U-Net} occur before the dimension reduction step, in this sense, lossless information is fed from the image of $f_{j,\theta}$ into the domain of $b_{j,\theta}$.

In blue, we show another backward process $b_{j,\phi}$ that is often present in U-Net architectures for generative models. 
This second backward process is used for unconditional sampling. 
In the context of HVAEs, we may refer to it as the (hierarchical) prior (and likelihood model).
It is trained to match its counterpart in black, without any information from the forward process. 
In HVAEs, this is enforced by a KL-divergence between distributions on the latent spaces $V_{-j}$.
The goal of either backward process is as follows:
\begin{enumerate}
    \item $b_{j,\theta}$ must be able to reconstruct the data from $f_{j,\theta}$, and in this sense it is reasonable to require $b_{j,\theta} f_{j,\theta} = I$;
    \item $b_{j,\phi}$ must connect the data to a known sampling distribution.
\end{enumerate}
The second backward process can be absent when the backward process $b_{j,\theta}$ is imposed to be the inverse of $f_{j,\theta}$, such as in Normalising Flow based models, or reversible score-based diffusion models.
In this case the invertibility is assured, and the boundary condition that the encoder connects to a sampling distribution must be enforced. 
For the purposes of our study, we will assume that in the absence of dimension reduction, the decoder is constrained to be an inverse of the encoder. 
This is a reasonable assumption: for instance, in HVAEs near perfect data reconstructions are readily achieved. 

For variational autoencoders, the encoder and decoder are not necessarily deterministic and involve resampling. 
To encapsulate this, we will work with the data as a measure and have $F_{\theta, j}$ and $B_{\theta, j}$ as the corresponding kernels imposed by $f_{j,\theta}$ and $b_{j,\theta}$, respectively. 

With all of these considerations in mind, for the purposes of our framework we provide a definition of an idealised U-Net which is an approximate encapsulation of all models using a U-Net architecture.

\begin{definition}(Idealised U-Net for generative modelling) \\
For each $j \in \{0, \dots ,J \}$, let $F_{j, \theta}, B_{j, \theta} : \Db(V_{-j}) \mapsto \Db(V_{-j})$ such that $B_{j, \theta}F_{j, \theta} \equiv I_{V_{-j}}$. 
A U-Net with (average pooling) dimension reduction $P_{-j+1}$ and dimension embedding $E_{-j}$ is the operator $\mathbf{U} : \Db(V_{-J}) \mapsto \Db(V_{-J})$ given by 
\begin{align}
    \mathbf{U} \coloneqq B_{J, \theta} E_{-J} \circ \cdots \circ  B_{1, \theta} E_{-1} \circ  P_{0} F_{1, \theta} \circ \cdots  \circ P_{-J+1}  F_{J, \theta},
    &&
    B_{j}F_{j} \equiv I.
\end{align}
\end{definition}

\begin{remark}
The condition $B_{j, \theta}F_{j, \theta} \equiv I_{V_{-j}}$ in our idealised U-Net (for unconditional sampling here) is either imposed directly (reversible flow based model), or approximated via skip connections. 
For instance, in our HVAE case, we have both a U-Net without skip connections (the $p$ distribution) and a U-Net with skip connections (the $q$ distribution).
The U-Net related to the $q$ distribution learns how to reconstruct our data from the reconstruction term in the ELBO cost function. 
The U-Net related to the $p$ distribution learns to mimic the $q$ distribution via the KL term in the ELBO of the HVAE, whose decoder is trained to invert its encoder --- $B_{j, \phi}F_{j, \phi} \equiv I_{V_{-j}}$ --- but the $p$ U-Net lacks skip connections. 
Thus, in the HVAE context, we are analysing U-Nets which must simultaneously reconstruct our data and lose their reliance on their skip connections due to the condition that the $q$ U-Net must be approximately equal to the $p$ U-Net.
\\

\end{remark}

\subsection{Hierarchical VAEs}
\label{app:Hierarchical VAEs}

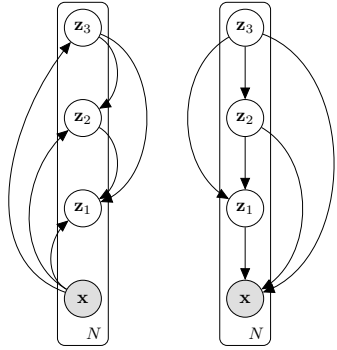
\begin{figure}[h!]
  \centering
    \centering
    \scalebox{.7}{\beginpgfgraphicnamed{recognition_model}
\begin{tikzpicture}

  \node[obs] (x) {$\mathbf{x}$}; %
  \node[latent, above=1 of x] (z1) {$\mathbf{z}_1$}; %
  \node[latent, above=1 of z1] (z2) {$\mathbf{z}_2$}; %
  \node[latent, above=1 of z2] (z3) {$\mathbf{z}_3$}; %

  \draw [->] (x) to [out=150,in=220] (z1);
  \draw [->] (x) to [out=150,in=220] (z2);  %
  \draw [->] (x) to [out=160,in=230] (z3);  %
  
  \draw [->] (z3) to [out=330,in=30] (z2);
  \draw [->] (z2) to [out=330,in=30] (z1);
  \draw [->] (z3) to [out=340,in=20] (z1);

  \plate[] {} {%
    (z1)(z2)(z3)(x)
  } {$N$} ;

\end{tikzpicture}
\endpgfgraphicnamed}
    \scalebox{.7}{\beginpgfgraphicnamed{recognition_model}
\begin{tikzpicture}

  \node[obs] (x) {$\mathbf{x}$}; %
  \node[latent, above=1 of x] (z1) {$\mathbf{z}_1$}; %
  \node[latent, above=1 of z1] (z2) {$\mathbf{z}_2$}; %
  \node[latent, above=1 of z2] (z3) {$\mathbf{z}_3$}; %

  \edge[] {z3} {z2}; %
  \edge[] {z2} {z1}; %
  \edge[] {z1} {x}; %
  
  \draw [->] (z3) to [out=210,in=150] (z1);
  \draw [->] (z3) to [out=340,in=20] (x);
  \draw [->] (z2) to [out=330,in=30] (x);

  \plate[] {} {%
     (z1)(z2)(z3)(x)
   } {$N$} ;

\end{tikzpicture}
\endpgfgraphicnamed}
  \caption{
  Conditioning structure in state-of-the-art HVAE models (VDVAE\cite{Child2020VeryImages} / NVAE\cite{Vahdat2020NVAE:Autoencoder}) with $L=3$. %
  [Left] Amortised variational posterior $q_{\theta}(\vv{z} \mid \v{x})$. [Right] Generative model $p_{\phi}(\v{x},\vv{z})$.
  }
  \label{fig:graph_model}
\end{figure}

A \textit{hierarchical Variational Autoencoder (HVAE)}~\footnote{We closely follow the introduction of hierarchical VAEs in \cite[\S 2.2]{Child2020VeryImages}.} is a VAE \cite{kingma2013auto} where latent variables are separated into $L$ groups $\vv{z} = (\v{z}_1, \v{z}_2, \dots, \v{z}_L)$ which conditionally depend on each other.  %
$L$ is often referred to as stochastic depth.
For convenience, we refer to the observed variable $\v{x}$ as $\v{z}_0$, so $\v{x} \equiv \v{z}_0$.
In HVAEs, latent variables typically follow a `bow tie', U-Net \cite{ronneberger2015u} type architecture with an information bottleneck \cite{Tishby2000TheMethod}, so $\dim(\v{z}_{l+1}) \leq \dim(\v{z}_l)$ for all $l=0, \ldots, L-1$.
Latent variables live on multiple \textit{resolutions}, either decreasing steadily \cite{LVAE, Maale2019BIVA:Modeling} or step-wise every few stochastic layers \cite{Vahdat2020NVAE:Autoencoder,Child2020VeryImages} in dimension.
We consider this multi-resolution property an important characteristic of HVAEs. %
It distinguishes HVAEs from other deep generative models, in particular vanilla diffusion models where latent and data variables are of equal dimension \cite{DDPM}.  %

As in a plain VAE with only a single group of latent variables, an HVAE has a likelihood $p_\phi(\v{x}|\vv{z})$, a prior $p_\phi(\vv{z})$ and an approximate posterior $q_\theta(\vv{z}|\v{x})$. 
To train the HVAE, one optimises the ELBO w.r.t. parameters $\phi$ and $\theta$ via stochastic gradient descent using the reparametrization trick
\begin{equation}
    \log p(\mathcal{D}) \geq \ELBO(\mathcal{D};\theta,\phi) = \expect_{\v x \sim \mathcal{D}}   \underbrace{ [ \expect_{\vv{z} \sim q_\theta(\vv{z}|\v{x})} \left[ \log p_\phi(\v{x}|\vv{z}) \right]}_{\text{Reconstruction loss}} - \underbrace{\KL[q_\theta(\vv{z}|\v{x})||p_\phi(\vv{z})] ] }_{\text{Prior loss} } .  %
    \label{eq:hvae_elbo}
\end{equation}

Numerous conditioning structures of the latent variables in HVAEs exist, and we  review them in \S\ref{sec:Related_work}. 
In this work, we follow \cite{Child2020VeryImages,Vahdat2020NVAE:Autoencoder, Kingma2016ImprovedFlow}: the latent variables in the prior and approximate posterior are estimated in the same order, from $\v{z}_L$ to $\v{z}_1$, conditioning `on all previous latent variables', i.e.

\noindent\begin{minipage}{.5\linewidth}
\begin{align}
    p_\phi(\vv{z}) = p_\phi(\v{z}_L) \prod_{l=1}^{L-1} p_\phi(\v{z}_l|\v{z}_{>l}) \label{eq:p_lvae_factorisation} %
\end{align} 
\end{minipage}
\begin{minipage}{.5\linewidth}
\begin{align}
    q_\theta(\vv{z}|\v{x}) = q_\theta(\v{z}_L|\v{x}) \prod_{l=1}^{L-1} q_\theta(\v{z}_l|\v{z}_{>l}, \v{x}) \label{eq:q_lvae_factorisation} %
\end{align}
\end{minipage}
We visualise the graphical model of this HVAE in Fig.~\ref{fig:graph_model}.
Recent HVAEs \cite{Child2020VeryImages, Vahdat2020NVAE:Autoencoder} capture this dependence on all previous latent variables $\v{z}_{>l}$ in their residual state as shown in \S\ref{sec:Example: HVAEs are Sum Representation Diffusion Discretisations}, imposing this conditional structure.
This implies a 1st-order Markov chain conditional on the previous residual state, not the previous $\v z_l$.
Such 1st-order Markov processes have shown great success empirically, such as in LSTMs \cite{hochreiter1997long}.
Further, note that in all previous work on HVAEs, the neural networks estimating the inference and generative distributions of the $l$-th stochastic layer are \textit{not} sharing parameters with those estimating other stochastic layers.  %

Intuitively, HVAEs' conditional structure together with a U-Net architecture imposes an inductive bias on the model to learn a \textit{hierarchy} of latent variables where each level corresponds to a different degree of abstraction.
In this work, we characterise this intuition via the regularisation property of U-Nets in \S\ref{sec:The regularisation property imposed by U-Net architectures with average pooling}.  

The distributions over the latent variables in both the inference and generative model are Gaussian with mean $\v \mu$ and a diagonal covariance matrix $\bm{\Sigma}$, i.e. for all $l=1, \dots, L$,
\begin{align}
    q_\theta(\v{z}_l|\v{z}_{>l}, \v{x}) &\sim \mathcal{N}(\v \mu_{l, \theta}, \bm{\Sigma}_{l, \theta}), \\
    p_\phi(\v{z}_l|\v{z}_{>l}) &\sim \mathcal{N}(\v \mu_{l, \theta}, \bm{\Sigma}_{l, \theta}), 
\end{align}
where mean and variances are estimated by neural networks with parameters $\phi$ and $\theta$ corresponding to stochastic layer $l$.
Note that $p_\phi(\v{z}_L|\v{z}_{>l}) = p_\phi(\v{z}_L)$, where the top-down block estimating $p_\phi(\v{z}_L)$ receives the zero-vector as input, and $q_\theta(\v{z}_L|\v{z}_{>L}, \v{x}) = q_\theta(\v{z}_L | \v{x})$, meaning that we infer without conditioning on other latent groups at the $L$-th step.
Further, VDVAE chooses $p_\phi(\v{x}|\vv{z})$ to be a discretized Mixture-of-Logistics likelihood.  %

\subsection{Sampling of Time Steps in HVAEs}
\textbf{Monte Carlo sampling of time steps in ELBO of HVAEs. }

We here provide one additional theoretical result.
We show that the ELBO of an HVAE can be written as an expected value over uniformly distributed time steps.

Previous work \cite{Kingma2021VariationalModels} \cite{ho2020denoising} (Eq. (13), respectively) showed that the diffusion loss term $\mathcal{L}_T(\v{x})$ in the ELBO of discrete-time diffusion models can be written as  %
\begin{align}
    \mathcal{L}_{T}(\v{x})=\frac{T}{2} \mathbb{E}_{\boldsymbol{\epsilon} \sim \mathcal{N}(0, \v{I}), i \sim U\{1, T\}}\left[(\operatorname{SNR}(s)-\operatorname{SNR}(t))\left\|\v{x}-\hat{\v{x}}_{\theta}\left(\v{z}_{t} ; t\right)\right\|_{2}^{2}\right]
    \label{eq:mc_diff}
\end{align}
which allows maximizing the variational lower-bound via a Monte Carlo estimator of Eq.~\ref{eq:mc_diff}, sampling time steps.

Inspired by this result for diffusion models, we provide a similar form of the ELBO for an HVAE with factorisation as in Eqs.~\eqref{eq:p_lvae_factorisation}-\eqref{eq:q_lvae_factorisation} (and the graphical model in Fig.~\ref{fig:graph_model}).
An HVAE's ELBO can be written as
\begin{align*}
    \log p(\v{x}) 
    &\geq \expect_{\vv{z} \sim q(\vv{z}|\v{x})} \left[ \log p(\v{x}|\vv{z}) \right] - L \expect_{l \sim \text{Unif}(1, L)} \left[ \expect_{\vv z \sim q(\vv z | \v x)} \log \frac{q(\v z_l | \v z_{>l}, \v x)}{p(\v z_l | \v z_{>l})}  \right].
\end{align*}
\begin{proof}
\begin{align*}
    \log p(\v{x}) 
    &\geq \expect_{\vv{z} \sim q(\vv{z}|\v{x})} \left[ \log p(\v{x}|\vv{z}) \right] - \KL\left[ q(\vv{z}|\v{x}) || p(\vv{z}) \right] \\
    &= \expect_{\vv{z} \sim q(\vv{z}|\v{x})} \left[ \log p(\v{x}|\vv{z}) \right] - \int \intd \vv z q(\vv{z}|\v{x}) \log \left[ \frac{\prod_{l=1}^{L} q(\v z_l | \v z_{>l}, \v x)}{\prod_{l=1}^{L} p(\v z_l | \v z_{>l})} \right] \\
    &= \expect_{\vv{z} \sim q(\vv{z}|\v{x})} \left[ \log p(\v{x}|\vv{z}) \right] - \sum_{l=1}^{L} \int \intd \vv z q(\vv{z}|\v{x}) \log \left[ \frac{q(\v z_l | \v z_{>l}, \v x)}{p(\v z_l | \v z_{>l})} \right] \\
    &= \expect_{\vv{z} \sim q(\vv{z}|\v{x})} \left[ \log p(\v{x}|\vv{z}) \right] - L \expect_{l \sim \text{Unif}(1, L)} \left[ \expect_{\vv z \sim q(\vv z | \v x)} \log \frac{q(\v z_l | \v z_{>l}, \v x)}{p(\v z_l | \v z_{>l})}  \right].
\end{align*}
\end{proof}

This allows reducing the computational and memory costs of the KL-terms in the loss and depends on how many Monte Carlo samples are drawn. 
However, in contrast to diffusion models, all intermediate stochastic layers (up to the top-most and bottom-most layer chosen when sampling time steps in the recognition and generative model, respectively) still need to be computed as each latent variable's distribution depends on all previous ones.

\newpage
\section{Code, computational resources, existing assets used}
\label{app:Code, computational resources, existing assets used}

\textbf{Code. }
We provide our PyTorch code base at \href{https://github.com/FabianFalck/unet-vdvae}{https://github.com/FabianFalck/unet-vdvae}. 
Our implementation is based on, modifies and extends the \href{https://github.com/openai/vdvae}{official implementation of VDVAE} \cite{Child2020VeryImages}.
Below, we highlight key contributions: 
\begin{itemize}
\item We implemented weight-sharing of individual ResNet blocks for a certain number of repetitions.
\item We implemented the datasets and the preprocessing of MNIST and CelebA, which were previously not used with VDVAE.
\item We implemented the option of synchronous and asynchronous processing in time (see Appendix \ref{app:Synchronous vs. asynchronous processing in time}).
\item We implemented Fourier features with hyperparameters choosing their frequencies following VDM \cite{Kingma2021VariationalModels}. 
One can concatenate them at three different locations as options.
\item We simplified the multi-GPU implementation.
\item We implemented an option to convert the VDVAE cell into a non-residual cell (see Appendix \ref{app:On the importance of a stochastic differential equation structure in HVAEs}).
\item We implemented logging of various metrics and plots with weight\&biases.
\item We implemented gradient checkpointing \cite{chen2016training} as an option in the decoder of VDVAE where the bulk of the computation occurs. 
We provide two implementations of gradient checkpointing, one based on the official PyTorch implementation which is unfortunately slow when using multiple GPUs, and a prototype for a custom implementation based on \href{https://github.com/csrhddlam/pytorch-checkpoint}{https://github.com/csrhddlam/pytorch-checkpoint}.
\end{itemize}

The \texttt{README.md} contains instructions on installation, downloading the required datasets, the setup of weights\&biases, and how to reproduce our main results.

\textbf{Computational resources. }
For the majority of time during this project, we used two compute clusters: 
The first cluster is a Microsoft Azure server with two Nvidia Tesla K80 graphic cards with 11GB of GPU memory each, which we had exclusive access to.
The second cluster is an internal cluster with 12 Nvidia GeForce GTX 1080 graphic cards and 10GB of GPU memory each, shared with a large number of users.
In the late stages of the project, in particular to perform runs on ImageNet32, ImageNet64 and CelebA, we used a large-scale compute cluster with A100 graphic cards with 40GB of GPU memory each.
We refer to the acknowledgements section for further details.

In the following, we provide a rough estimate of the total compute required to reproduce our main experiments. 
Compute time until convergence scales with the depth of the HVAEs. 
For the shallower HVAEs in our small-scale experiments in \S \ref{app:add_exp_details_results_``More from less'': Parameter efficiency in HVAEs}, training times range from several days to a week. 
For our larger-scale experiments on MNIST and CIFAR10, training times range between 1 to 3 weeks.
For our deepest runs on ImageNet32 and CelebA, training times range between 2.5 to 4 weeks.

For orientation, in Table \ref{tab:depth_time}, we provide an estimate of the training times of our large-scale runs in Table \ref{tab:sota_quant_comp}.
We note that these runs have been computed on different hardware, i.e. the training times are only to some degree comparable, yet give an indication.

\begin{table}[h!]
\caption{
A large-scale study of parameter efficiency in HVAEs.
For all our runs in Table \ref{tab:sota_quant_comp}, we report their stochastic depth and estimated training time.
}
\centering
\begin{tabular}{cccc} \toprule  
 & Method &  Depth & Training time \\ \midrule   %
\textbf{MNIST}  ($28 \times 28$) \\
& \texttt{WS-VDVAE} (ours) & 57 & $\approx 5$ days \\   %
& \texttt{VDVAE$^*$} (ours) & 43 & $\approx 5$ days  \\   %
\midrule
\textbf{CIFAR10} ($32 \times 32$) \\
& \texttt{WS-VDVAE}  (ours) & 268 & $\approx 18$ days   \\  %
& \texttt{WS-VDVAE}  (ours) & 105 & $\approx 13$ days \\  %
& \texttt{VDVAE$^*$} (ours) & 43 & $\approx 9$ days \\  %
\midrule
\textbf{ImageNet} ($32 \times 32$) \\
& \texttt{WS-VDVAE}  (ours) & 169 & $\approx 20$ days  \\  %
& \texttt{WS-VDVAE}  (ours) & 235 & $\approx 24$ days  \\  %
& \texttt{VDVAE$^*$} (ours) & 78 & $\approx 16$ days \\  %
\midrule
\textbf{CelebA} ($64 \times 64$) \\ 
& \texttt{WS-VDVAE}  (ours) & 125 & $\approx 27$ days   \\  %
 & \texttt{VDVAE$^*$} (ours) & 75 & $\approx 21$ days    \\  %
\bottomrule 
\end{tabular}
\label{tab:depth_time}
\end{table}

\textbf{Existing assets used. }
In the experiments, our work directly builds on top of the \href{https://github.com/openai/vdvae}{official implementation of VDVAE} \cite{Child2020VeryImages} (MIT License).
We use the datasets reported in Appendix \ref{app:Datasets}.
In our implementation, we make use of the following existing assets and list them together with their licenses: 
PyTorch \cite{pytorch}, highlighting the torchvision package for image benchmark datasets, and the gradient checkpointing implementation (custom license), 
Numpy \cite{harris2020array} (BSD 3-Clause License)
Weights\&Biases~\cite{wandb} (MIT License), 
Apex \cite{apex} (BSD 3-Clause ``New'' or ``Revised'' License), 
Pickle~\cite{pickle} (license not available), 
Matplotlib \cite{matplotlib} (PSF License), 
ImageIO \cite{imageio} (BSD 2-Clause ``Simplified'' License), 
MPI4Py \cite{dalcin2021mpi4py} (BSD 2-Clause ``Simplified'' License), 
Scikit-learn~\cite{scikit-learn} (BSD 3-Clause License), 
and Pillow \cite{umesh2012image} (custom license).

\section{Datasets}
\label{app:Datasets}

In our experiments, we make use of the following datasets: 
MNIST \cite{lecun2010mnist}, CIFAR10 \cite{krizhevsky2009learning}, ImageNet32 \cite{deng2009imagenet,chrabaszcz2017downsampled}, ImageNet64 \cite{deng2009imagenet,chrabaszcz2017downsampled}, and CelebA \cite{liu2015faceattributes}. 
We briefly discuss these datasets, focussing on their preprocessing, data splits, data consent and commenting on potential personally identifiable information or offensive content in the data.
We refer to the training set as images used during optimisation, the validation set as images used to guide training (e.g. to compute evaluation metrics during training) but not used for optimisation directly, and the test set as images not looked at during training and only to compute performance of completed runs.
For all datasets, we fix the training-validation-test split over different runs, and we scale images to be approximately centred and having a standard deviation of one based on statistics computed on the respective training set.  
If not stated otherwise, we use a modified version of the implementation of these datasets in \cite{Child2020VeryImages}.

\textbf{MNIST. }
The MNIST dataset \cite{lecun2010mnist} contains gray-scale handwritten images of 10 digit classes (`0' to `9') with resolution $28 \times 28$.
It contains $60,000$ training and $10,000$ test images, respectively.
From the training images, we use $55,000$ images as the training set and $5000$ images as the validation set. 
We use all $10,000$ test images as the testing set.
We build on top of the implementation provided in NVAE \cite{Vahdat2020NVAE:Autoencoder} (\href{https://github.com/NVlabs/NVAE/blob/master/datasets.py}{\url{https://github.com/NVlabs/NVAE/blob/master/datasets.py}}), which itself uses torchvision \cite{pytorch}, and dynamically binarize the images, meaning that pixel values are binary, as drawn from a Bernoulli distribution with the probabilities given by the scaled gray-scale values in $[0,1]$.
Furthermore, we pad each image with zeros so to obtain the resolution $32 \times 32$.

The dataset is highly standardised and cropped to individual digits so that offensive content or personally identifiable information can be excluded.
As the original NIST database from which MNIST was curated is no longer available, we cannot comment on whether consent was obtained from the subjects writing and providing these digits \cite{falck2021multi}.

\textbf{CIFAR10. }
The CIFAR10 dataset \cite{krizhevsky2009learning} contains coloured images from 10 classes ('airplane', 'automobile', 'bird', 'cat', 'deer', 'dog', 'frog', 'horse', 'ship', 'truck') with resolution $32 \times 32$.
It contains $50,000$ training and $10,000$ test images, respectively.
We split the training images into $45,000$ images in the training set and $5000$ images in the validation set, and use all $10,000$ test images as the test set.

CIFAR10 was constructed from the so-called 80 million tiny images dataset by Alex Krizhevsky, Vinod Nair, and Geoffrey Hinton \cite{80milliontiny}.
On the official website of the 80 million tiny images dataset, the authors state that this larger dataset was officially withdrawn by the authors on June 29th, 2020 due to offensive images being identified in it \cite{prabhu2020large}.
The authors of the 80 million tiny images dataset do not comment on whether CIFAR10, which is a subset of this dataset, likewise contains these offensive images or is unaffected.
\cite{krizhevsky2009learning} states that the images in the 80 million tiny images dataset were retrieved by searching the web for specific nouns. 
The authors provide no information to which degree consent was obtained from the people who own these images.

\textbf{ImageNet32. }
The ImageNet32 dataset, a downsampled version of the ImageNet database \cite{deng2009imagenet,chrabaszcz2017downsampled}, contains $1,281,167$ training and $50,000$ test images from $10,000$ classes with resolution $32 \times 32$.
From the training images, $5,000$ images as the validation set and the remaining $1,276,167$ as the training set, and further use all $50,000$ test images as the test set.

ImageNet is a human curated collection of images downloaded from the web via search engines.
While ImageNet used Amazon Mechanical Turk to lable the images, we were unable to find information on processes which ensured no personally identifiable or offensive content was contained in the images, which is somewhat likely given the ``in-the-wild'' nature of the dataset.
The ImageNet website states that the copyright of the images does not belong to authors of ImageNet.

\textbf{ImageNet64. }
The ImageNet64 dataset, a second downsampled version of the ImageNet database \cite{deng2009imagenet,chrabaszcz2017downsampled}, likewise contains $1,281,167$ training and $50,000$ validation images with resolution $64 \times 64$.
We use the same data splits as for ImageNet32.
Refer to the above paragraph on ImageNet32 for discussion of personally identifiable information, offensive content and consent.

\textbf{CelebA. }
The CelebA dataset \cite{liu2015faceattributes} contains $162,770$ training, $19,867$ validation and $19,962$ test images with resolution $64 \times 64$ which we directly use as our training, validation and test set, respectively.
Our implementation is a modified version of the one provided in NVAE \cite{Vahdat2020NVAE:Autoencoder} (\href{https://github.com/NVlabs/NVAE/blob/master/datasets.py}{\url{https://github.com/NVlabs/NVAE/blob/master/datasets.py}}).

CelebA images are ``obtained from the Internet''.
The authors state that these images are not the property of the authors of associated institutions \cite{celebawebsite}.
As this dataset shows the faces of humans, these images are personally identifiable.
We were unable to identify a process by which consent for using these images was obtained, or how potential offensive content was prevented.

\section{Potential negative societal impacts}
\label{app:Potential negative societal impacts}

Our work provides mainly theoretical and methodological contributions to U-Nets and HVAEs, and we hence see no direct negative societal impacts.
Since U-Nets are widely used in applications, our theoretical results and any future work derived from them may downstream improve such applications, and thus also enhance their performance in malicious uses.
In particular, U-Nets are widely used in generative modelling, and here, our work may have an effect on the quality of `deep fakes', fake datasets or other unethical uses of generative modelling.
For HVAEs, our work may inspire novel models which may lead to improved performance and stability of these models, also when used in applications with negative societal impact.

\section{Model and training details}
\label{app:Model and training details}

\textbf{On the stability of training runs. }
\label{app:Stability}
VDVAE uses several techniques to improve the training stability of HVAEs: 
First, gradient clipping \cite{mikolov2012statistical} is used, which reduces the effective learning rate of a mini-batch if the gradient norm surpasses a specific threshold. 
Second, gradient updates are skipped entirely when gradient norms surpass a second threshold, typically chosen higher than the one for gradient clipping.
Third, gradient updates are also skipped if the gradient update would cause an overflow, resulting in NaN values.

In spite of the above techniques to avoid deterioration of training in very deep HVAEs, particularly when using a lot of weight-sharing, we experienced stability problems during late stages of training. 
These were particularly an issue on CIFAR10 in the late stages of training (on average roughly after 2 weeks of computation time), and often resulted in NaN values being introduced or posterior collapse.
We did not extensively explore ways to prevent these in order to do minimal changes compared to vanilla VDVAE. 
We believe that an appropriate choice of the learning rate (e.g. with a decreasing schedule in later iterations) in combination with other changes to the hyperparameters may greatly help with these issues, but principled fixes of, for instance, the instabilities identified in Theorem \ref{thm:discrete-VDVAE} are likewise important.

\textbf{Gradient checkpointing, and other alternatives to reduce GPU memory cost. }
A practical limitation of training deep HVAEs (with or without weight-shared layers) is their GPU memory cost:
Training deeper HVAEs means storing more intermediate activations in GPU memory during the forward pass, when memory consumption reaches its peak at the start of the backward pass. %
This limits the depth of the networks that can be trained on given GPU resources.
To address this issue, particularly when training models which may not even fit on used hardware, we provide a prototype of a custom\footnote{Our implementation deviates from the official PyTorch implementation of gradient checkpointing which is slow when using multiple GPUs, and is based on \href{https://github.com/csrhddlam/pytorch-checkpoint}{https://github.com/csrhddlam/pytorch-checkpoint}.} \textit{gradient checkpointing} implementation. 
Checkpointing occurs every few ResNet blocks which trades off compute for memory and can be used as an option.
In gradient checkpointing, activations are stored only at specific nodes (checkpoints) in the computation graph, saving GPU memory, and are otherwise recomputed on-demand, requiring one additional forward pass per mini-batch \cite{chen2016training}.
Training dynamics remain unaltered. 
We note that other techniques exist specifically targeted at residual networks: 
For example, \cite{huang2016deep} propose to stochastically drop out entire residual blocks at training time \footnote{
This technique is called ``stochastic depth'' as the active depth of the network varies at random. 
In this work, however, we go with our earlier definition of this term which refers to the number of stochastic layers in our network, and thus avoid using its name to prevent ambiguities.
}. 
This technique has two disadvantages: 
It changes training dynamics, and peak memory consumption varies between mini-batches, where particularly the latter is an inconvenient property for the practitioner as it may cause out-of-memory errors.

\section{Additional experimental details and results}
\label{app:Experimental details}

In this section, we provide additional experimental details and results.

\textbf{Hyperparameters and hyperparameter tuning. } 
In the following, we describe the hyperparameters chosen in our experiments.
As highlighted in the main text, we use the state-of-the-art hyperparameters of VDVAE \cite{Child2020VeryImages} wherever possible. 
This was possible for CIFAR10, ImageNet32 and ImageNet64. 
On MNIST and CelebA, VDVAE \cite{Child2020VeryImages} did not provide experimental results.
For MNIST, we took the hyperparameters of CIFAR10 as the basis and performed minimal hyperparameter tuning, mostly increasing the batch size and tuning the number and repetitions of residual blocks.
For CelebA, we used the hyperparameters of ImageNet64 with minimal hyperparameter tuning, focussing on the number and repetitions of the residual blocks.
For all datasets, the main hyperparameter we tuned was the number and repetitions (through weight-sharing) of residual blocks ceteris paribus, i.e. without searching over the space of other important hyperparameters.
As a consequence, it is likely that further hyperparameter tuning would improve performance as changing the number of repetitions changes (the architecture of) the model.

We provide three disjunct sets of hyperparameters: \textit{global} hyperparameters (Table \ref{tab:global_hyperparams}), which are applicable to all runs, \textit{data-specific} hyperparameters (Table \ref{tab:data_spec_hyperparams}), which are applicable to specific datasets, and \textit{run-specific} hyperparameters, which vary by run.
The run-specific hyperparameters will be provided in the respective subsections of \S \ref{app:Experimental details}, where applicable.

In the below tables, `factor of \# channels in conv. blocks' refers to the multiplicative factor of the number of channels in the bottleneck of a (residual) block used throughout VDVAE.
`\# channels of  $z_l$' refers to C in the shape [C, H, W] of the latent conditional distributions in the approximate posterior and prior, where height H and width W are determined by the resolution of latent $\B{z}_l$.
Likewise, `\# channels in residual state' refers to C in the shape [C, H, W] of the residual state flowing through the decoder of VDVAE.
`Decay rate $\gamma$ of evaluation model' refers to the multiplicative factor by which the latest model parameters are weighted during training to update the evaluation model.

\begin{table}[h!]
\caption{
Global hyperparameters.
}
\centering
\begin{tabular}{cccccc} 
\toprule
factor of \# channels in conv. blocks & 0.25  \\
Gradient skipping threshold & 3000 \\
Adam optimizer: Weight decay & 0.01 \\
Adam optimizer: $\beta_1$ & 0.9 \\
Adam optimizer: $\beta_2$ & 0.9 \\
\bottomrule 
\end{tabular}
\label{tab:global_hyperparams}
\end{table}

\begin{table}[h!]
\caption{
Data-specific hyperparameters.
}
\centering
\begin{tabular}{cccccc} 
\toprule
Dataset & MNIST & CIFAR10 & ImageNet32 & ImageNet64 & CelebA \\ \midrule
Learning rate & 0.0001 & 0.0002 & 0.00015 & 0.00015 & 0.00015 \\
\# iterations for learning rate warm-up & 100 & 100 & 100 & 100 & 100 \\
Batch size & 200 & 16 & 8 & 4 & 4 \\
Gradient clipping threshold & 200 & 200 & 200 & 220 & 220 \\
\midrule
\# channels of  $\B{z}_l$ & 8 & 16 & 16 & 16 & 16 \\
\# channels in residual state & 32  & 384 & 512 & 512  & 512 \\
Decay rate $\gamma$ of evaluation model & 0.9999 & 0.9999 & 0.999  & 0.999  & 0.999 \\
\bottomrule 
\end{tabular}
\label{tab:data_spec_hyperparams}
\end{table}

\subsection{``More from less'': Parameter efficiency in HVAEs}
\label{app:add_exp_details_results_``More from less'': Parameter efficiency in HVAEs}

In this experiment, we investigate the effect of repeating ResNet blocks in the bottom-up and top-down pass via weight-sharing. 
\texttt{rN} indicates that a ResNet block is repeated \texttt{N} times through weight-sharing where \texttt{r} is to be treated like an operator and \texttt{N} is a positive integer. 
In contrast, \texttt{xN}, already used in the official implementation of VDVAE, indicates \texttt{N} number of ResNet blocks without weight-sharing.

In Tables \ref{tab:small_scale_test_1} and \ref{tab:small_scale_test_2}, we provide the NLLs on the test set at convergence corresponding to the NLLs on the validation set during training which we reported in Fig. \ref{fig:small_scale_param_efficiency}.
In general, weight-sharing tends to improve NLL, and models with significantly less parameters reach or even surpass other models with more parameters. 
We refer to the main text for the intuition of this behavior.
In Table \ref{tab:small_scale_test_2} (CIFAR10), we find that the not weight-sharing runs have test NLLs noticeably deviating from the results on the validation set, yet the overall trend of more weight-sharing improving NLL tends to be observed. 
This is in line with our general observation that our HVAE models are particularly unstable on CIFAR10.

\begin{table}[h!]
\caption{
A small-scale study on parameter efficiency of HVAEs on \textit{MNIST}.
We compare models with one, two, three and four parameterised blocks per resolution ($\{x1, x2, x3, x4\}$) against models with a single parameterised block per resolution weight-shared $\{2,3,5,10,20\}$ times ($\{r2,r3,r5,r10,r20\}$).
We report NLL ($\downarrow$) measured on the test set, corresponding to the results on the validation set in Fig. \ref{fig:small_scale_param_efficiency}.
NLL performance increases with more weight-sharing repetitions and surpasses models without weight-sharing but with more parameters. 
}
\centering
\begin{tabular}{ccc} 
\toprule
Neural architecture &  \# Params & NLL ($\downarrow$)  \\    
\midrule
\texttt{r1}/\texttt{x1} & 107k & $\leq 86.87$ \\   %
\texttt{r2} & 107k & $\leq 85.25$ \\  %
\texttt{r3} & 107k & $\leq 84.92$ \\  %
\texttt{r5} & 107k & $\leq 83.92$ \\  %
\texttt{r10} & 107k & $\leq 82.67$ \\  %
\texttt{r20} & 107k & $\leq 81.84$ \\  %
\midrule
\texttt{x2} & 140k & $\leq 84.44$ \\  %
\texttt{x3} & 173k & $\leq 82.64$ \\  %
\texttt{x4} & 206k & $\leq 82.46$ \\  %
\bottomrule
\end{tabular}
\label{tab:small_scale_test_1}
\end{table}

\begin{table}[h!]
\caption{
A small-scale study on parameter efficiency of HVAEs on \textit{CIFAR10}.
We compare models with with one, two, three and four parameterised blocks per resolution ($\{x1, x2, x3, x4\}$) against models with a single parameterised block per resolution weight-shared $\{2,3,5,10,20\}$ times ($\{r2,r3,r5,r10,r20\}$).
We report NLL ($\downarrow$) measured on the test set, corresponding to the results on the validation set in Fig. \ref{fig:small_scale_param_efficiency}.
NLL performance tends to increase with more weight-sharing repetitions. 
However, in contrast to the validation set (see Fig. \ref{fig:small_scale_param_efficiency}) where this trend is evident, it is less so on the test set.%
}
\centering
\begin{tabular}{ccc} 
\toprule
Neural architecture &  \# Params & NLL ($\downarrow$)  \\    
\midrule
\texttt{r1}/\texttt{x1} & 8.7m & $\leq 4.17$ \\   %
\texttt{r2} & 8.7m & $\leq 4.93$ \\  %
\texttt{r3} & 8.7m & $\leq 4.78$ \\  %
\texttt{r5} & 8.7m & -- \\  %
\texttt{r10} & 8.7m & $\leq 4.32$ \\  %
\texttt{r20} & 8.7m & $\leq 3.54$ \\  %
\midrule
\texttt{x2} & 13.0m & $\leq 5.77$ \\  %
\texttt{x3} & 17.3m & $\leq 3.07$ \\  %
\texttt{x4} & 21.6m & $\leq 3.01$ \\  %
\bottomrule
\end{tabular}
\label{tab:small_scale_test_2}
\end{table}

In Table \ref{tab:sota_quant_comp_details}, we provide key run-specific hyperparameters for the large-scale runs corresponding to Table \ref{tab:sota_quant_comp} in the main text. 
Two points on the architecture of the encoder and the decoder are worth noting: 
First, note that the decoder typically features more parameters and a larger stochastic depth than the encoder. 
We here follow VDVAE which observed this distribution of the parameters to be beneficial.
Second, note that while we experienced a benefit of weight-sharing, there is a diminishing return of the number of times a specific cell is repeated. 
Hence, we typically repeat a single block for no more than 10-20 times, beyond which performance does not improve while computational cost increases linearly with the number of repetitions.
Exploring how to optimally exploit the benefit of weight-sharing in HVAEs would be an interesting aspect for future work.

\begin{table}[t]
\scriptsize
\caption{
A large-scale study of parameter efficiency in HVAEs.
We here provide key run-specific hyperparameters corresponding to the results reported in Table \ref{tab:sota_quant_comp} in the main text.
Note that the row order of our runs directly corresponds with Table \ref{tab:sota_quant_comp}.
$\delta$ refers to gradient clipping threshold. 
$\gamma$ refers to the gradient skipping threshold.
We use the same nomenclature for number of cells (\texttt{x}) and number of repetitions for one block (\texttt{r}) as before.
In addition, as in VDVAE's official code base, we use \texttt{d} to indicate average pooling, where the integer before \texttt{d} indicates the resolution on which we pool, and the integer after indicates the down-scaling factor.
Further, \texttt{m} indicates interpolating, where we up-scale from a source (integer after \texttt{m}) to a target resolution (integer before \texttt{m}).
}
\centering
\begin{tabular}{p{.4cm}p{1.2cm}p{.2cm}p{.2cm}p{.2cm}p{4.4cm}p{5cm}} \toprule  
Dataset & Method & Batch size & $\delta$ & $\gamma$ & Encoder Architecture & Decoder Architecture \\ \midrule   %
\multirow{2}{*}[0cm]{\rotatebox{90}{\parbox{1.1cm}{\tiny \textbf{MNIST} \newline $28\times28$}}}  %
& \texttt{WS-VDVAE} (ours) & 70 & - & 200 & \texttt{32r3,32r3,32r3,32r3,32r3,32d2, 16r3,16r3,16r3,16d2, 8x6,8d2,4x3,4d4,1x3} & \texttt{1x1,4m1,4x2,8m4, 8x5,16m8,16r3,16r3,16r3,16r3,16r3,32m16, 32r3,32r3,32r3,32r3,32r3,32r3,32r3, 32r3,32r3,32r3}   \\   %
& \texttt{VDVAE$^*$} (ours) & 70 & - & 200 & \texttt{32x11,32d2,16x6,16d2, 8x6,8d2,4x3,4d4,1x3} & \texttt{1x1,4m1,4x2,8m4,8x5,16m8,16x10,32m16,32x21} \\   %
\midrule
\multirow{3}{*}[0cm]{\rotatebox[origin=c]{90}{\parbox{2.25cm}{\tiny \textbf{CIFAR10} \newline $32 \times 32$}}} 
& \texttt{WS-VDVAE}  (ours) & 16 & 400 & 4000 & \texttt{32r12,32r12,32r12,32r12,32r12,32r12, 32d2,16r12,16r12,16r12,16r12,16d2, 8r12,8r12,8r12,8r12,8r12,8r12,8d2, 4r12,4r12,4r12,4d4,1r12,1r12,1r12} & \texttt{1r12,4m1,4r12,4r12,8m4, 8r12,8r12,8r12,8r12,8r12, 16m8,16r12,16r12,16r12,16r12,16r12, 32m16,32r12,32r12,32r12,32r12, 32r12,32r12,32r12,32r12,32r12}     \\  %
& \texttt{WS-VDVAE}  (ours) & 16 & 200 & 2500 & \texttt{32r3,32r3,32r3,32r3,32r3,32r3, 32r3,32r3,32r3,32r3,32r3,32d2, 16r3,16r3,16r3,16r3,16r3,16r3,16d2, 8x6,8d2,4x3,4d4,1x3} & \texttt{1x1,4m1,4x2,8m4,8x5, 16m8,16r3,16r3,16r3,16r3,16r3, 16r3,16r3,16r3,16r3,16r3,32m16, 32r3,32r3,32r3,32r3,32r3,32r3, 32r3,32r3,32r3,32r3,32r3,32r3, 32r3,32r3,32r3,32r3,32r3,32r3, 32r3,32r3,32r3}   \\  %
& \texttt{VDVAE$^*$} (ours) & 16 & 200 & 400 & \texttt{32x11,32d2,16x6,16d2, 8x6,8d2,4x3,4d4,1x3} & \texttt{1x1,4m1,4x2,8m4,8x5, 16m8,16x10,32m16,32x21}  \\  %
\midrule
\multirow{3}{*}[0cm]{\rotatebox{90}{\parbox{2cm}{\tiny  \textbf{ImageNet} \newline $32 \times 32$}}} 
& \texttt{WS-VDVAE}  (ours) & 8 & 200 & 5000 & \texttt{32r10,32r10,32r10,32r10,32d2, 16r10,16r10,16r10,16d2,8x8,8d2, 4x6,4d4,1x6} & \texttt{1x2,4m1,4x4,8m4,8x9,16m8, 16r10,16r10,16r10,16r10,16r10,32m16, 32r10,32r10,32r10,32r10,32r10,32r10, 32r10,32r10,32r10,32r10}    \\  %
& \texttt{WS-VDVAE}  (ours) & 8 & 200 & 5000 & \texttt{32r6,32r6,32r6,32r6,32r6, 32r6,32r6,32r6,32r6,32d2, 16r6,16r6,16r6,16r6,16r6,16d2, 8x8,8d2,4x6,4d4,1x6} & \texttt{1x2,4m1,4x4,8m4,8x9, 16m8,16r6,16r6,16r6, 16r6,16r6,16r6,16r6,16r6,16r6, 16r6,16r6,32m16,32r6,32r6, 32r6,32r6,32r6,32r6,32r6,32r6, 32r6,32r6,32r6,32r6,32r6,32r6,32r6,32r6, 32r6,32r6,32r6,32r6,32r6,32r6,32r6,32r6,32r6}  \\  %
& \texttt{VDVAE$^*$} (ours) & 8 & 200 & 300 & \texttt{32x15,32d2,16x9,16d2,8x8,8d2, 4x6,4d4,1x6} & \texttt{1x2,4m1,4x4,8m4,8x9,16m8,16x19,32m16,32x40} \\  %
\midrule
\multirow{2}{*}[0cm]{\rotatebox{90}{\parbox{1.5cm}{\tiny  \textbf{CelebA} \newline $64 \times 64$}}} 
& \texttt{WS-VDVAE}  (ours) & 4 & 220 & 3000 & \texttt{64r3,64r3,64r3,64r3,64r3,64r3,64r3, 64d2,32r3,32r3,32r3,32r3,32r3,32r3, 32r3,32r3,32r3,32r3,32r3,32d2, 16r3,16r3,16r3,16r3,16r3,16r3,16d2, 8r3,8r3,8r3,8d2,4r3,4r3,4r3,4d4, 1r3,1r3,1r3} & \texttt{1r3,1r3,4m1,4r3,4r3,4r3, 8m4,8r3,8r3,8r3,8r3, 16m8,16r3,16r3,16r3,16r3,16r3,16r3,16r3, 32m16,32r3,32r3,32r3,32r3,32r3,32r3, 32r3,32r3,32r3,32r3,32r3,32r3,32r3, 32r3,32r3,32r3,64m32,64r3,64r3,64r3,64r3, 64r3,64r3,64r3,64r3}  \\  %
 & \texttt{VDVAE$^*$} (ours) & 4 & 220 & 3000 & \texttt{64x11,64d2,32x20,32d2, 16x9,16d2,8x8,8d2,4x7,4d4,1x5} & \texttt{1x2,4m1,4x3,8m4,8x7,16m8, 16x15,32m16,32x31,64m32,64x12}   \\  %
\bottomrule 
\end{tabular}
\label{tab:sota_quant_comp_details}
\end{table}

\clearpage
\subsection{HVAEs secretly represent time and make use of it}
\label{app:add_exp_details_results_HVAEs secretly represent time and make use of it}

In this experiment, we measure the $L_2$ norm of the residual state at every ResNet block in both the forward (bottom-up/encoder) and backward (top-down/decoder) model.
Let $x_i$ be the output of ResNet block $i$ in the bottom-up model, and $y_i$ be the input of ResNet block $i$ in the top-down model for one batch.
In the following, augmenting Fig. \ref{fig:magnitude_increase_mnist} on MNIST in the main text,  we measure $\| x_i \|_2$ or $\| y_i \|_2$, respectively, over 10 batches. 
We also use this data to compute appropriate statistics (mean and standard deviation) which we plot.
We measure the state norm in the forward and backward pass for models trained on CIFAR10 and ImageNet32 in Figs. \ref{fig:magnitude_increase_cifar} and \ref{fig:magnitude_increase_imagenet32}, respectively. 
We note that the forward pass of the ImageNet32 has a slightly unorthodox, yet striking pattern in terms of state norm magnitude, presumably caused by an overparameterisation of the model.
In summary, these findings provide further evidence that the residual state norm of VDVAEs represents time.

\begin{figure}[h!]
    \centering
    \includegraphics[width=.48\linewidth]{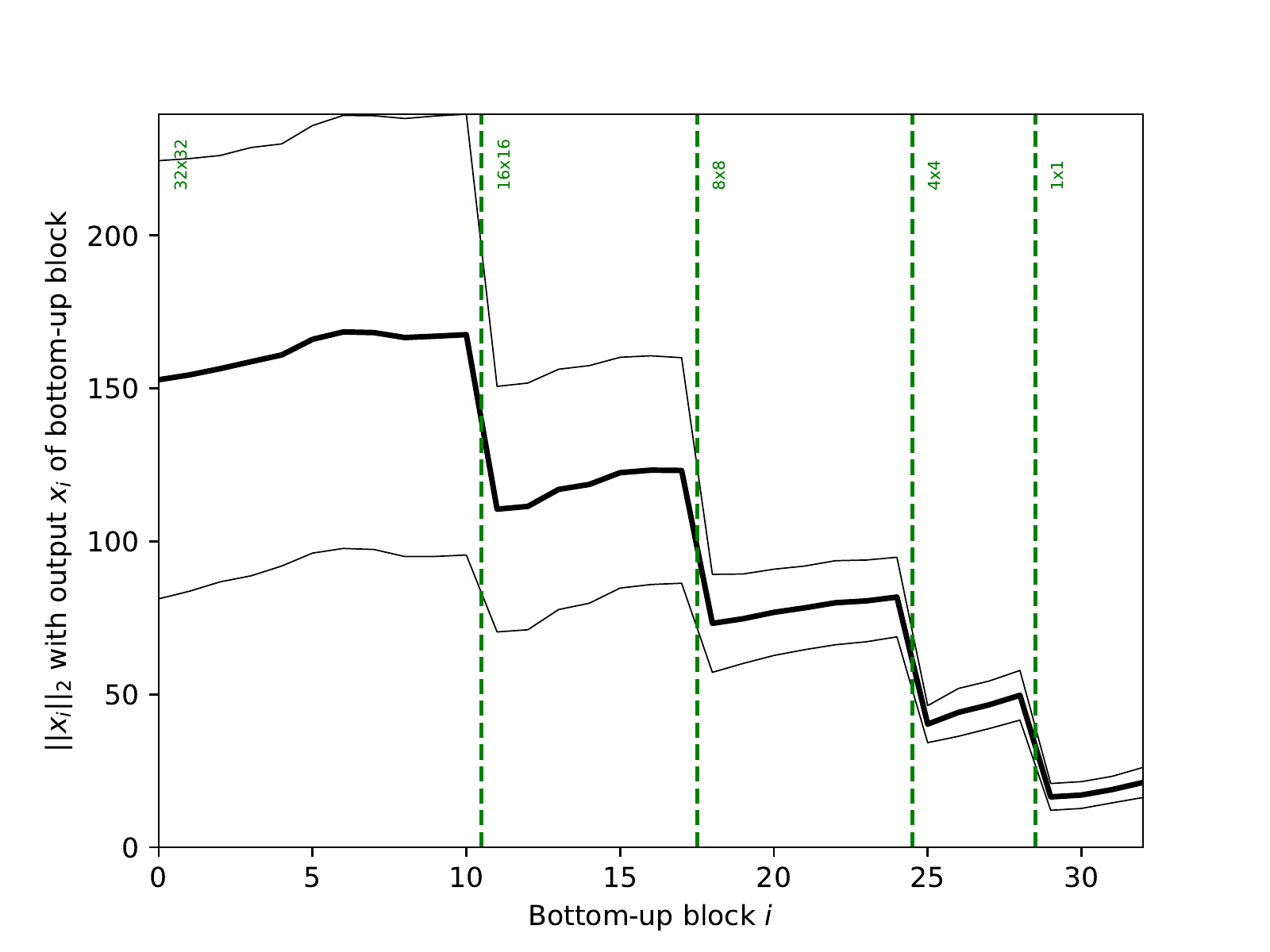}
    \includegraphics[width=.48\linewidth]{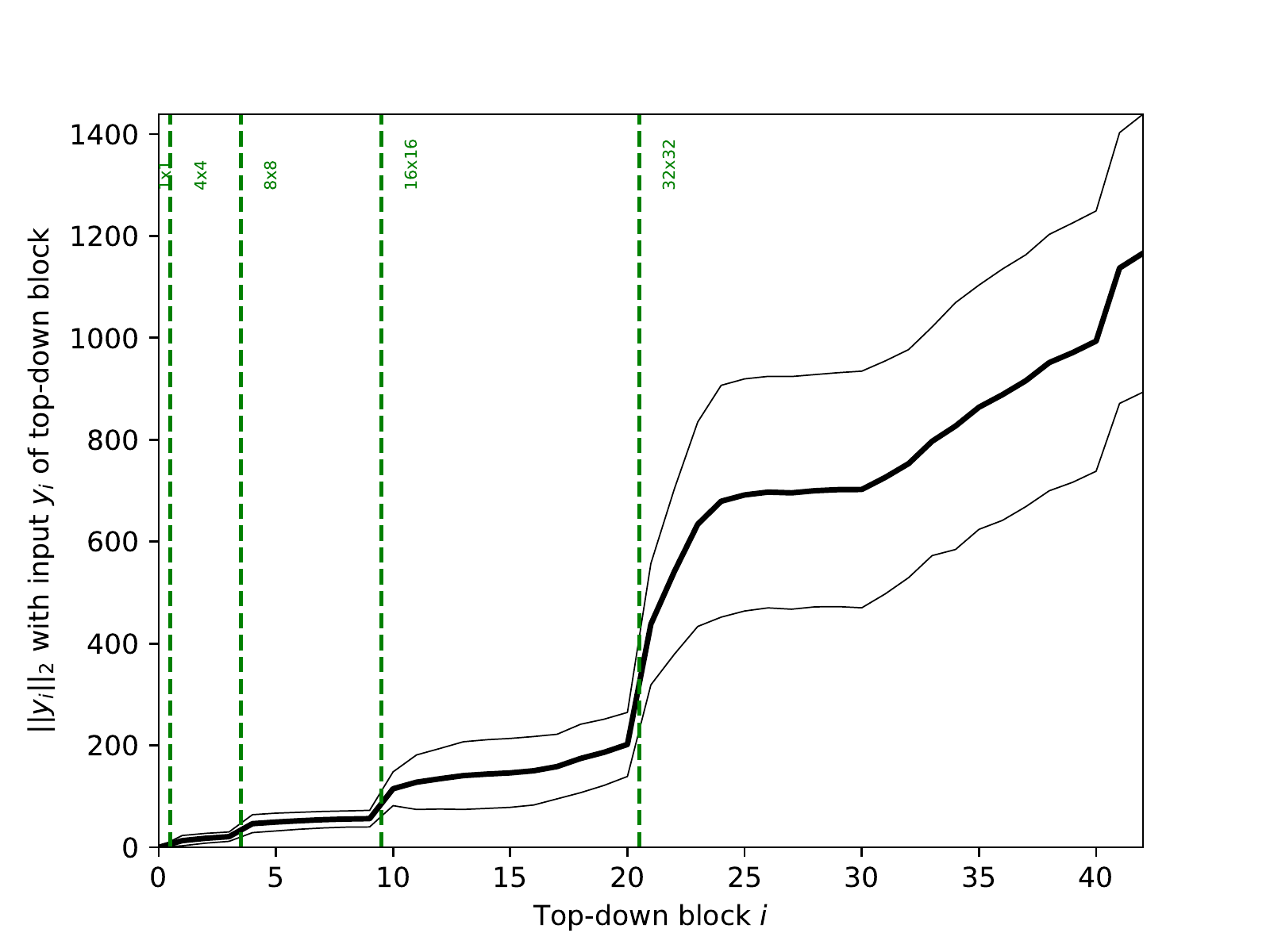}
    \caption{
    HVAEs are secretly representing time \textit{on CIFAR10}: 
    We measure the $L_2$-norm of the residual state at every residual block $i$ for the [Left] forward (bottom-up) pass, and [Right] the backward (top-down) pass, respectively, over 10 batches with 100 data points each.
    The thick line refers to the average and the thin, outer lines refer to $\pm 2$ standard deviations. 
    }
    \label{fig:magnitude_increase_cifar}
\end{figure}

\clearpage

\begin{figure}[h!]
    \centering
    \includegraphics[width=.48\linewidth]{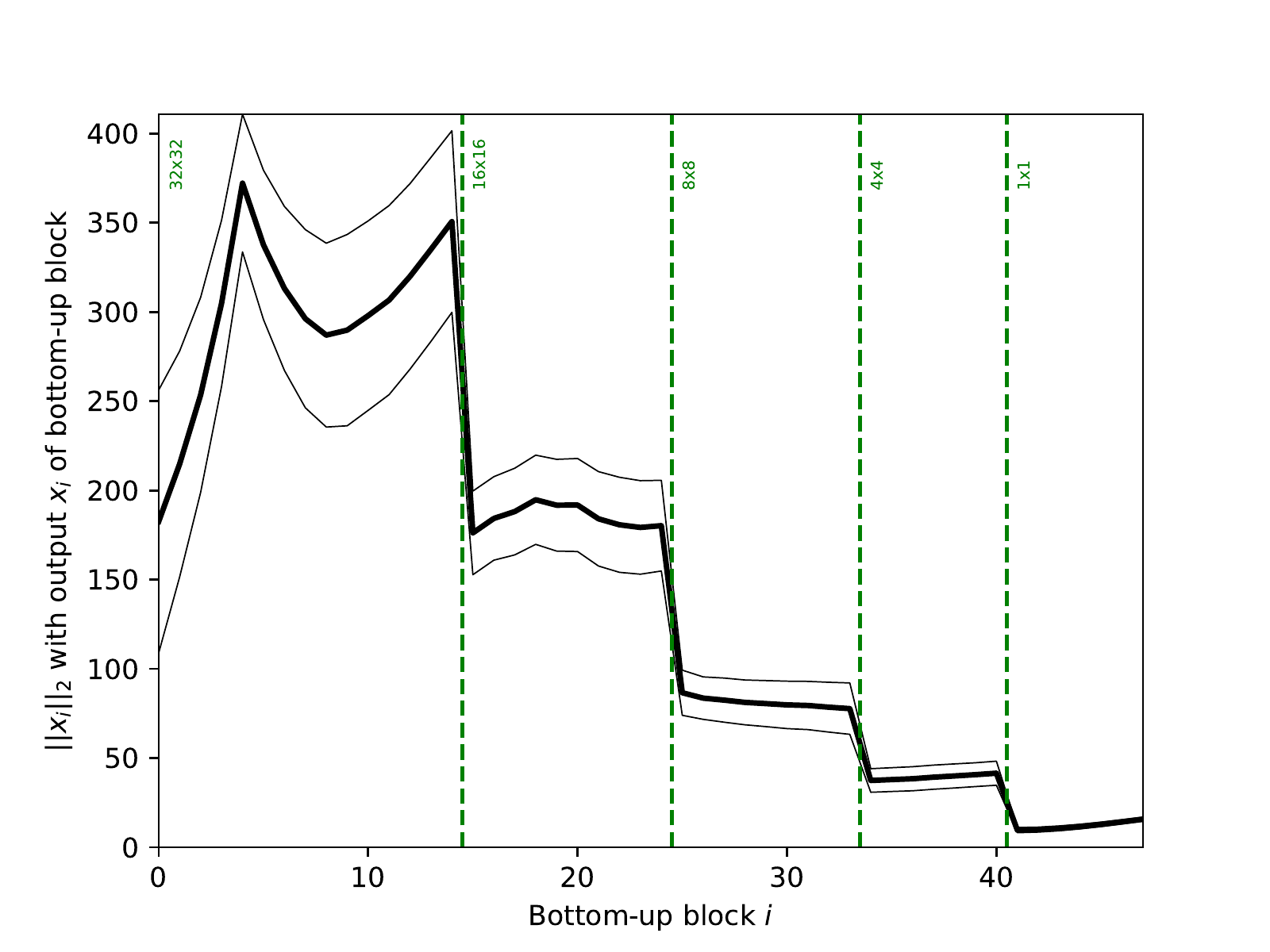}
    \includegraphics[width=.48\linewidth]{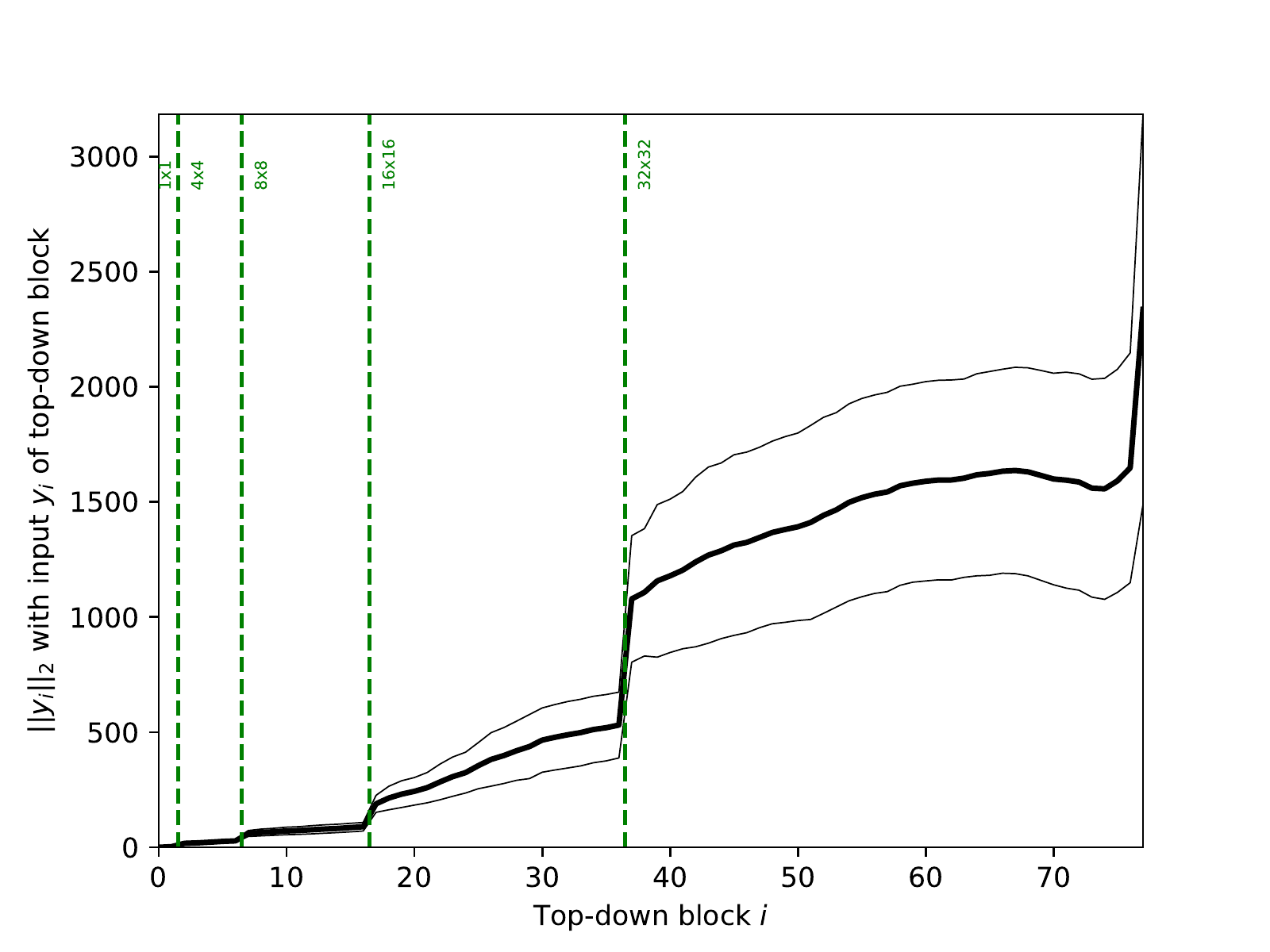}
    \caption{
    HVAEs are secretly representing time \textit{on ImageNet32}: 
    We measure the $L_2$-norm of the residual state at every residual block $i$ for the [Left] forward (bottom-up) pass, and [Right] the backward (top-down) pass, respectively, over 10 batches with 100 data points each.
    The thick line refers to the average and the thin, outer lines refer to $\pm 2$ standard deviations. 
    }
    \label{fig:magnitude_increase_imagenet32}
\end{figure}

When normalising the residual state in our experiments in Table \ref{tab:state_norm} (case ``normalised''), we do so at the same positions where we measure the state norm above. 
At the output of every forward ResNet block $x_i$ and the input of every backward ResNet block $y_i$, we assign
\begin{align*}
    x_i \leftarrow \frac{x_i}{\| x_i \|_2} && 
    y_i \leftarrow \frac{y_i}{\| y_i \|_2}
\end{align*}
for every mini-batch during training.
This results in a straight line in these plots for the ``normalised'' case. 
As the natural behavior of VDVAEs is---as we measured---to learn a non-constant norm, normalising the state norm has a deteriorating consequence, as we observe in Table \ref{tab:state_norm}.
In contrast, the regular, unnormalised runs (case ``non-normalised'') show well-performing results.

We further analysed the normalised state norm experiments in Table \ref{tab:state_norm}.
The normalised MNIST and CIFAR10 runs terminated early (indicated by \blackcross), more precisely after 18 hours and 4.5 days of training, respectively.
From the very start of the optimisation, the normalised models have poor training behavior.
To show this, in Fig. \ref{fig:normalised_vs_nonnormalised_nll}, we illustrate the NLL on the validation set during training for the three normalised runs as compared to regular, non-normalised training.
Validation ELBO only improves for a short time, after which the normalised runs deteriorate, showing no further improvement or even a worse NLL.

\begin{figure}[h!]
    \centering
    \includegraphics[width=.32\linewidth]{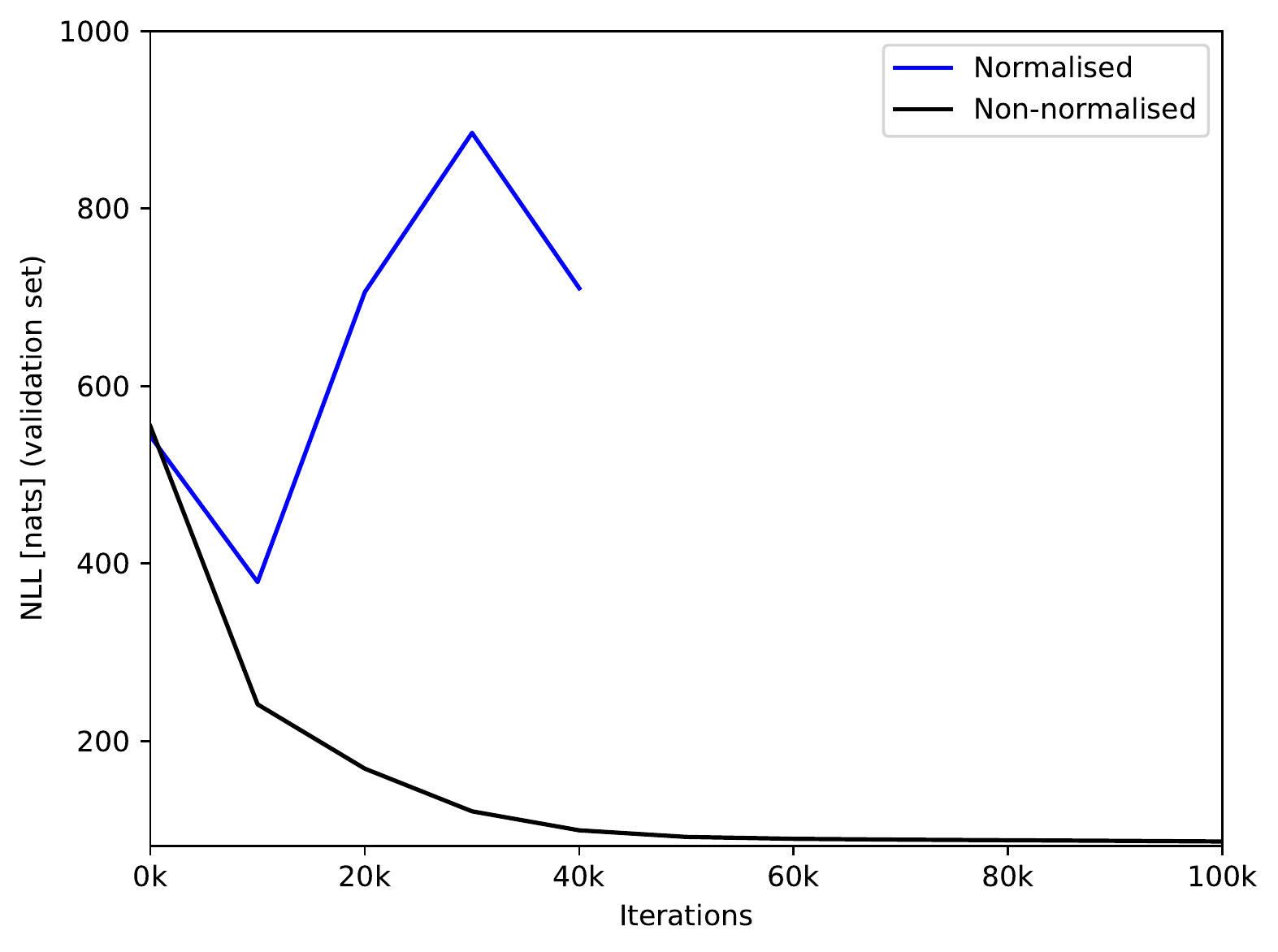}
    \includegraphics[width=.32\linewidth]{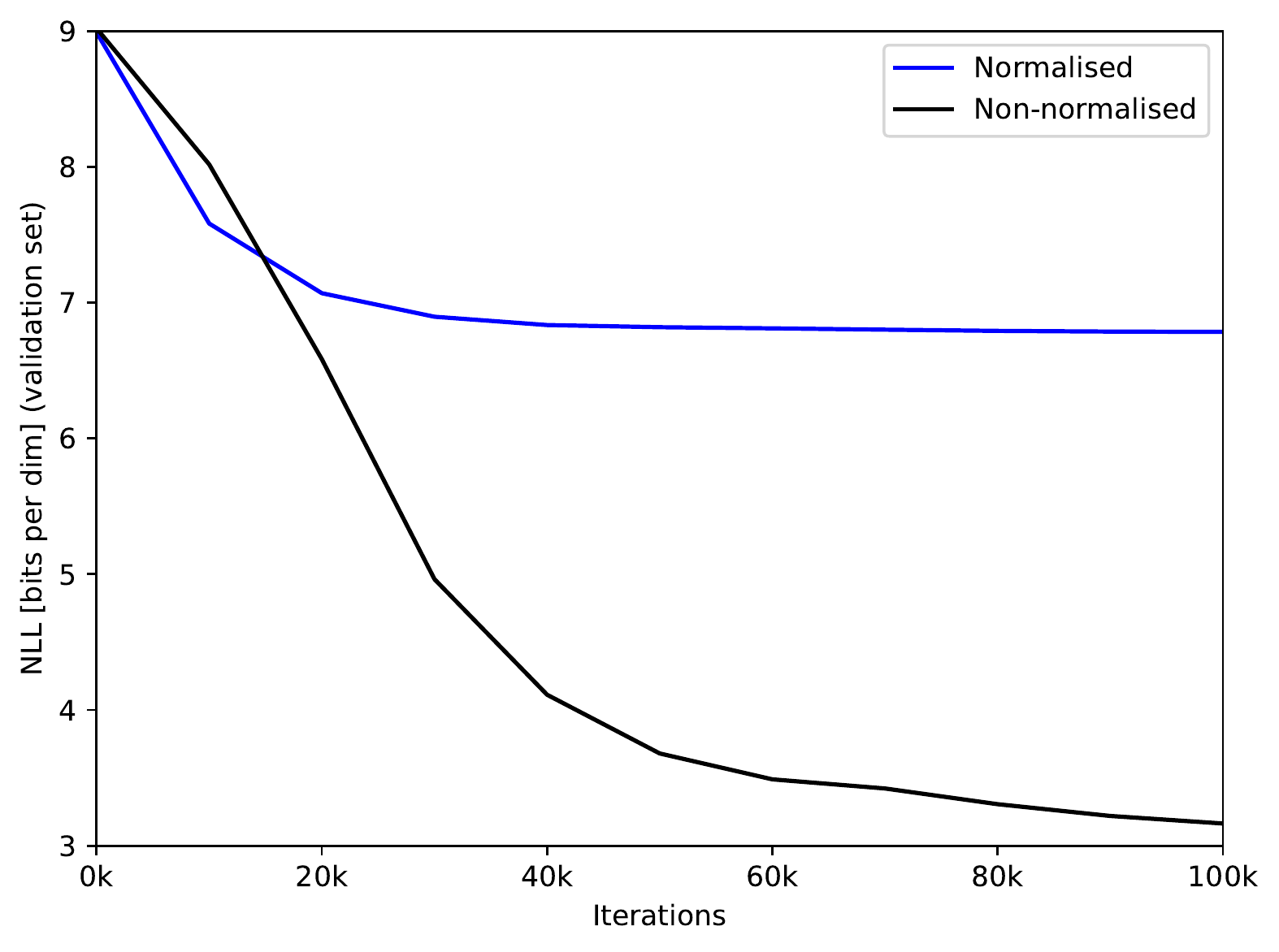}
    \includegraphics[width=.32\linewidth]{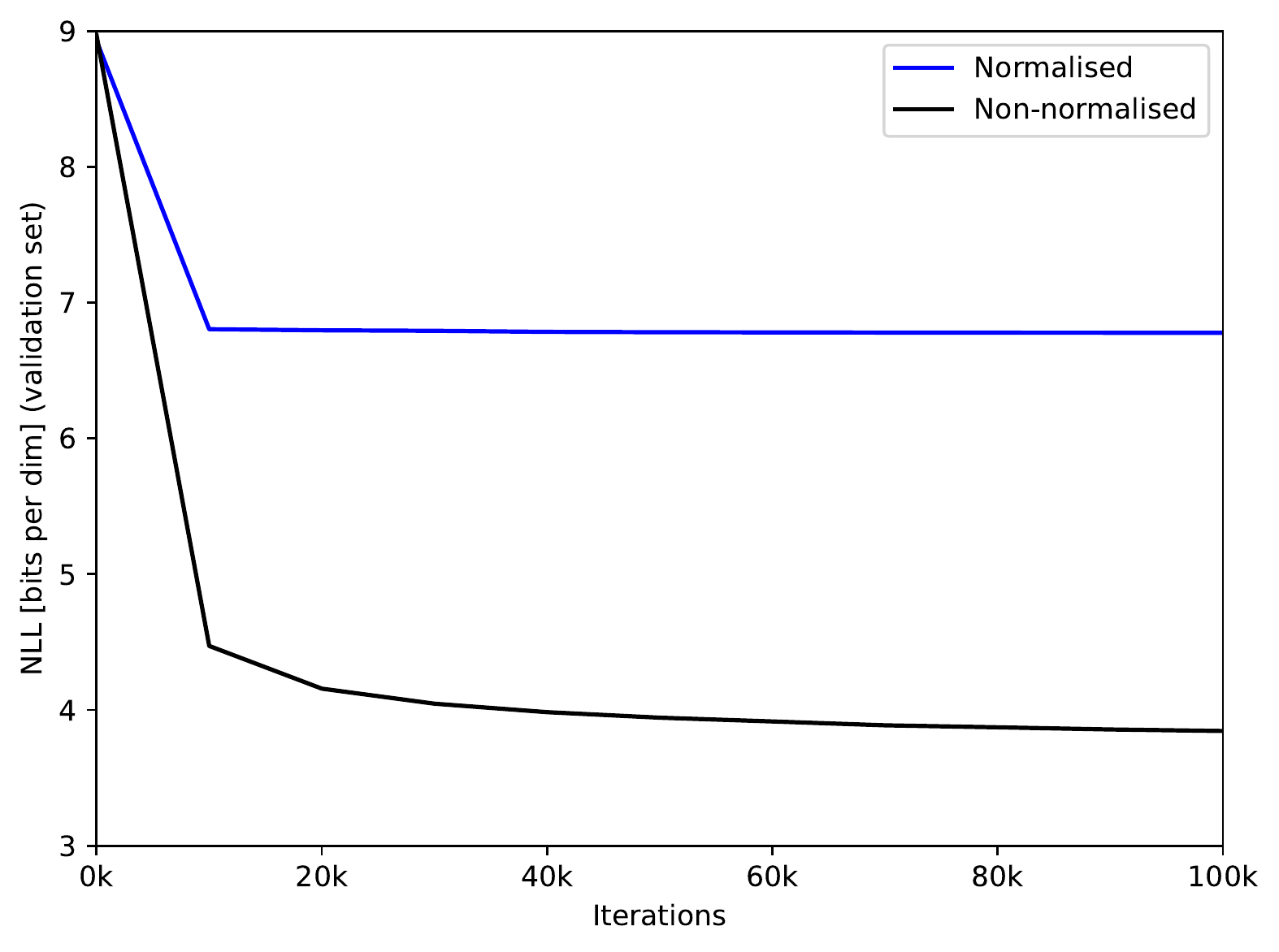}
    \caption{
    On the training dynamics of VDVAE with and without a normalised residual state norm.
    NLL ($\downarrow$) measured on the validation set of MNIST [left], CIFAR10 [middle] and ImageNet32 [right].
    The normalised runs suffer from poor training dynamics from the very start of the optimisation and even terminate early on MNIST and CIFAR10, indicating that VDVAE makes use of the time representing state norm during training.
    }
    \label{fig:normalised_vs_nonnormalised_nll}
\end{figure}

\ff{
}

\clearpage
\subsection{Sampling instabilities in HVAEs}
\label{app:add_exp_details_results_Sampling instabilities in HVAEs}

When retrieving unconditional samples from our models, we scale the variances in the unconditional distributions with a temperature factor $\tau$, as is common practice.
We tune $\tau$ ``by eye'' to improve the fidelity of the retrieved images, yet do not cherry pick these samples.
In Figs. \ref{fig:uncond_samples_more_cifar10} to \ref{fig:uncond_samples_more_celeba}, we provide additional, not cherry-picked unconditional samples for models trained on CIFAR10, ImageNet32, ImageNet64, MNIST and CelebA, extending those presented in Fig. \ref{fig:uncond_samples_instability}.
As shown earlier, the instabilities in VDVAE result in poor unconditional samples for CIFAR10, ImageNet32 and ImageNet64, but relatively good samples for MNIST and CelebA.

\begin{figure}[h!]
    \centering
    \includegraphics[width=.5\linewidth]{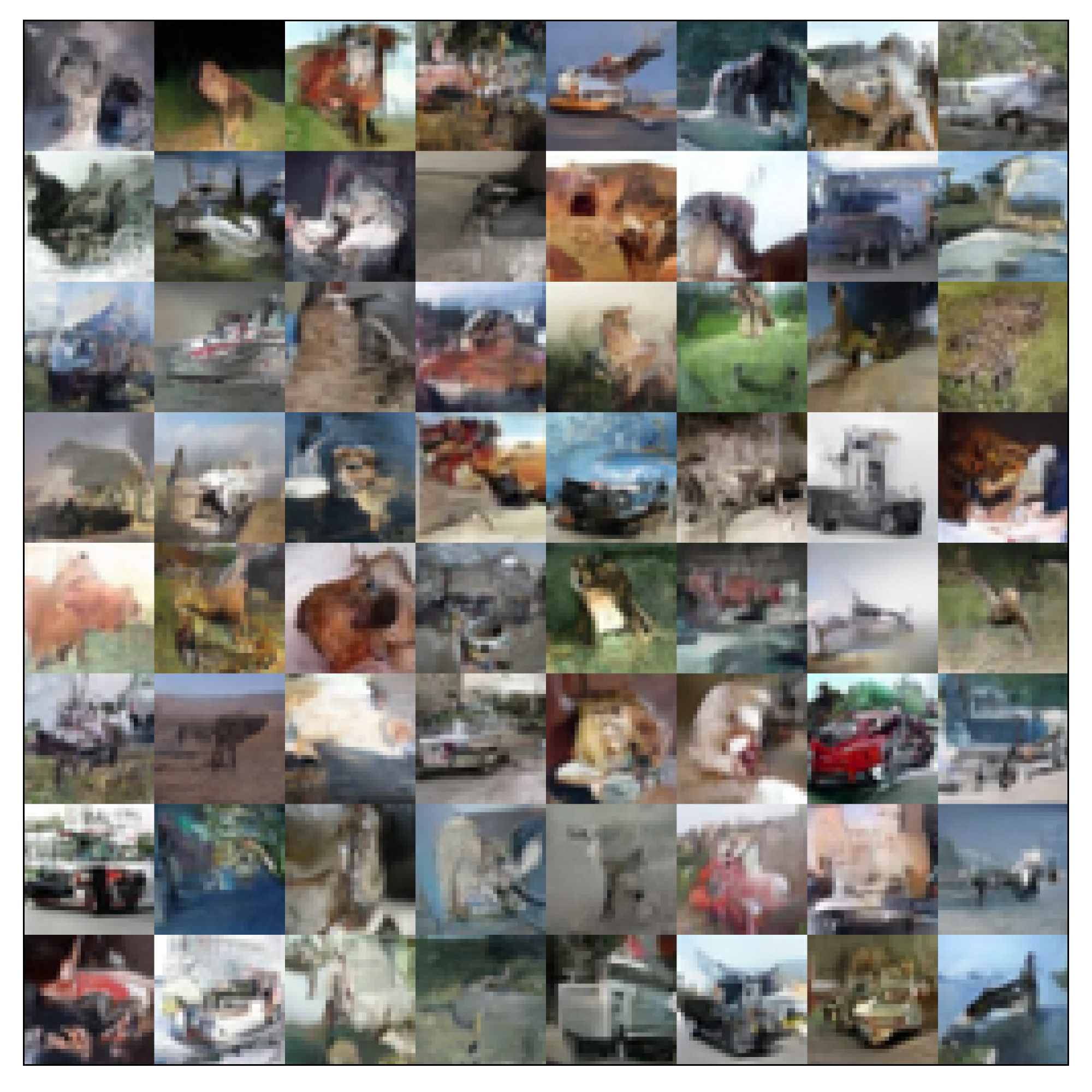}
    \caption{
    Further unconditional samples (not cherry-picked) of \texttt{VDVAE$^*$} on \textit{CIFAR10}, augmenting those presented in Fig. \ref{fig:uncond_samples_instability}.
    While samples on MNIST and CelebA demonstrate high fidelity and diversity, samples on CIFAR10, ImageNet32 and ImageNet64 are diverse, but  unrecognisable, demonstrating the instabilities identified by Theorem \ref{thm:discrete-VDVAE}.
    We chose the temperature as $\tau=0.9$.
    }
    \label{fig:uncond_samples_more_cifar10}
\end{figure}

\begin{figure}[h!]
    \centering
    \includegraphics[width=.5\linewidth]{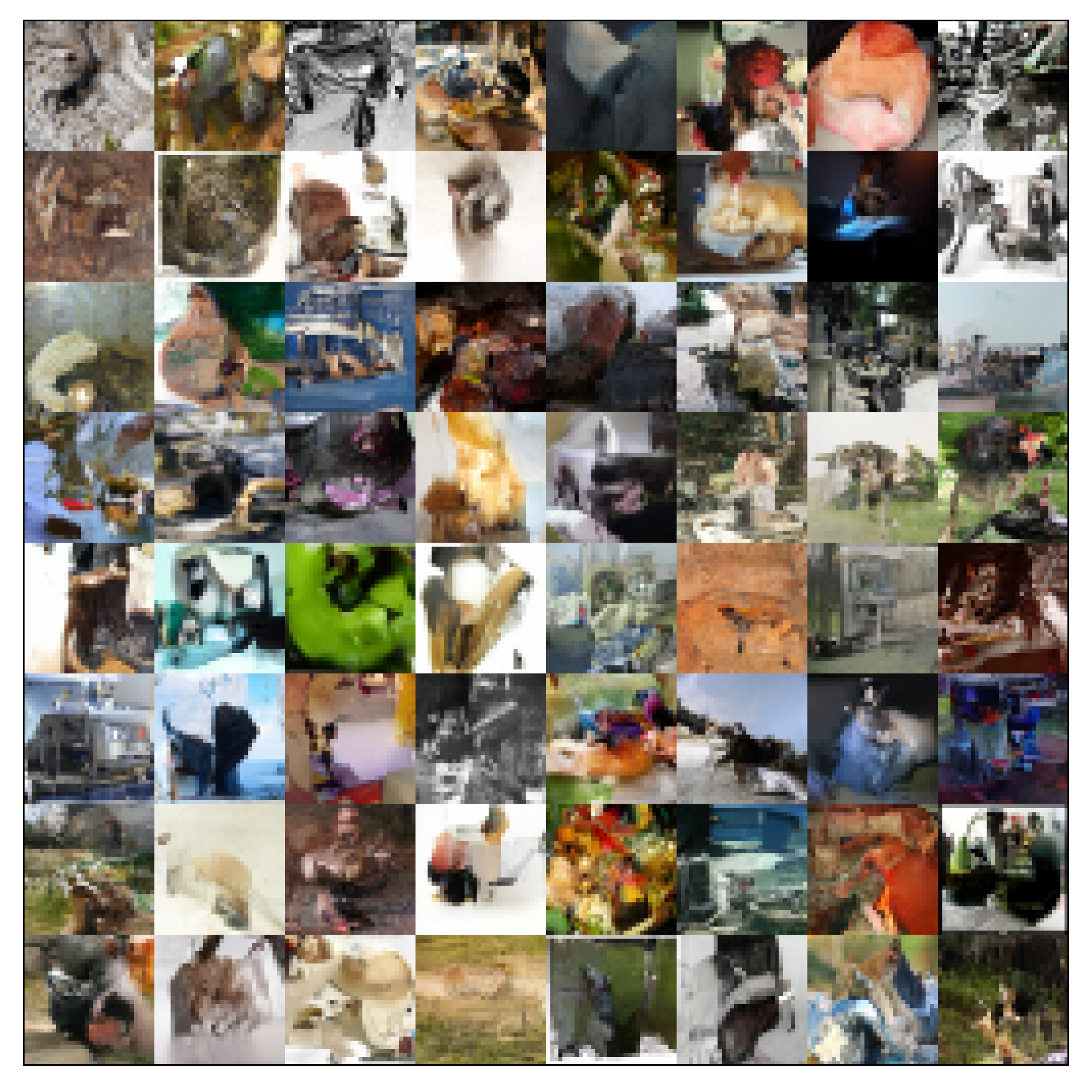}
    \caption{
    Further unconditional samples (not cherry-picked) of \texttt{VDVAE$^*$} on \textit{ImageNet32}, augmenting those presented in Fig. \ref{fig:uncond_samples_instability}.
    While samples on MNIST and CelebA demonstrate high fidelity and diversity, samples on CIFAR10, ImageNet32 and ImageNet64 are diverse, but  unrecognisable, demonstrating the instabilities identified by Theorem \ref{thm:discrete-VDVAE}.
    We chose the temperature as $\tau=1.0$.
    }
    \label{fig:uncond_samples_more_i32}
\end{figure}

\begin{figure}[h!]
    \centering
    \includegraphics[width=.5\linewidth]{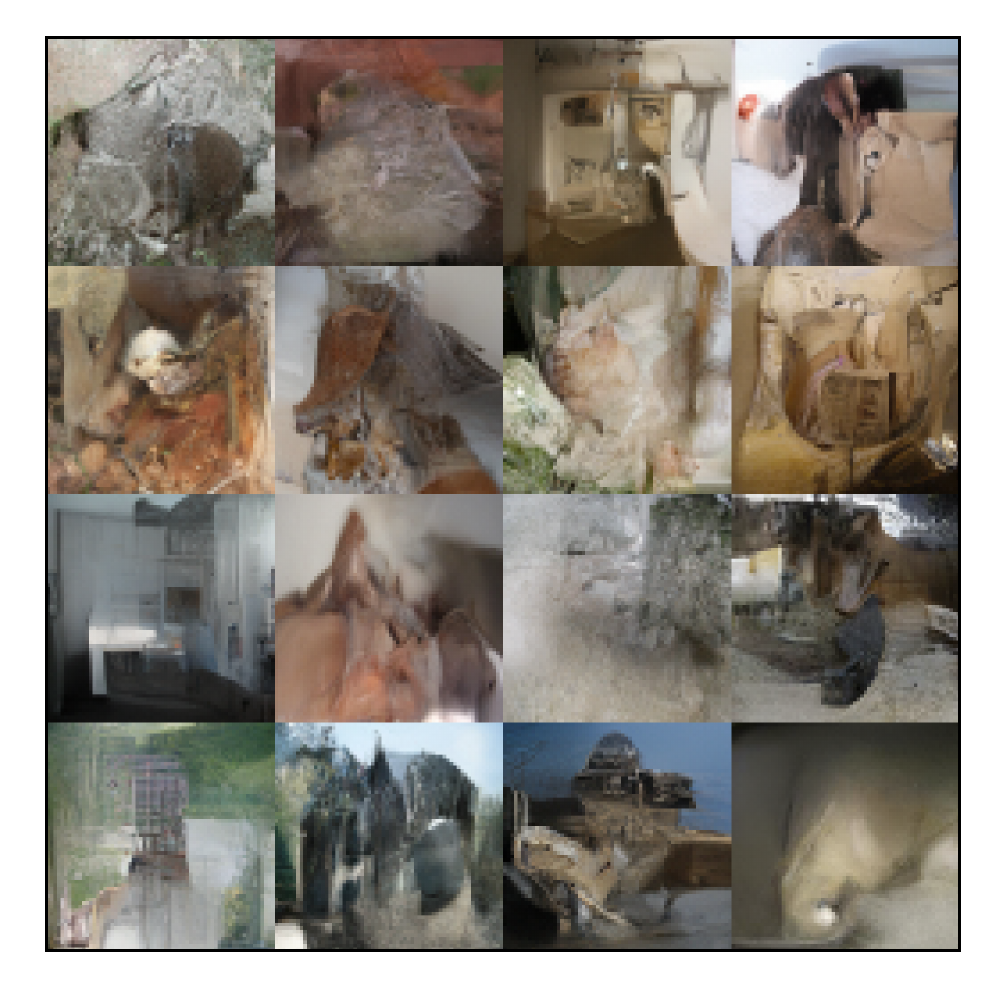}
    \caption{
    Further unconditional samples (not cherry-picked) of \texttt{VDVAE$^*$} on \textit{ImageNet64}, augmenting those presented in Fig. \ref{fig:uncond_samples_instability}.
    While samples on MNIST and CelebA demonstrate high fidelity and diversity, samples on CIFAR10, ImageNet32 and ImageNet64 are diverse, but  unrecognisable, demonstrating the instabilities identified by Theorem \ref{thm:discrete-VDVAE}.
    Temperatures $\tau$ are tuned for maximum fidelity.    
    We chose the temperature as $\tau=0.9$.
    }
    \label{fig:uncond_samples_more_i64}
\end{figure}

\begin{figure}[h!]
    \centering
    \includegraphics[width=.5\linewidth]{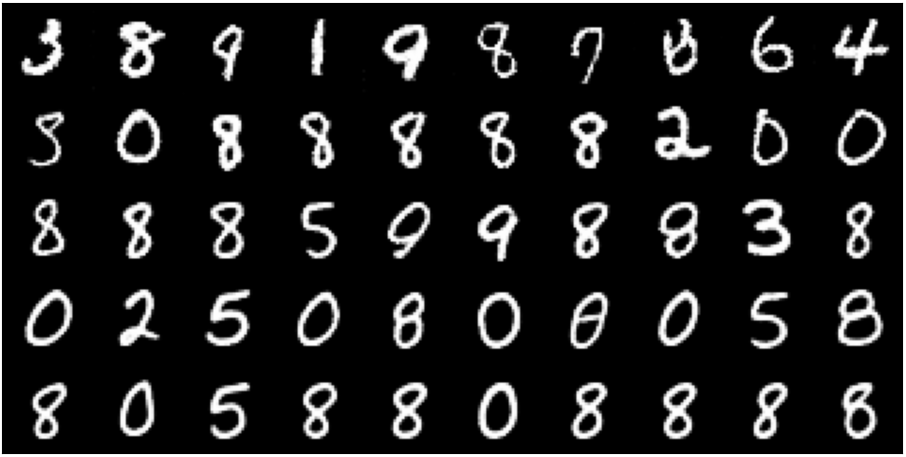}
    \caption{
    Further unconditional samples (not cherry-picked) of \texttt{VDVAE$^*$} on \textit{MNIST}, augmenting those presented in Fig. \ref{fig:uncond_samples_instability}.
    While samples on MNIST and CelebA demonstrate high fidelity and diversity, samples on CIFAR10, ImageNet32 and ImageNet64 are diverse, but  unrecognisable, demonstrating the instabilities identified by Theorem \ref{thm:discrete-VDVAE}.
    We chose the temperatures as $\tau \in \{ 1.0, 0.9, 0.8, 0.7, 0.5 \}$ (corresponding to the rows).
    }
    \label{fig:uncond_samples_more_mnist}
\end{figure}

\begin{figure}[h!]
    \centering
    \includegraphics[width=.5\linewidth]{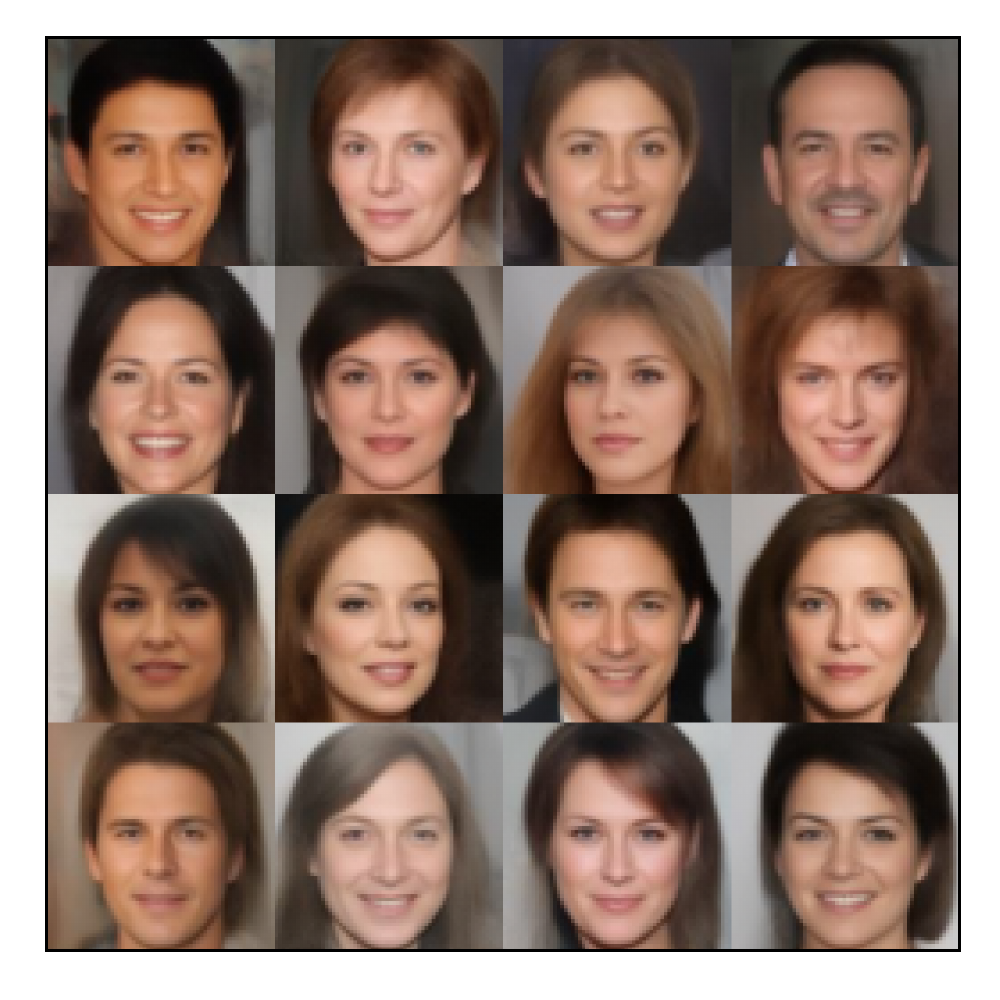}
    \caption{
    Further unconditional samples (not cherry-picked) of \texttt{VDVAE$^*$} on \textit{CelebA}, augmenting those presented in Fig. \ref{fig:uncond_samples_instability}.
    While samples on MNIST and CelebA demonstrate high fidelity and diversity, samples on CIFAR10, ImageNet32 and ImageNet64 are diverse, but  unrecognisable, demonstrating the instabilities identified by Theorem \ref{thm:discrete-VDVAE}.
    We chose the temperature as $\tau=0.5$.    
    }
    \label{fig:uncond_samples_more_celeba}
\end{figure}

\clearpage

In addition, we here also visualise the representational advantage of HVAEs. 
Fig. \ref{fig:samples_representational_adv} shows samples where we gradually increase the number of samples from the posterior vs. the prior distributions in each resolution across the columns.
This means that in column 1, we sample the first latent $\B{z}_l$ \textit{in each resolution} from the (on encoder activations conditional) posterior $q$, and all other latents from the prior $p$.
A similar figure, but gradually increasing the contribution of the posterior across the blocks of all resolutions (i.e. column 1 samples $\B{z}_l$ from the posterior in the very first resolution only) is shown in VDVAE \cite[Fig. 4]{Child2020VeryImages}.
Fig. \ref{fig:samples_representational_adv} 

\begin{figure}[h!]
    \centering
    \includegraphics[width=1\linewidth]{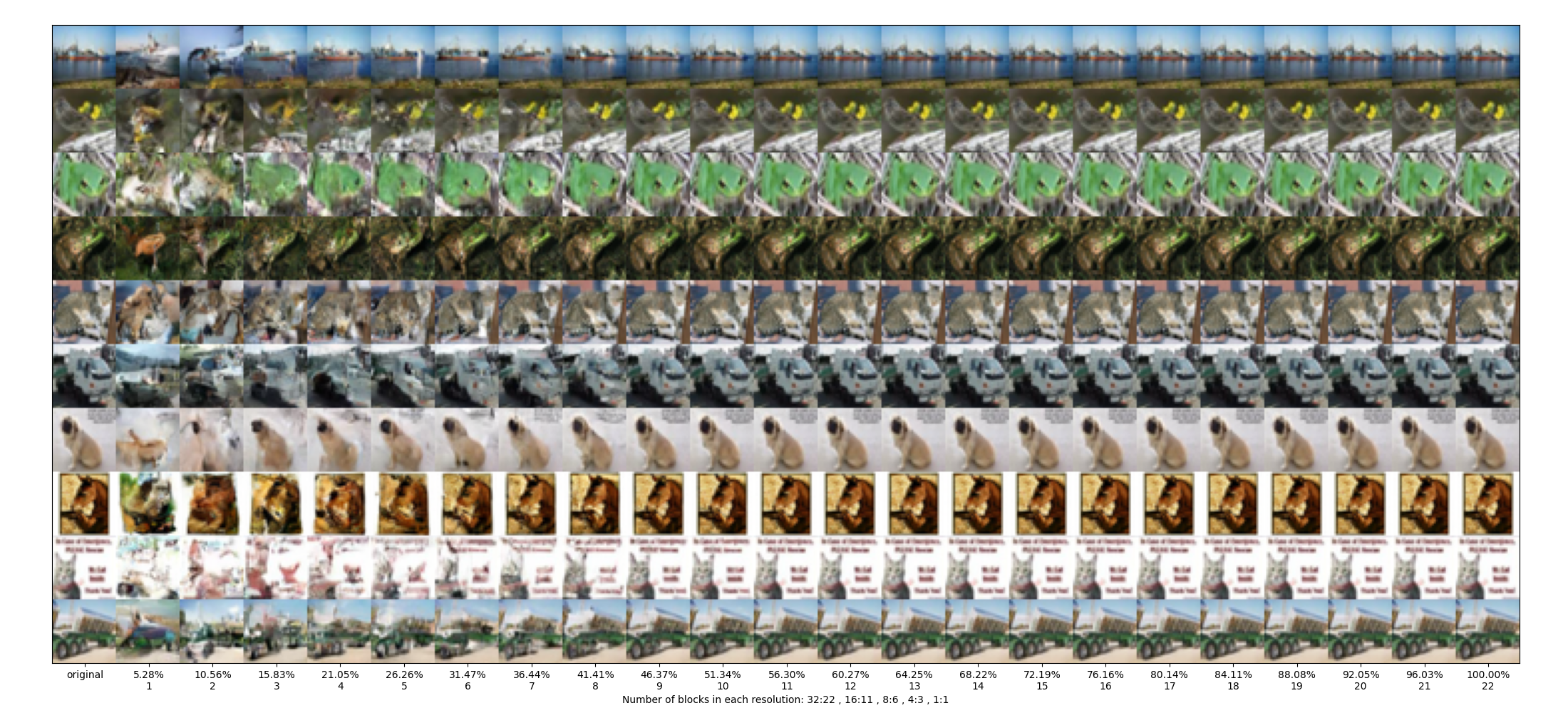}
    \caption{
    Samples drawn from our model when gradually increasing the contribution of the approximate posterior.
    In each column with integer $s$, we sample the first $s$ latent variables from the approximate posterior in each resolution, i.e.  $\v{z}_i \sim q(\v{z}_i | \v{z}_{>i})$  (up to the maximum number of latent variables in each resolution), and $\v{z}_j \sim p(\v{z}_j | \v{z}_{>j})$ for all other latent variables.
    The percentage number indicates the fraction of the number of latent variables among all latent variables sampled from the approximate posterior.
    In the left-most column, we visualise corresponding input images.
    }
    \label{fig:samples_representational_adv}
\end{figure}

\clearpage
\newpage

\subsection{Ablation studies}
\label{app:add_exp_details_results_Ablation_studies}

\subsubsection{Number of latent variables}
The number of latent variables increase when increasing the stochastic depth through weight-sharing. 
Thus, an important ablation study is the question whether simply increasing the number of latent variables improves HVAE performance, which may explain the weight-sharing effect.
On CIFAR10, we find that this is not the case: 
In Table \ref{tab:num_latent_vars}, we analyse the effect of increasing the number of latent variables ceteris paribus. 
Furthermore, in Fig. \ref{fig:num_latent_vars}, we report validation NLL during training for the same runs. 
In this experiment, we realise the increase in number of latent variables by increasing the number of channels in each latent variable $\v z_l$ exponentially while slightly decreasing the number of blocks so to keep the number of parameters roughly constant.
Both results indicate that the number of latent variables, at least for this configuration on CIFAR10, do not add performance and hence cannot explain the weight-sharing performance.

\begin{table}[h!]
\caption{
On the effect of the number of latent variables on CIFAR10.
We report the NLL on the test set at convergence.
}
\centering
\begin{tabular}{ccc} 
\toprule
\# of latent variables & \# Params & NLL ($\downarrow$)  \\   %
\midrule
396k & 39m & 2.88 \\  %
792k & 39m & 2.88 \\  %
1.584m & 39m & 2.87 \\  %
3.168m &  39m& 2.88  \\  %
\bottomrule
\end{tabular}
\label{tab:num_latent_vars}
\end{table}

\vspace{1cm}

\begin{figure}[h!]
    \centering
    \includegraphics[width=.5\linewidth]{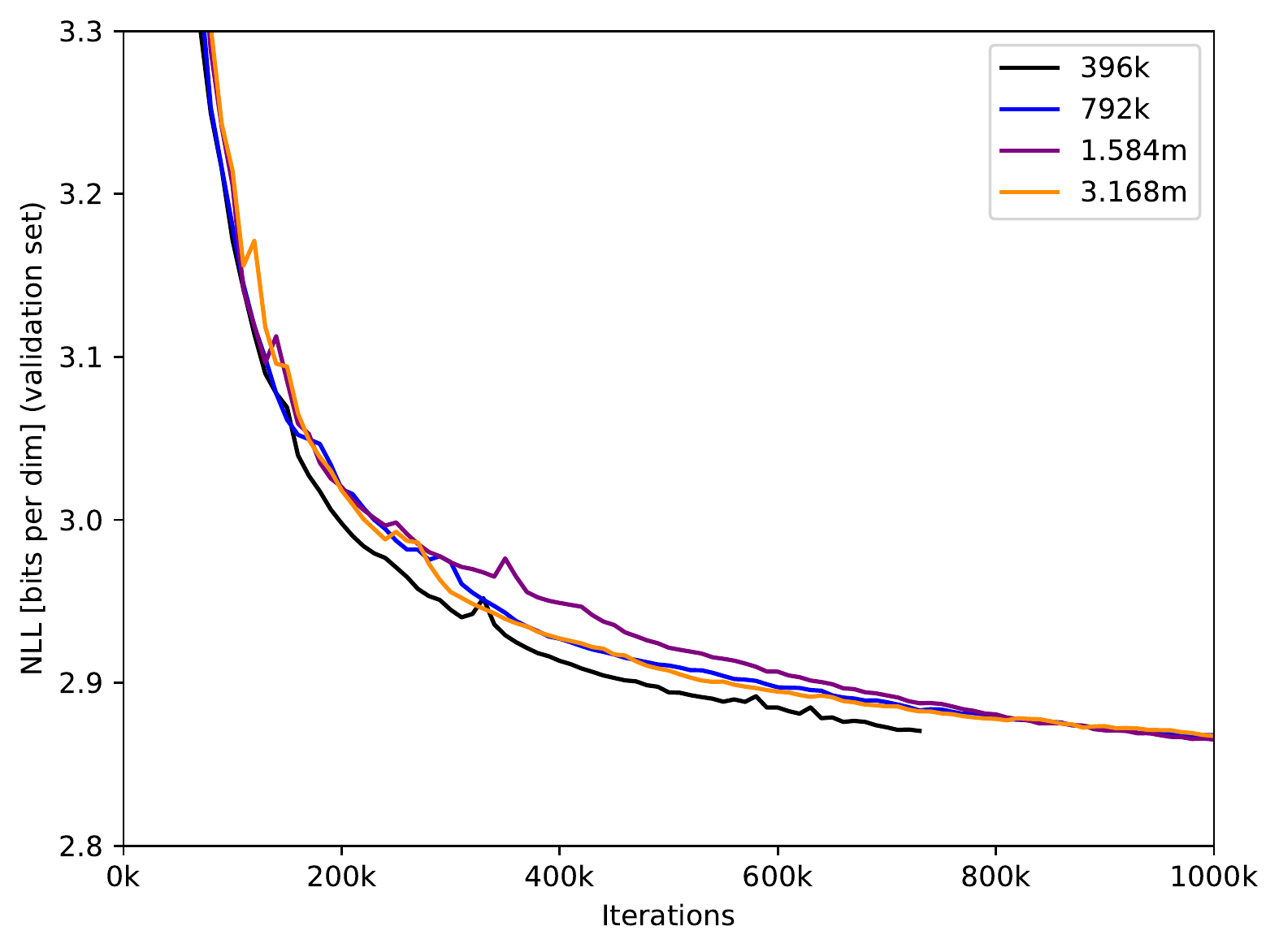}
    \caption{
    On the effect of the number of latent variables.
    We report NLL on the validation set of CIFAR10 during training.
    }
    \label{fig:num_latent_vars}
\end{figure}

\clearpage

\subsubsection{Fourier features}

In this experiment, we are interested in the effect of Fourier features imposed onto a Haar wavelet basis due to the inductive bias of the U-Net.
Intuitively, we would expect that Fourier features do not add performance as the U-Net already imposes a good basis for images.
We now validate this hypothesis experimentally: 
We compute Fourier features in every ResNet block at three different locations as additional channels and varying frequencies.
We implement Fourier features closely following VDM \cite{Kingma2021VariationalModels}: 
Let $h_{i,j,k}$ be an element of a hidden activation of the network, for example of a sampled latent $h = \v z_l$, in spatial position $(i,j)$ and channel $k$. 
Then, for each scalar $h_{i,j,k}$, we add two transformed scalars for each frequency governed by $\beta$ as follows: 
\begin{align*}
    f_{i, j, k}^{\beta}=\sin \left(z_{i, j, k} 2^{\beta} \pi\right), \text { and } g_{i, j, k}^{\beta}=\cos \left(z_{i, j, k} 2^{\beta} \pi\right).
\end{align*}
In our experiments, we experiment with different choices for $\beta$, but typically select two values at a time (as in VDM), increasing the number of channels in the resulting activation by a factor of five.
Fourier features are computed on and concatenated to activations at three different locations (in three separate experiments): At the input of the ResNet block, after sampling, and for the input of the two branches parameterising the posterior and prior distributions.

In Tables \ref{tab:fourier_mnist} and \ref{tab:fourier_cifar}, we report performance when concatenating Fourier features at every ResNet block in these three locations.
In all cases, Fourier features deteriorate performance in this multi-resolution wavelet basis, particularly for high-frequencies which often lead to early termination due to numeric overflows.
However, if training only a single-resolution model where no basis is enforced, training does not deteriorate, not even for high-frequency Fourier features, yet performance can neither be improved.
Furthermore, we experimented with computing and concatenating the Fourier features only to the input image of the model, hypothesising numerical instabilities caused by computing Fourier transforms at every ResNet block, and report results in Table \ref{tab:fourier_input}. 
Here, performance is significantly better as runs no longer deteriorate, but Fourier features still do not improve performance compared to not using Fourier features at all.

\begin{table}[h!]
\caption{
Fourier features introduced and concatenated in every ResNet block at three different locations on \textit{MNIST}.
VDVAE typically deteriorates or has poor performance.
}
\centering
\begin{tabular}{cc} \toprule
Exponent $\beta$ & NLL \\ 
\midrule 
\textbf{Loc. 1} \hfill \\
$ [ 1,2 ] $ & $\leq 78.4$ \\  %
$[3, 4 ]$ & $\leq 80.55$ \\   %
$[5, 6]$ (\blackcross) & --  \\  %
\midrule
\textbf{Loc. 2} \hfill \\
$[1,2]$  & $\leq 554.50$ \\  %
$[3, 4]$  (\blackcross) & -- \\  %
$[5, 6]$  (\blackcross) & --  \\   %
\midrule
\textbf{Loc. 3} \hfill \\
$[1,2]$ (\blackcross) & -- \\  %
$[2, 3]$ & $\leq 306.67$ \\   %
$[3, 4]$ & $\leq 345.67$  \\  %
\midrule
\textbf{Loc. 1 \& single-res.} \hfill  \\
$[3, 4]$ & $\leq 87.55$ \\  %
$[5, 6]$ & $\leq 86.96$  \\  %
$[7, 8]$ & $\leq 91.67$  \\  %
\midrule
\textbf{No Fourier Features} & $\leq 79.81$ \\  %
\bottomrule
\end{tabular}
\label{tab:fourier_mnist}
\end{table}

\begin{table}[h!]
\caption{
Fourier features introduced and concatenated in every ResNet block at three different locations on \textit{CIFAR10}.
VDVAE typically deteriorates or achieves a poor performance.
}
\centering
\begin{tabular}{cc} \toprule
Exponent $\beta$ & NLL \\ 
\midrule 
\textbf{Loc. 1} \hfill \\
$[3, 4]$ (\blackcross) & --  \\   %
$[5, 6]$ (\blackcross) & --  \\   %
$[7, 8]$ & $\leq 8.94$  \\   %
\midrule
\textbf{Loc. 2} \hfill \\    
$[3, 4]$ (\blackcross) & --  \\   %
$[5, 6]$ (\blackcross) & --  \\   %
$[7, 8]$ (\blackcross) & --  \\   %
\midrule
\textbf{Loc. 3} \hfill \\
$[3, 4]$ (\blackcross) & --  \\   %
$[5, 6]$ & $\leq 8.94$  \\   %
$[7, 8]$ & $\leq 8.99$  \\   %
\midrule
\textbf{No Fourier Features} & $\leq 2.87$ \\  %
\bottomrule
\end{tabular}
\label{tab:fourier_cifar}
\end{table}

\vspace{-.5cm}

\begin{table}[h!]
\caption{
Fourier features introduced on the input image of the model only, with results on  \textit{CIFAR10}.
While performing better than if introduced at every ResNet block, still Fourier features do not improve performance compared to using no Fourier features at all.
}
\centering
\begin{tabular}{cc} \toprule
Exponent $\beta$ & NLL \\ 
\midrule 
\textbf{Fourier Features on input only} \hfill \\
$[3, 4]$  & $\leq 2.95$  \\   %
$[5, 6]$  & $\leq 2.96$  \\   %
$[7, 8]$ & $\leq 2.89$  \\   %
\midrule
\textbf{No Fourier Features} & $\leq 2.87$ \\  %
\bottomrule
\end{tabular}
\label{tab:fourier_input}
\end{table}

\clearpage

\subsubsection{On the effect of a multi-resolution bridge. }

State-of-the-art HVAEs have a U-Net architecture with pooling and, hence, are multi-resolution bridges (see Theorem \ref{thm:vdvae_sde}).
We investigate the effect of multiple resolutions in HVAEs (here with spatial dimensions $\{32^2,16^2,8^2,4^2,1^2\}$) against a single resolution (here with spatial dimension $32^2$).
We choose the number of blocks for the single resolution model such that they are distributed in the encoder and decoder proportionally to the multi-resolution model and the total number of parameters are equal in both, ensuring a fair comparison.
As we show in Table~\ref{tab:resolutions}, the multi-resolution models perform slightly better than their single-resolution counterparts, yet we would have expected this difference to be more pronounced. 
We also note that it may be worth measuring other metrics for instance on fidelity, such as the FID score \cite{heusel2017gans}.
Additionally, multi-resolution models have a representational advantage due to their Haar wavelet basis representation (illustrated in Appendix~\ref{app:add_exp_details_results_Ablation_studies}, Fig. \ref{fig:samples_representational_adv}). 

\begin{table}[h!]
\caption{
Single- vs. multi-resolution HVAEs. 
}
\centering
\begin{tabular}{ccc} \toprule  %
\# Resolutions & \# Params & NLL \\ 
\midrule 
\textbf{MNIST}  \\  %
Single & 328k & $\leq 81.40$   \\  %
Multiple & 339k & $\leq 80.14$  \\  %
\midrule
\textbf{CIFAR10}  \\   %
Single & 39m & $\leq 2.89$   \\  %
Multiple & 39m & $\leq 2.87$  \\  %
\midrule
\textbf{ImageNet32}  \\   %
Single & 119m & $\leq 3.68$   \\  %
Multiple & 119m & $\leq 3.67$ \\  %
\bottomrule
\end{tabular}
\label{tab:resolutions}
\end{table}

\subsubsection{On the importance of a stochastic differential equation structure in HVAEs}
\label{app:On the importance of a stochastic differential equation structure in HVAEs}

A key component of recent HVAEs is a residual cell, as outlined in \S\ref{sec:Related_work}. 
The residual connection makes HVAEs discretise an underlying SDE, as we outlined in this work.
Experimentally, it was previously noted as being crucial for stability of very deep HVAEs.
Here, we are interested in ablating the importance of imposing an SDE structure into HVAEs: 
We compare models with a residual HVAE cell (as in VDVAE) with a non-residual HVAE cell which is as close to VDVAE as possible to ensure a fair comparison.
The non-residual VDVAE cell does not possess a residual state which flows through the backbone architecture. 
We achieve this by removing the connection between the first and second element-wise addition in VDVAE's cell (see \cite[Fig. 3]{Child2020VeryImages}), which is equivalent to setting $Z_{i,+} = 0$.
Hence, in the non-residual cell, during training and evaluation, the reparameterised sample is directly taken forward.
Note that this is distinct from the Euler-Maruyama cell which features a residual connection. 
Our experiments confirm that a \textit{residual} cell is key for training stability, as illustrated in Table \ref{tab:cell} and Fig. \ref{fig:cell}:
Without a residual state flowing through the decoder, models quickly experience posterior collapse of the majority of layers during training.  %

\begin{table}[h!]
\caption{
Residual vs. non-residual VDVAE cell.
The residual HVAE strongly outperforms a non-residual VDVAE cell, where the latter's training deteriorates. 
This is also analysed in Fig. \ref{fig:samples_representational_adv}.
We report NLL on the test set at convergence, or at the last model checkpoint before deterioration of training.
}
\centering
\begin{tabular}{cc} \toprule
Cell type & NLL \\ 
\midrule 
\textbf{MNIST}  \\  %
Residual VDVAE cell & $\leq80.05$ \\  %
Non-residual VDVAE cell & $\leq112.58$ \\  %
\midrule 
\textbf{CIFAR10}  \\   %
Residual VDVAE cell & $\leq 2.87$ \\  %
Non-residual VDVAE cell & $\leq 3.66$ \\  %
\midrule
\textbf{ImageNet}  \\   %
Residual VDVAE cell & $\leq 3.667$ \\  %
Non-residual VDVAE cell (\blackcross) & $\leq 4.608$ \\  %
\bottomrule
\end{tabular}
\label{tab:cell}
\end{table}

\begin{figure}[h!]
    \centering
    \includegraphics[width=.48\linewidth]{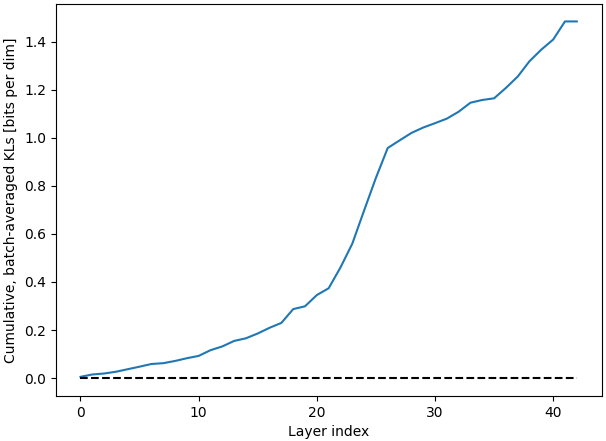}
    \includegraphics[width=.48\linewidth]{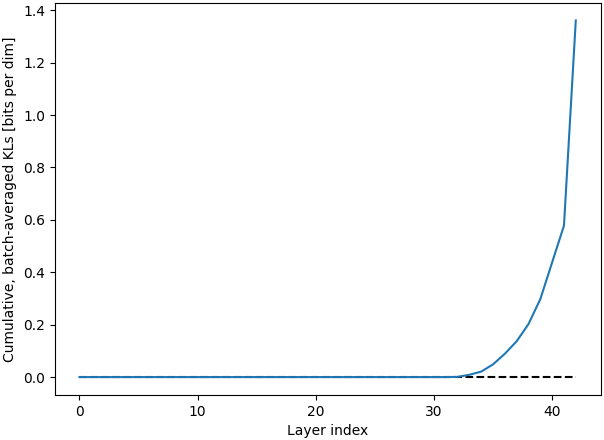}
    \caption{
    Cumulative sum of KL-terms in the ELBO of a residual and non-residual VDVAE, averaged over a batch at convergence.
    We report the two CIFAR10 runs in Table \ref{tab:cell}.
    The posterior collapses for the majority of the latent variables in the non-residual VDVAE cell case [right], but carries information for all latent variables in the regular, residual cell case [left].
    }
    \label{fig:cell}
\end{figure}

\subsubsection{Synchronous vs. asynchronous processing in time}
\label{app:Synchronous vs. asynchronous processing in time}

During the bottom-up pass, VDVAE takes forward the activation of the last time step in each resolution which is passed to every time step in the top-down pass on the same resolution (see Fig. 3 in VDVAE \cite{Child2020VeryImages}).
In this ablation study, we were interested in this slightly peculiar choice of an \textit{asynchronous} forward and backward process and to what degree it is important for performance.
We thus compare an asynchronous model, with skip connections as in VDVAE \cite{Child2020VeryImages}, with a \textit{synchronous} model, where activations from the bottom-up pass are taken forward to the corresponding time step in the top-down pass.
In other words, in the synchronous case, the skip connection mapping between time steps in the encoder and decoder is `bijective', and it is not `injective', but `surjective' in the asynchronuous case.
We realise the synchronous case by choosing the same number of blocks in the encoder as \texttt{VDVAE$^*$} has in the decoder, i.e. constructing a `symmetric' model.
To ensure a fair comparison, both models (synchronous and asynchronous) are further constructed to have the same number of parameters.
In Table \ref{tab:synch_asynch}, we find that synchronous and asynchronous processing achieve comparable NLL, indicating that the asynchronous design is not an important contributor to performance in VDVAE.
We note, however, that an advantage of the asynchronous design, which is exploited by VDVAE, is that the bottom-up and top-down architectures can have different capacities, i.e. have a different number of ResNet blocks.
VDVAE found that a more powerful decoder was beneficial for performance \cite{Child2020VeryImages}.

\begin{table}[h!]
\caption{
Synchronous vs. asynchronous processing in time.
We report NLL on the test set on CIFAR10 and ImageNet32, respectively.
}
\centering
\begin{tabular}{cc} \toprule
Processing & NLL \\ 
\midrule 
\textbf{CIFAR10}  \\  %
Synchronous & $\leq 2.85$ \\  %
Asynchronous & $\leq 2.86$ \\  %
\midrule 
\textbf{ImageNet32}  \\  %
Synchronous & $\leq 3.69$ \\  %
Asynchronous & $\leq 3.69$ \\  %
\bottomrule
\end{tabular} 
\label{tab:synch_asynch}
\end{table}

\end{appendices}

\end{document}